\documentclass{article}

\usepackage{graphicx}
\usepackage{booktabs} 

\usepackage{bm}
\usepackage{amsthm}
\usepackage{siunitx}
\usepackage{graphicx}
\usepackage{subcaption}
\usepackage[dvipsnames]{xcolor}
\usepackage{bbm}
\usepackage{amsmath}
\usepackage{amssymb}
\usepackage{physics}
\usepackage{placeins}

\newtheorem{theorem}{Theorem}[section]

\newtheorem{lemma}[theorem]{Lemma}

\DeclareMathOperator*{\Ave}{Ave}
\DeclareMathOperator*{\diag}{diag}
\DeclareMathOperator*{\Hess}{\text{Hess}}

\newcommand{\R}{\mathbb{R}}
\newcommand{\Lagr}{\mathcal{L}}
\newcommand{\one}{{\mathbbm{1}}}

\usepackage{hyperref}

\usepackage[accepted]{icml2019}

\icmltitlerunning{Measurements of Three-Level Hierarchical Structure in the Outliers in the Spectrum of Deepnet Hessians}

\begin{document}

\twocolumn[
\icmltitle{Measurements of Three-Level Hierarchical Structure\\ in the Outliers in the Spectrum of Deepnet Hessians}



\icmlsetsymbol{equal}{*}

\begin{icmlauthorlist}
\icmlauthor{Vardan Papyan}{stanford}
\end{icmlauthorlist}

\icmlaffiliation{stanford}{Department of Statistics, Stanford University, California 94305, USA}

\icmlcorrespondingauthor{Vardan Papyan}{papyan@stanford.edu}

\icmlkeywords{Machine Learning, ICML}

\vskip 0.3in
]



\printAffiliationsAndNotice{}  

\begin{abstract}
We consider deep classifying neural networks.
We expose a structure in the derivative of the logits with respect to the parameters of the model, which is used to explain the existence of outliers in the spectrum of the Hessian.
Previous works decomposed the Hessian into two components, attributing the outliers to one of them, the so-called Covariance of gradients.
We show this term is not a Covariance but a second moment matrix, i.e., it is influenced by means of gradients.
These means possess an additive two-way structure that is the source of the outliers in the spectrum.
This structure can be used to approximate the principal subspace of the Hessian using certain ``averaging'' operations, avoiding the need for high-dimensional eigenanalysis.
We corroborate this claim across different datasets, architectures and sample sizes.
\end{abstract}

\section{Introduction}
We consider a $C$-class classification problem. We are given a sample of $n$ training examples, $n_c$ in each class, $\bigcup_{c=1}^C \{ (x_{i,c}, y_c) \}_{i=1}^{n_c}$, where $x_{i,c}$ is the $i$-th example in the $c$-th class and $y_c$ is its corresponding one-hot vector. The goal is to predict the labels of unseen data based on the limited examples provided for training. State-of-the-art methods fit a deep neural network, parameterized by a vector of parameters $\theta \in \R^p$, to the training data by minimizing the empirical loss
\begin{equation} \label{eq:loss}
     \Lagr(\theta) = \Ave_{i,c} \{ \ell ( f(x_{i,c}; \theta), y_c) \},
\end{equation}
averaged across the training data through the operator $\Ave_{i,c}$. Here, $f(x_{i,c}; \theta) \in \R^C$ are the logits -- the output of the classifier prior to the softmax layer -- while \mbox{$\ell ( f(x_{i,c}; \theta), y_c ) \in \R^+$} is the cross-entropy loss between the softmax of $f(x_{i,c}; \theta)$ and the one-hot vector $y_c$.

In this work, we investigate the Hessian of the training loss, given by
\begin{equation}
    \Hess(\theta) = \Ave_{i,c} \left\{ \pdv[2]{\ell ( f(x_{i,c}; \theta), y_c)}{\theta} \right\}.
\end{equation}
Using the Gauss-Newton decomposition, the above can be written as a summation of two components
\begin{align}
    & \Hess(\theta) \\
    = & \underbrace{\Ave_{i,c} \left\{ \sum_{c'=1}^C \pdv{\ell ( z, y_c)}{z_{c'}} \Bigg|_{z=f(x_{i,c}; \theta)} \pdv[2]{f_{c'}(x_{i,c}; \theta)}{\theta} \right\}}_H \nonumber \\
    + & \underbrace{\Ave_{i,c} \left\{ \pdv{f(x_{i,c};\theta)}{\theta}^T \pdv[2]{\ell ( z, y_c)}{z} \Bigg|_{z = f(x_{i,c}; \theta)} \pdv{f(x_{i,c};\theta)}{\theta} \right\}}_G, \nonumber
\end{align}
where $f_{c'}(x_{i,c}; \theta)$ is the value in the $c'$-th coordinate of the logits $f(x_{i,c}; \theta)$ (similarly for $z_{c'}$). In what follows, we refer to $c'$ as a \textit{logit coordinate}.

Many works studied the Hessian over the years, both from the theoretical and practical point of view \cite{hochreiter1997flat,keskar2016large,chaudhari2016entropy,dinh2017sharp,hoffer2017train,pennington2017geometry,pennington2018spectrum,jastrzkebski2018relation,yaida2018fluctuation,geiger2018jamming,spigler2018jamming}. Of particular relevance to us are two recent works that studied the spectrum of the Hessian. In \cite{sagun2016eigenvalues,sagun2017empirical} the authors showed on small-scale networks that the spectrum exhibits a `spiked' behavior, with $C$ outliers isolated from a continuous bulk. In \cite{papyan2018full}, the authors corroborated these findings on modern deepnets with tens of millions of parameters, by applying state-of-the-art tools in modern high-dimensional numerical linear algebra to approximate the full spectrum of the Hessian. They showed that the the outliers can be attributed to the $G$ component, while the majority of the energy in the bulk can be attributed to the $H$ component.

In this work, our goal is to shed light on what is the origin of the outliers observed in $G$. We provide two motivations for this question:
\begin{enumerate}
    \item In \cite{gur2018gradient} the authors analyzed the dynamics of stochastic gradient descent (SGD) \textit{as a function of epochs}. They observed that the gradients of SGD live in a small subspace of rank $C$, spanned by the top eigenvectors of the Hessian, and remarked that utilizing this low-dimensional eigenspace could lead to optimization benefits. These same top eigenvectors of the Hessian were attributed to the $G$ term \mbox{in \cite{papyan2018full}}.
    
    In this paper we show that these outliers are caused by a certain structure in the data underlying $G$. Once this underlying structure is known, we can efficiently compute approximations to the principal subspace. The necessary computations are much simpler even than the power method.
    
    \item In \cite{papyan2018full} the authors initiated the investigation of the separation of the top outliers from the bulk \textit{as a function of sample size}.
    
    In this work we make progress on this question directly by investigating the dynamics of the outliers as a function of sample size. We explain the structure causing the outliers and predict their size without performing eigenanalysis, but rather averaging certain quantities. This provides an alternative to eigenanalysis, which might be easier to analyze and might have better theoretical properties.
\end{enumerate}

\subsection{Contributions}
We commence this work with the observation that $G$ is a second moment matrix and not a Covariance -- the difference between the two being that in the latter a mean is not subtracted from each sample. The aforementioned outliers are a direct sequence of this lack of centering operation and can be computed from the means not being subtracted.

We then show that $G = \frac{1}{n} \Delta \Delta^T$. The rows of $\Delta \in \R^{p \times nC}$ correspond to the coordinates in the space of model parameters and the columns of $\Delta$ can be indexed by three indices, $(i,c,c')$: $i$ corresponds to the index of a sample in a certain class, $c$ corresponds to the class, and $c'$ corresponds to a logit coordinate. Given this indexing, each column in $\Delta$ can be denoted by $\delta_{i,c,c'}$. The $i$-th sample in the $c$-th class has $C$ columns in $\Delta$ associated with it, $\{ \delta_{i,c,c'} \}_{c'}$. These correspond to the $C$ logit coordinates. We depict the matrix $\Delta$ and its partitioning in Figure \ref{fig:Delta}.

This indexing naturally partitions the columns in $\Delta$ into $C^2$ groups -- one for each combination of class $c$ and logit coordinate $c'$. Each of these groups can be characterized by a \textit{group} mean $\delta_{c,c'}$ and a Covariance $\Sigma_{c,c'}$, which are computed from of all the columns that fall into it, $\{ \delta_{i,c,c'} \}_i$. The collection of all group means $\{ \delta_{c,c'} \}_{c' \neq c}$ associated with the same class $c$, but different logit coordinates $c'$, can be considered a \textit{cluster}, characterized by its mean $\delta_c$ and Covariance $\Sigma_c$.

\begin{figure}[t!]
    \centering
    \includegraphics[width=0.5\textwidth]{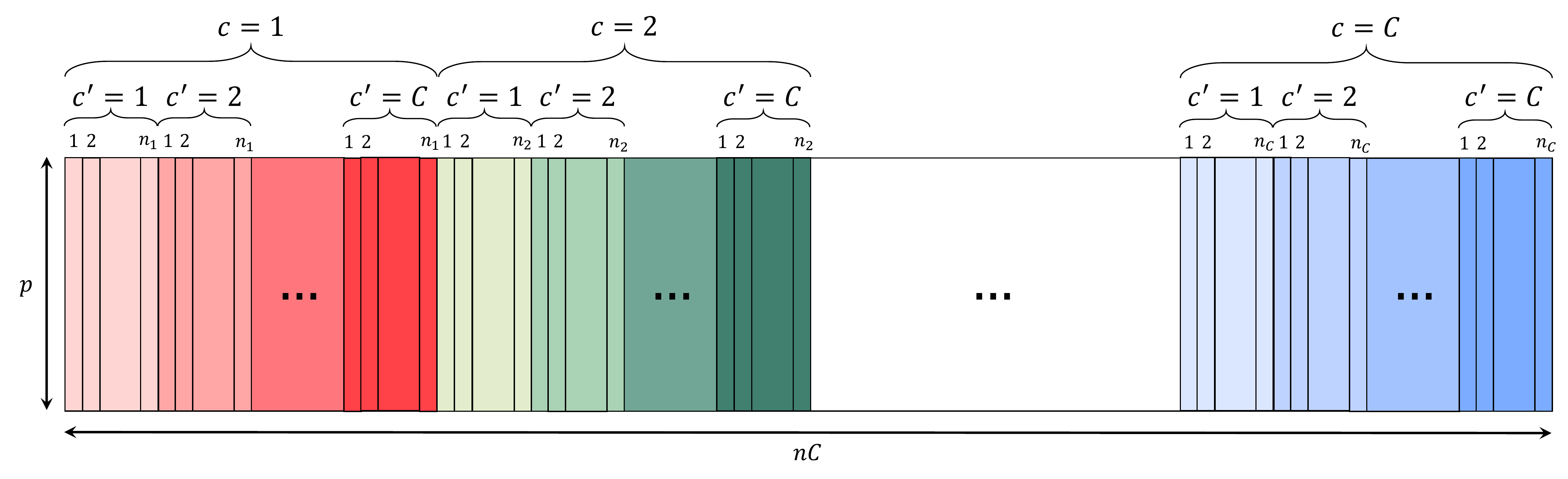}
    \caption{\textit{Partitioning of the columns in $\Delta \in \R^{p \times nC}$.} The columns can be indexed by three indices, $(i,c,c')$: $i$ corresponds to the index of a sample in a certain class, $c$ corresponds to the class, and $c'$ corresponds to a logit coordinate. Given this indexing, each column is denoted by $\delta_{i,c,c'}$.
    } \label{fig:Delta}
\end{figure}

Intuitively, we think of $\{ \delta_{i,c,c'} \}_i$ as members of a \textit{group} with a group mean $\delta_{c,c'}$ and Covariance $\Sigma_{c,c'}$. Moreover, we think of the group means $\{ \delta_{c,c'} \}_{c' \neq c}$ as being themselves members of a \textit{cluster} with a cluster center $\delta_c$ and Covariance $\Sigma_c$. Figure \ref{fig:definitions} illustrates this intuition while summarizing the above-mentioned definitions. This figure also defines other objects ($G_0, \dots, G_3$) that will be introduced in the next sections.

Our main finding in this work is that the top-$C$ outliers in the spectrum of $G$ can be approximated from the eigenvalues of the matrix $G_1 = \Ave_c \{ \delta_c \delta_c^T \}$. Equally, these could be approximated from the Gram of cluster centers $\{ \delta_c \}_c$.

We show that the cluster centers $\{ \delta_c \}_c$ are far apart and the cluster members $\{ \delta_{c,c'} \}_{c' \neq c}$ are tightly scattered around the cluster center. In other words, the \textit{within-cluster variation} is small compared to the \textit{between-cluster variation}. This configuration makes the outliers in $G$ attributable to the Gram of the cluster centers. We illustrate this phenomenon in Figures \ref{fig:tSNE} and \ref{fig:imagenet}, showing t-SNE \cite{maaten2008visualizing} plots of the cluster members $\{ \delta_{c,c'} \}_{c,c'}$ and the cluster centers $\{ \delta_c \}_c$.

We also investigate this phenomenon throughout the epochs of SGD. Figure \ref{fig:imagenet_epochs} shows t-SNE plots of the cluster members $\{ \delta_{c,c'} \}_{c,c'}$ in different epochs. We observe that the cluster members $\{ \delta_{c,c'} \}_{c,c'}$ cluster around the cluster centers $\{ \delta_c \}_c$ only after a certain number of epochs. Prior to that the cluster members $\{ \delta_{c,c'} \}_{c,c'}$ cluster according to the logit coordinate $c'$ and not the true class $c$.

We substantiate our claims, regarding a connection between the outliers in $G$ and the corresponding eigenvalues of the Gram of cluster centers, by testing them empirically across different canonical datasets, contest winning architectures and various training sample sizes. We observe that the top-$C$ outliers in $G$ deviate from their predicted value by a small margin. This phenomenon is well known in the context of Random Matrix Theory (RMT), where the magnitude of such deviations can be computed using dedicated tools.

\begin{figure}[t]
    \centering
    \includegraphics[width=0.35\textwidth]{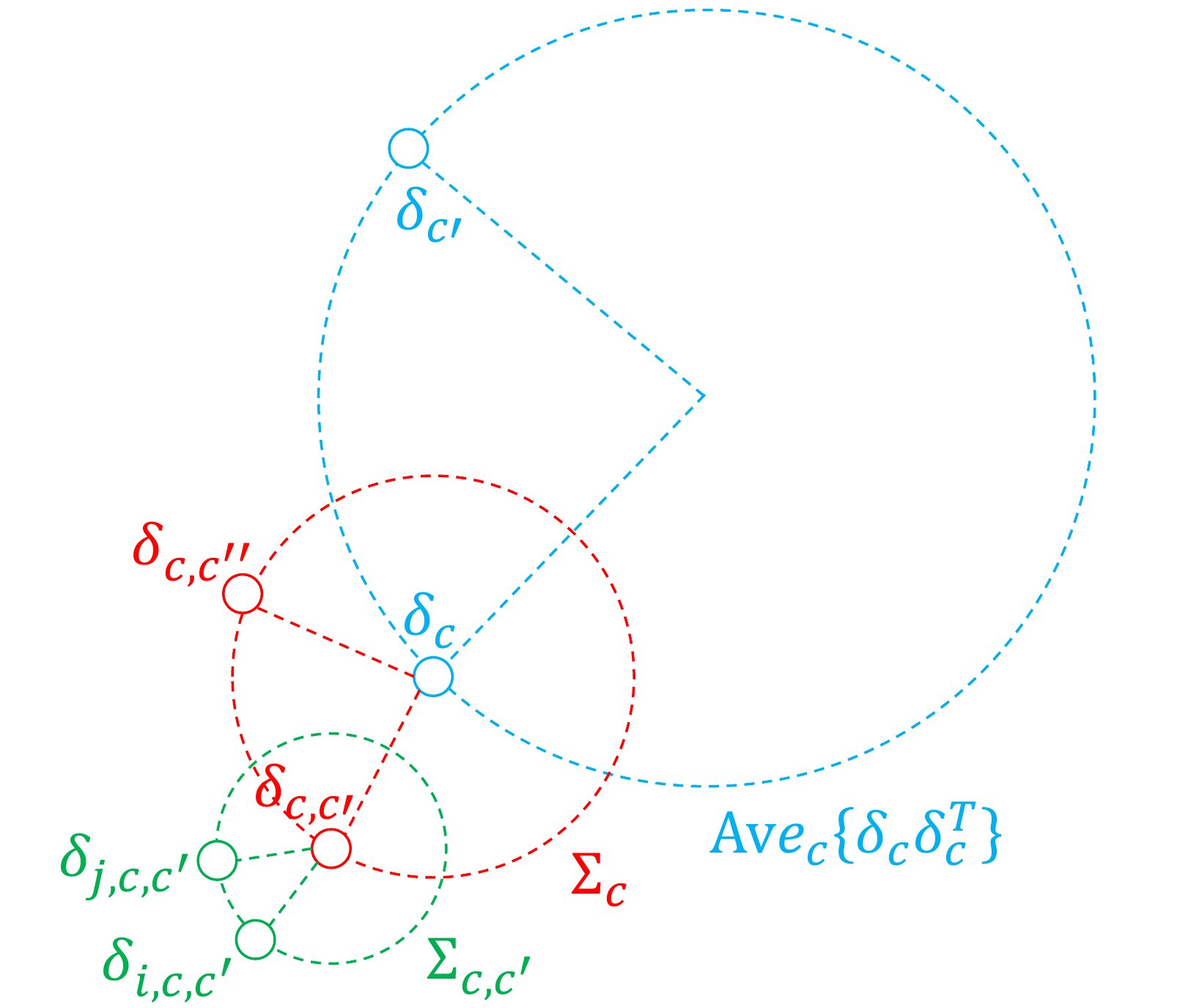}
    \caption{\textit{Three-level hierarchical decomposition of the second moment matrix $G=G_0+G_1+G_2+G_3$.} The coarsest level, depicted in \textbf{{\color{Cerulean}blue}}, is comprised of $\{ \delta_c \}_c$, whose second moment is given by $G_1 = \Ave_c \{ \delta_c \delta_c^T \}$. The top-$C$ outliers observed in the spectrum of $G$ are due to eigenvalues of this matrix. The middle level, depicted in \textbf{{\color{Red}red}}, is composed of $\{ \delta_{c,c'} \}_{c,c'}$. For a certain $c$, $\{ \delta_{c,c'} \}_{c' \neq c}$ are centered around $\delta_c$ and their Covariance is $\Sigma_c$. Averaging this Covariance over all classes results in $G_2 = \Ave_c \{ \Sigma_c \}$. The finest level, depicted in \textbf{{\color{ForestGreen}green}}, includes all the logit derivatives $\{ \delta_{i,c,c'} \}_i$. For a given pair of $c$ and $c'$, $\{ \delta_{i,c,c'} \}_i$ are centered around $\delta_{c,c'}$ and their Covariance is $\Sigma_{c,c'}$. Averaging this Covariance over all pairs of classes gives $G_3 = \frac{1}{C} \sum_{c,c'} \Sigma_{c,c'}$. The elements $\{ \delta_{c,c} \}_c$ are not plotted above, however their location is at the point zero.
    } \label{fig:definitions}
\end{figure}

\begin{figure*}
    \centering
    \begin{subfigure}[t]{0.33\textwidth}
        \centering
        \includegraphics[width=1\textwidth]{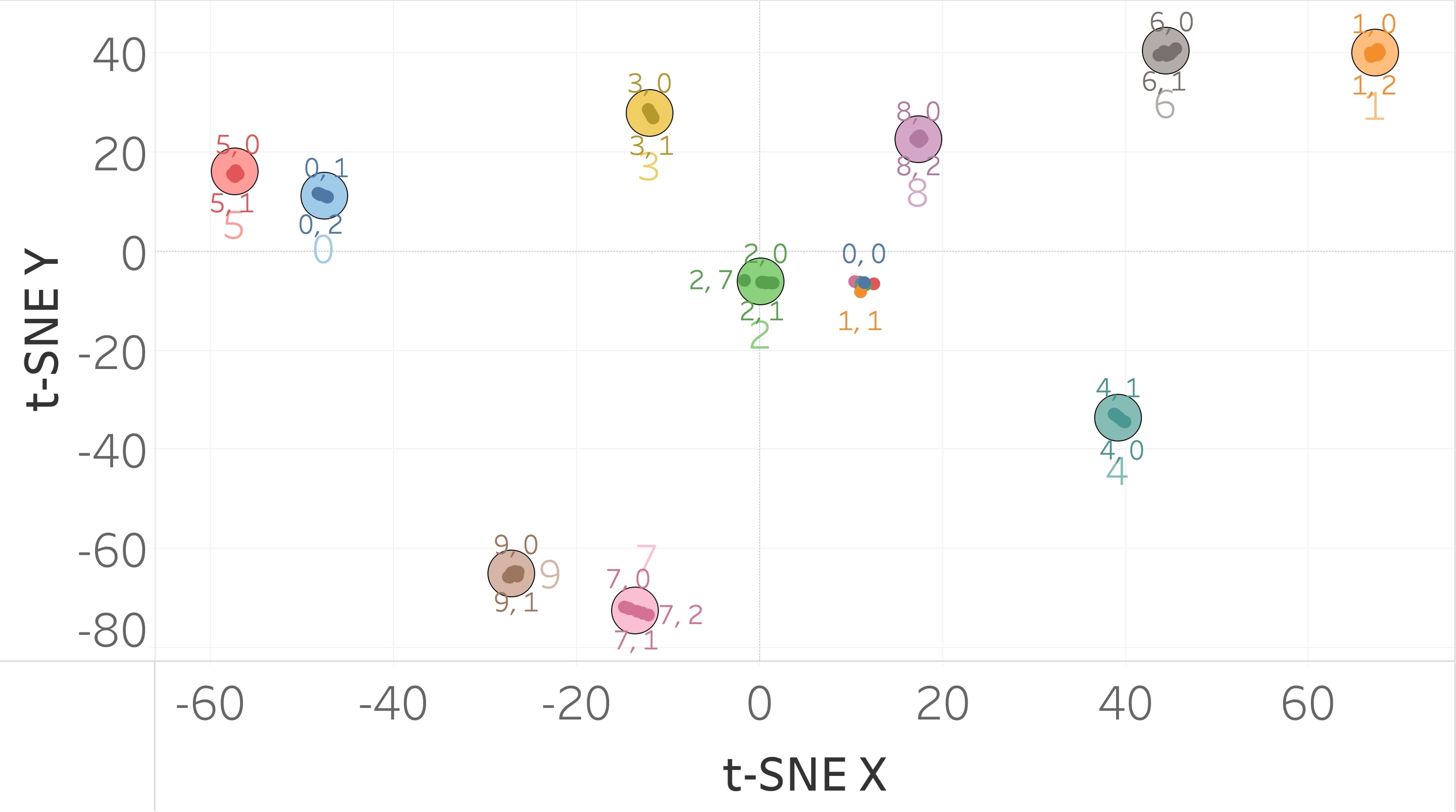}
        \caption{MNIST, 13 examples per class.}
    \end{subfigure}%
    ~
    \begin{subfigure}[t]{0.33\textwidth}
        \centering
        \includegraphics[width=1\textwidth]{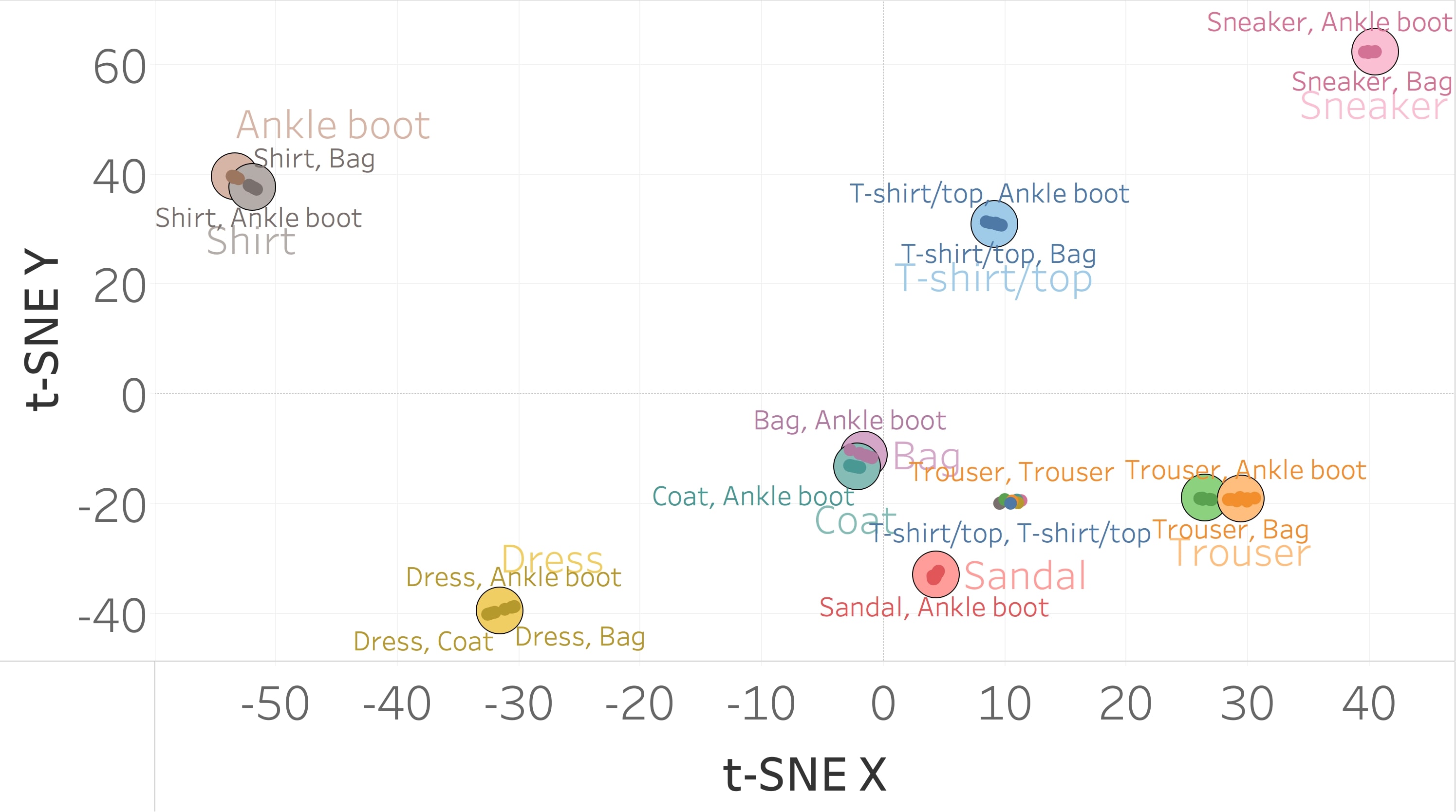}
        \caption{Fashion MNIST, 13 examples per class.}
    \end{subfigure}%
    ~
    \begin{subfigure}[t]{0.33\textwidth}
        \centering
        \includegraphics[width=1\textwidth]{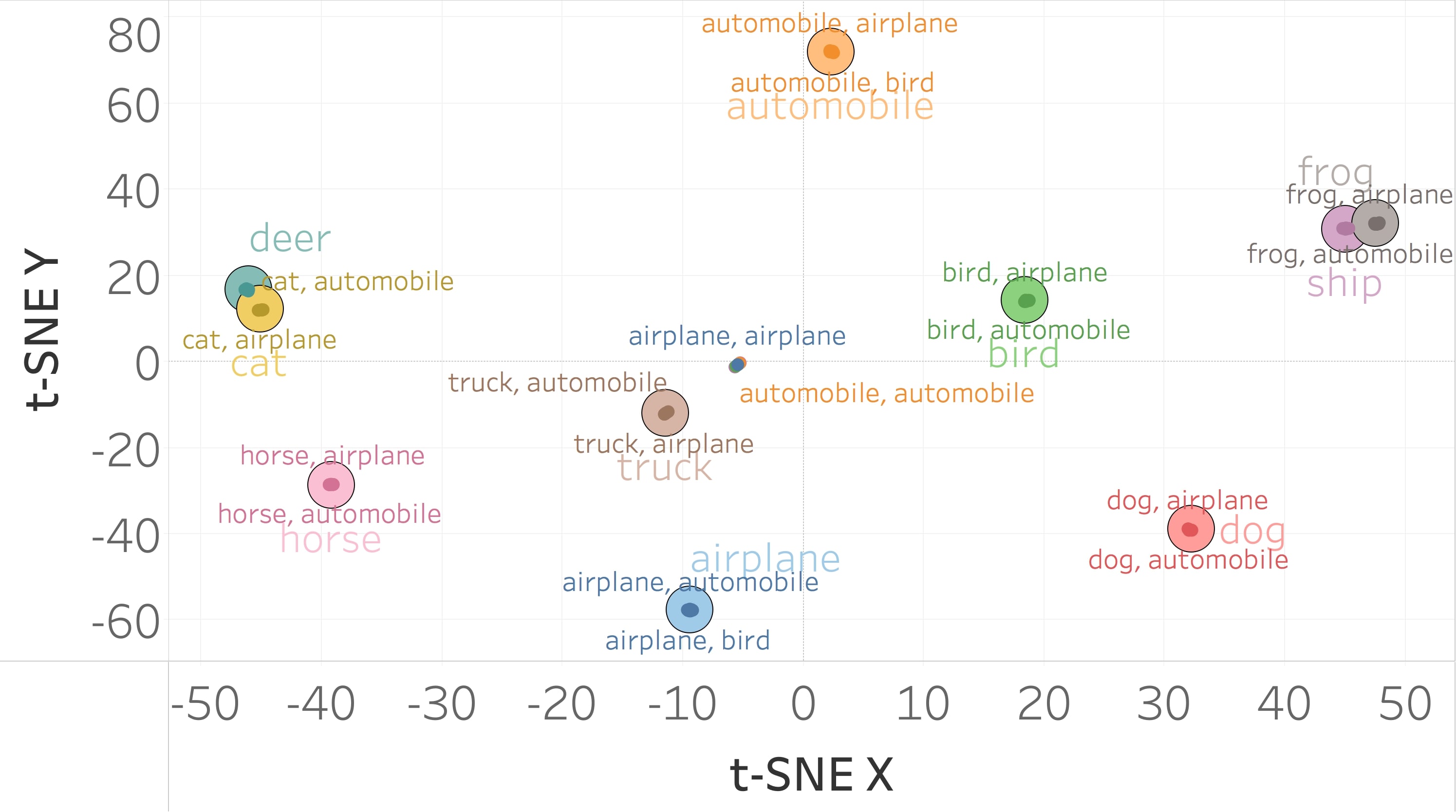}
        \caption{CIFAR10, 13 examples per class.}
    \end{subfigure}
    
    
    
    \vspace{0.25cm}
    
    \centering
    \begin{subfigure}[t]{0.33\textwidth}
        \centering
        \includegraphics[width=1\textwidth]{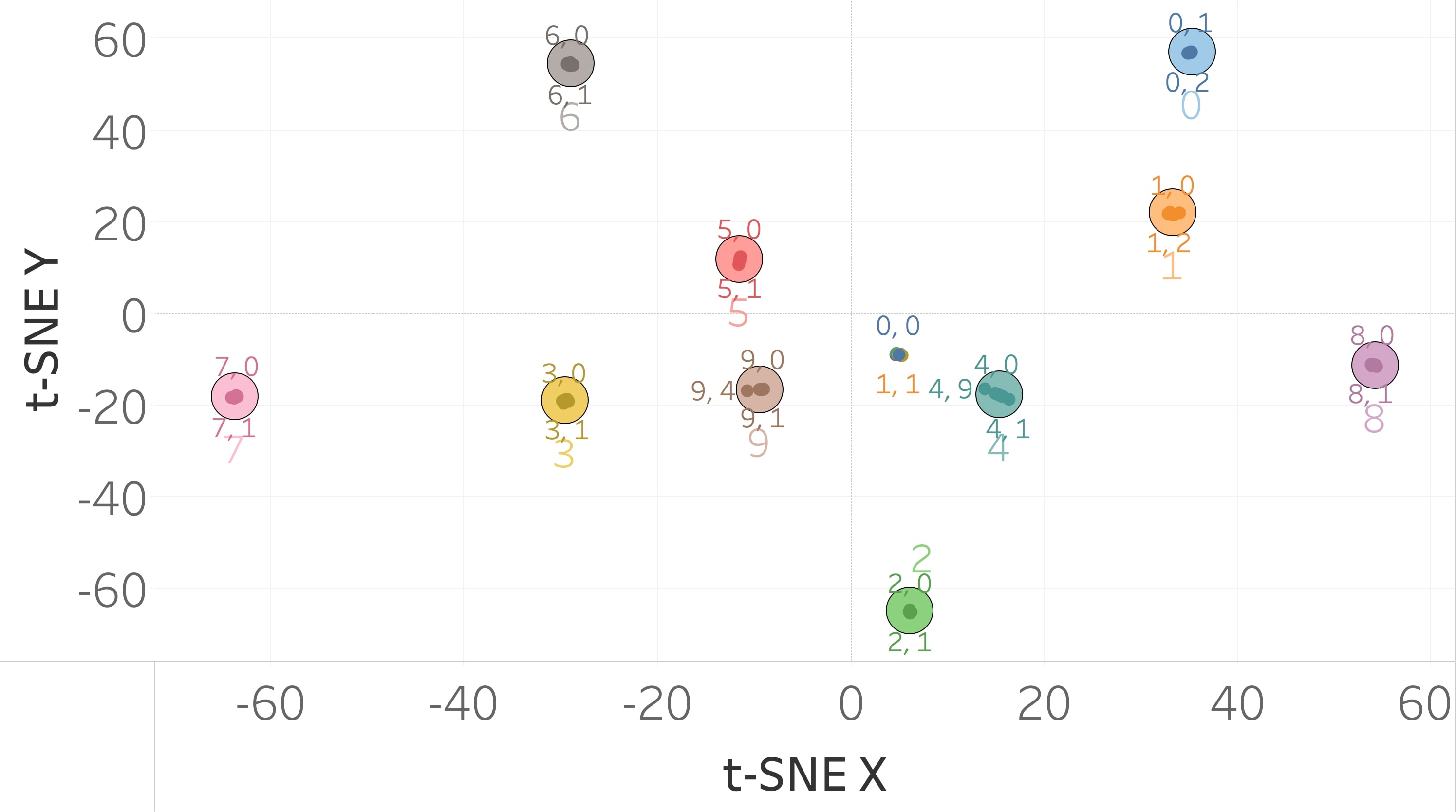}
        \caption{MNIST, 702 examples per class.}
    \end{subfigure}%
    ~
    \begin{subfigure}[t]{0.33\textwidth}
        \centering
        \includegraphics[width=1\textwidth]{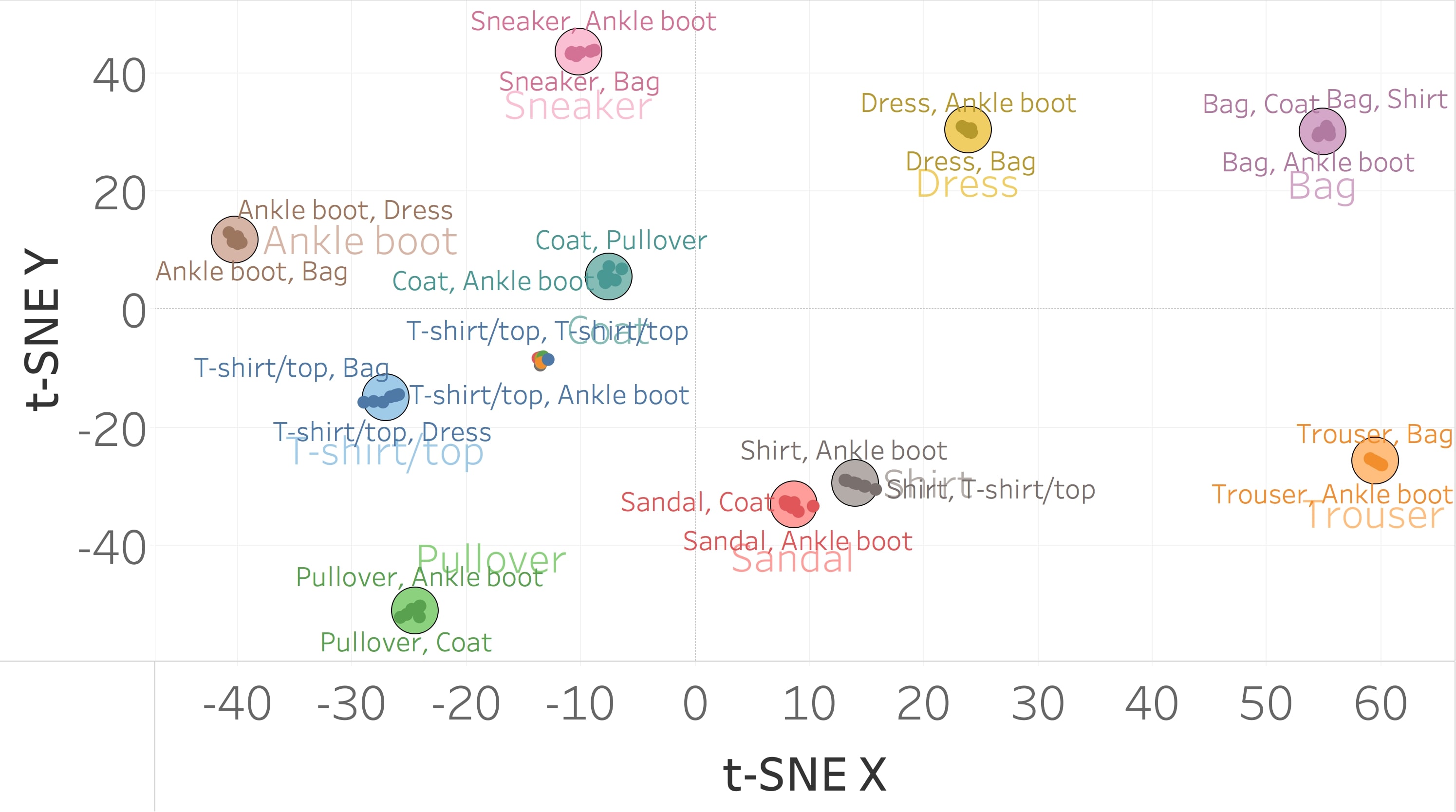}
        \caption{Fashion MNIST, 702 examples per class.}
    \end{subfigure}%
    ~
    \begin{subfigure}[t]{0.33\textwidth}
        \centering
        \includegraphics[width=1\textwidth]{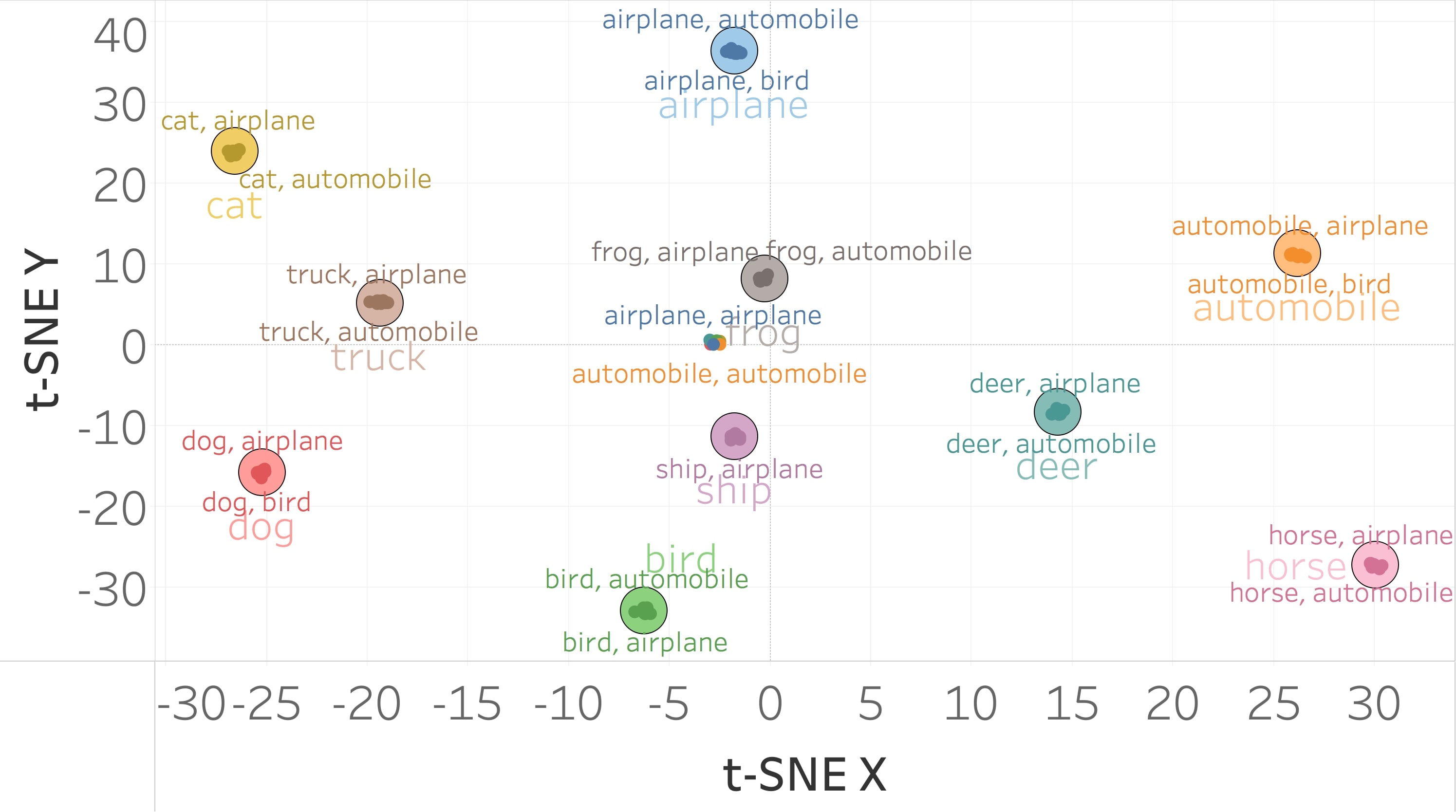}
        \caption{CIFAR10, 702 examples per class.}
    \end{subfigure}
    
    
    
    \vspace{0.25cm}
    
    \centering
    \begin{subfigure}[t]{0.33\textwidth}
        \centering
        \includegraphics[width=1\textwidth]{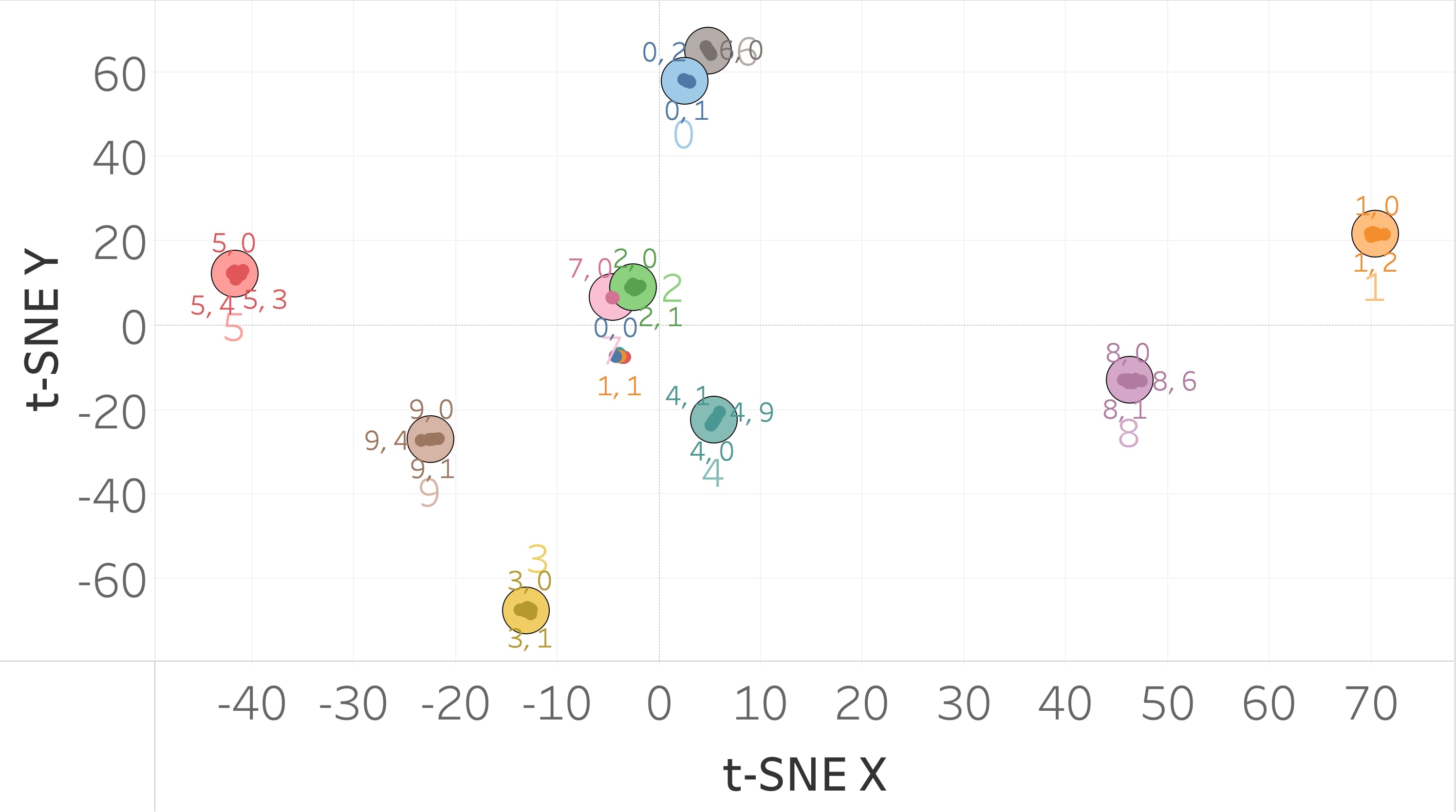}
        \caption{MNIST, 5000 examples per class.}
    \end{subfigure}%
    ~
    \begin{subfigure}[t]{0.33\textwidth}
        \centering
        \includegraphics[width=1\textwidth]{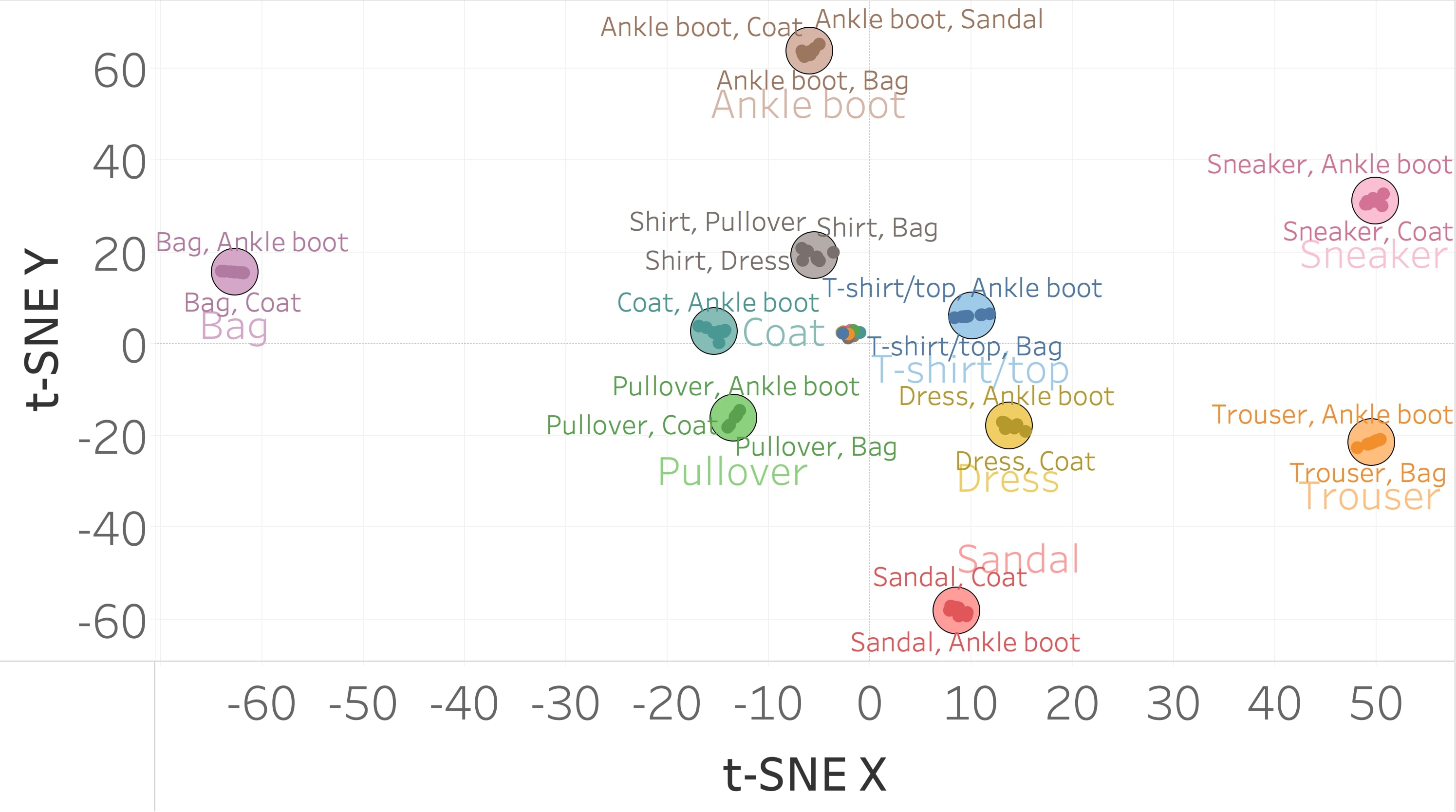}
        \caption{Fashion MNIST, 5000 examples per class.}
    \end{subfigure}%
    ~
    \begin{subfigure}[t]{0.33\textwidth}
        \centering
        \includegraphics[width=1\textwidth]{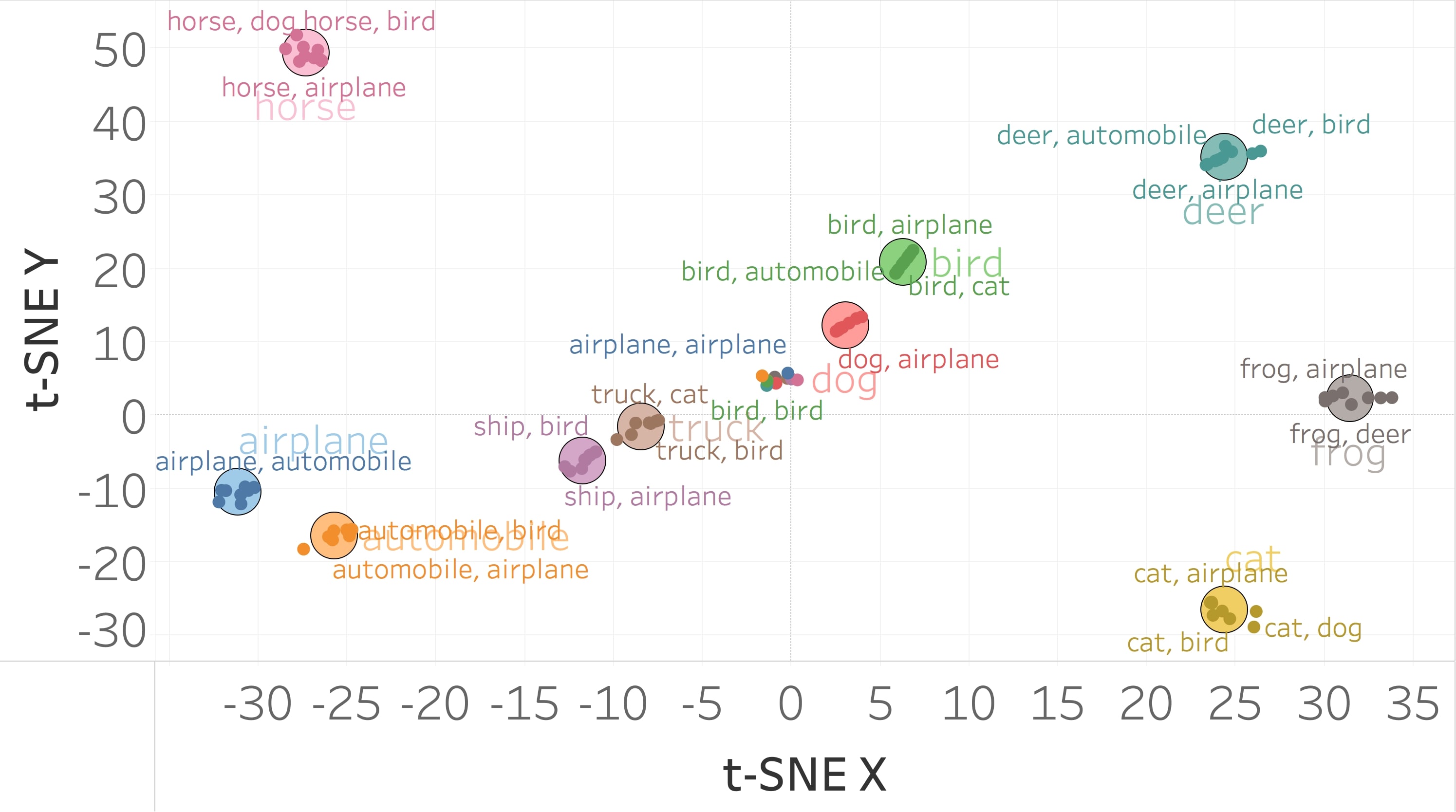}
        \caption{CIFAR10, 5000 examples per class.}
    \end{subfigure}
    \caption{\textit{t-SNE visualization of the hierarchical structure for the DenseNet40 architecture.} Each column of panels corresponds to a different dataset, and each row to a different sample size. Each panel depicts the two-dimensional t-SNE embedding of the cluster members $\{ \delta_{c,c'} \}_{c,c'}$ and the cluster centers $\{ \delta_c \}_c$. All circles are colored according to the class $c$. The $\delta_c$ are marked with large circles and have a label, written in large font, attached to them. The $\delta_{c,c'}$ are marked with small circles and a subset of them also have a label, written in smaller font, attached to them. The label is a concatenation of the two class names corresponding to $c$ and $c'$. This plot asserts the three level hierarchy. At level one we have the cluster centers $\{ \delta_c \}_c$. At level two, next to each cluster center $\delta_c$, we find cluster members $\{ \delta_{c,c'} \}_{c' \neq c}$. Although not plotted, at level three, next to each $\delta_{c,c'}$ we would find $\{ \delta_{i,c,c'} \}_i$. We also observe a cluster which contains all the $\{ \delta_{c,c} \}_c$ ($c=c'$). These points are clustered together because their norms are close to zero, when compared to the other points.
    } \label{fig:tSNE}
\end{figure*}

\begin{figure*}
    \centering
    \begin{subfigure}[t]{0.4\textwidth}
        \centering
        \includegraphics[width=1\textwidth]{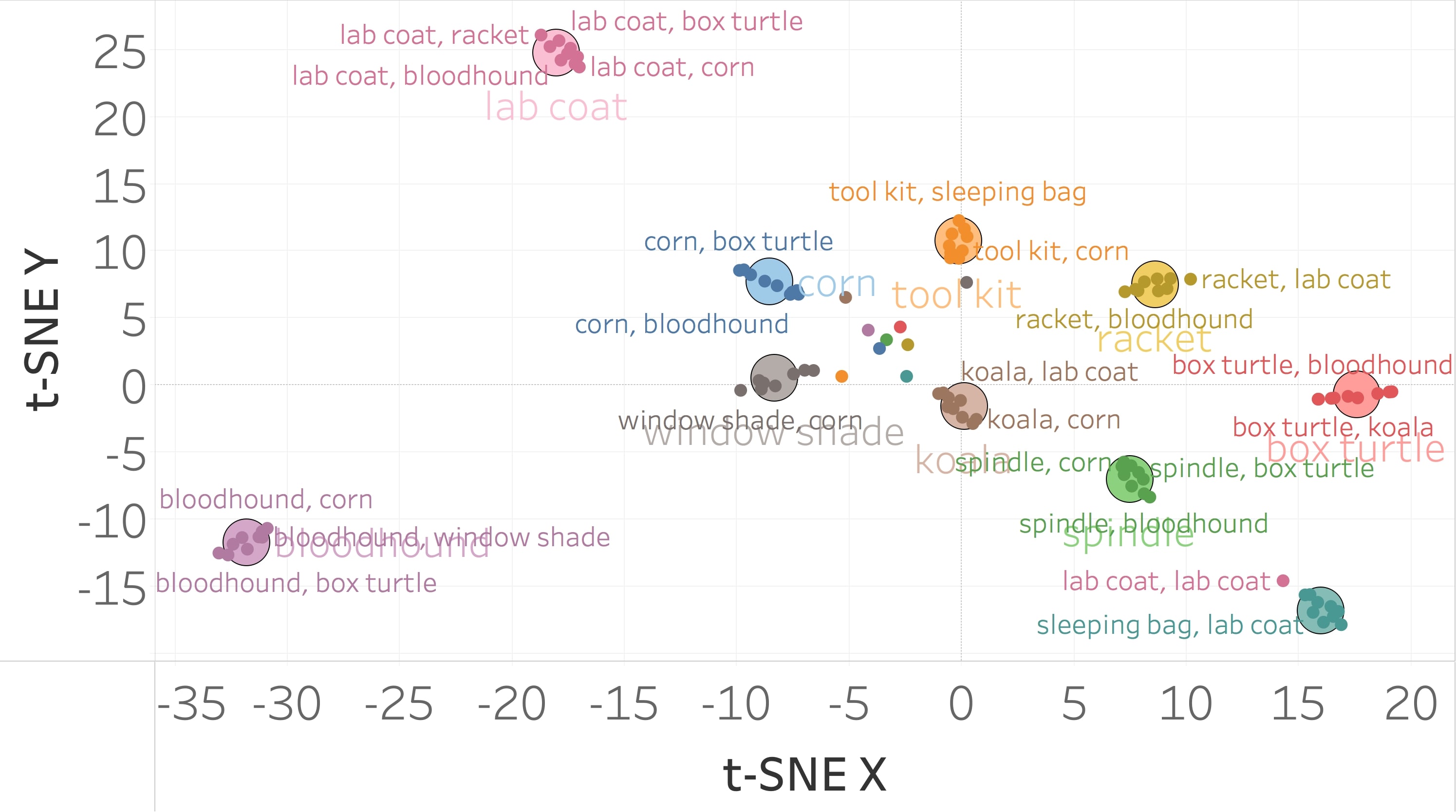}
        \caption{VGG16.}
    \end{subfigure}%
    ~
    \begin{subfigure}[t]{0.4\textwidth}
        \centering
        \includegraphics[width=1\textwidth]{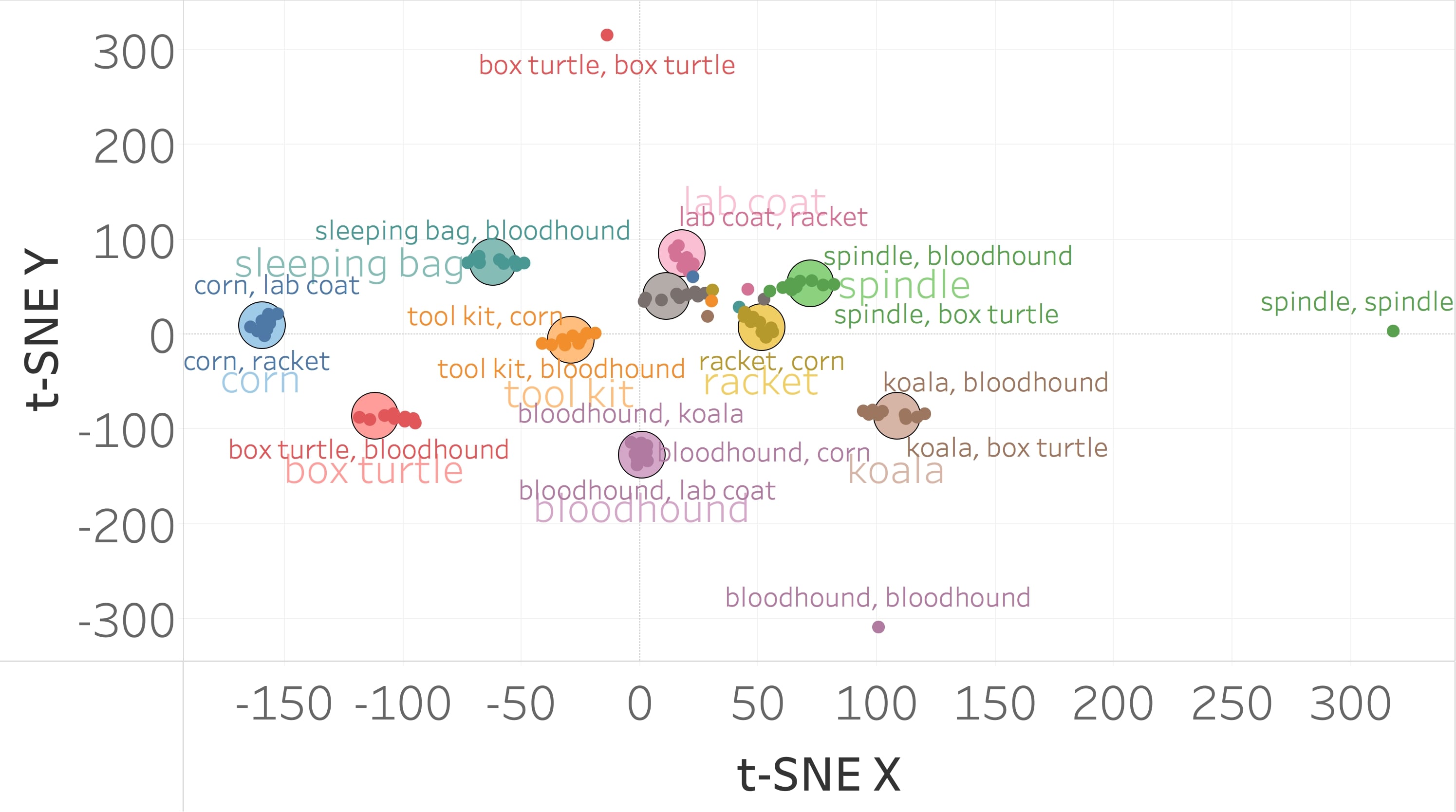}
        \caption{ResNet50.}
    \end{subfigure}
    \caption{\textit{t-SNE visualization of the hierarchical structure in ImageNet.} Each panel depicts the two-dimensional t-SNE embedding of the cluster members and the cluster centers for a different architecture. For further details, see caption of Figure \ref{fig:tSNE}. VGG16 was trained on 600 examples per class and ResNet50 on the full dataset. For visualization purposes, we subset randomly ten classes. By and large, $\delta_{c,c'}$ with fixed $c$ and varying $c'$ cluster around $\delta_c$, except for ``oddballs'' -- which turn out to correspond to $c=c'$. Unlike Figure \ref{fig:tSNE}, these are not tightly clustered together. This is likely due to their norm not being close to zero, caused by a lack of convergence of SGD.} \label{fig:imagenet}
\end{figure*}

\begin{figure*} [t!]
    \centering
    \begin{subfigure}[t]{0.19\textwidth}
        \centering
        \includegraphics[width=1\textwidth]{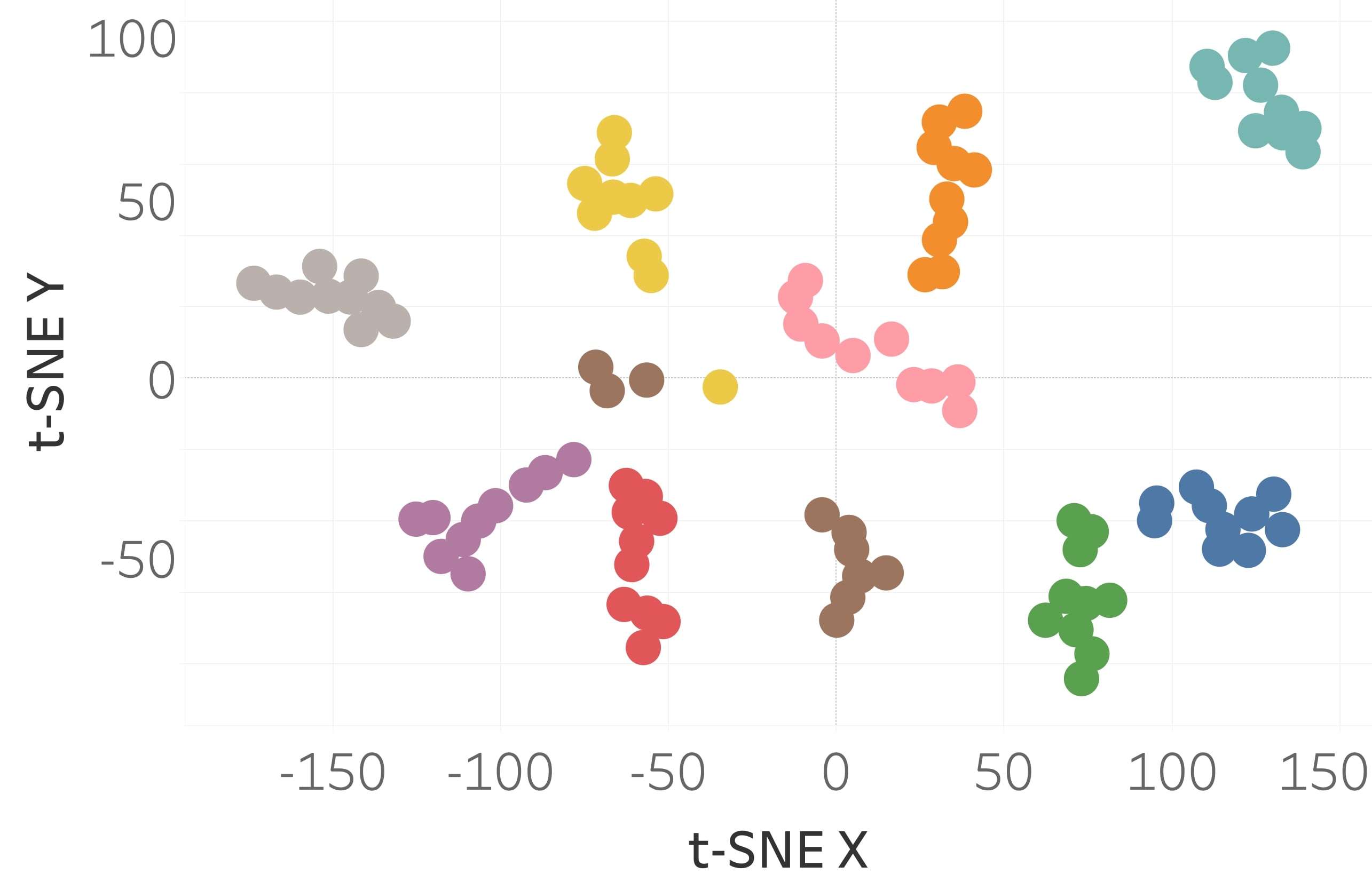}
        \caption{ResNet50, epoch 1, colored according to $c'$.}
    \end{subfigure}%
    ~
    \begin{subfigure}[t]{0.19\textwidth}
        \centering
        \includegraphics[width=1\textwidth]{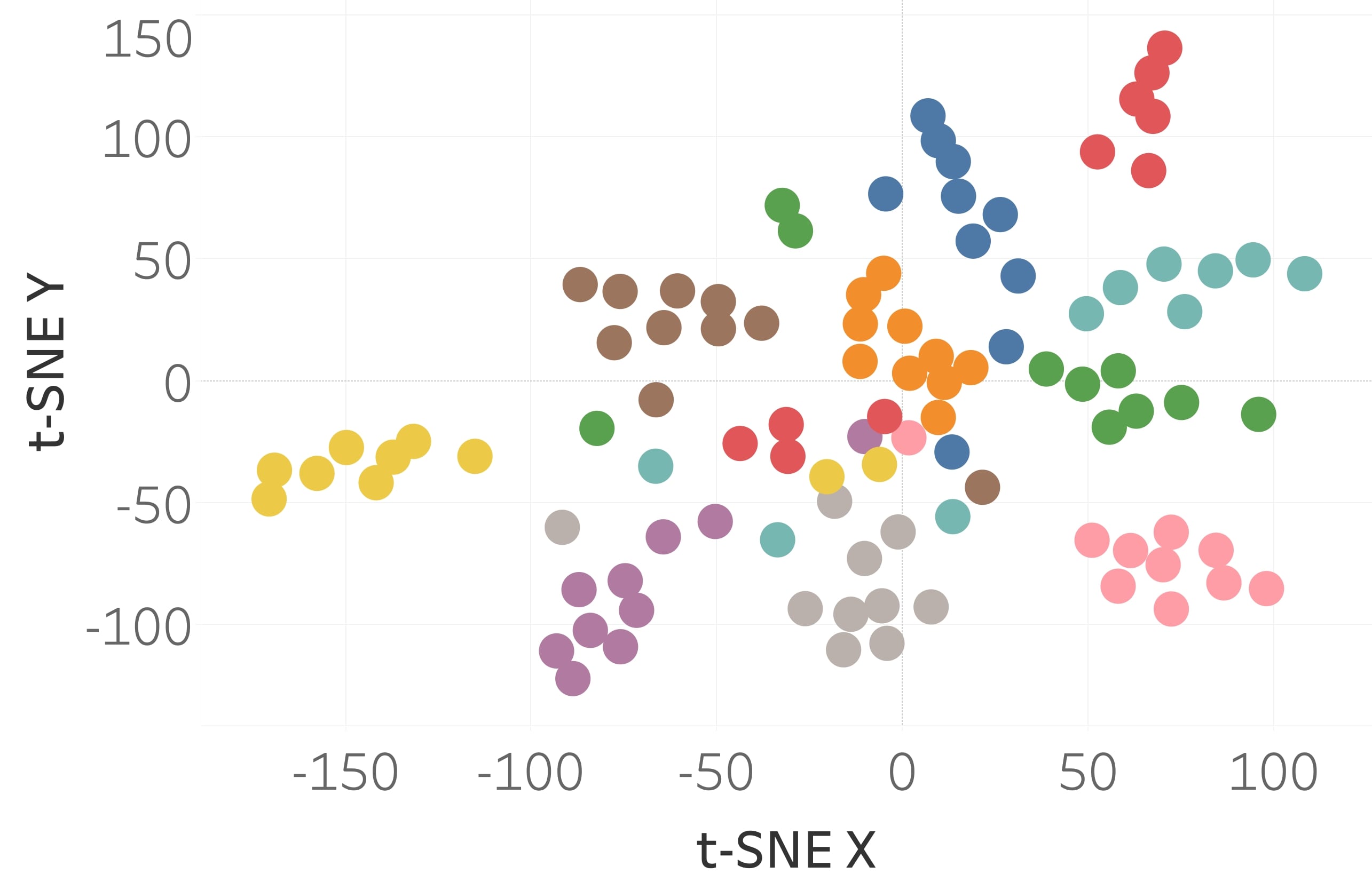}
        \caption{ResNet50, epoch 10, colored according to $c'$.}
    \end{subfigure}%
    ~
    \begin{subfigure}[t]{0.19\textwidth}
        \centering
        \includegraphics[width=1\textwidth]{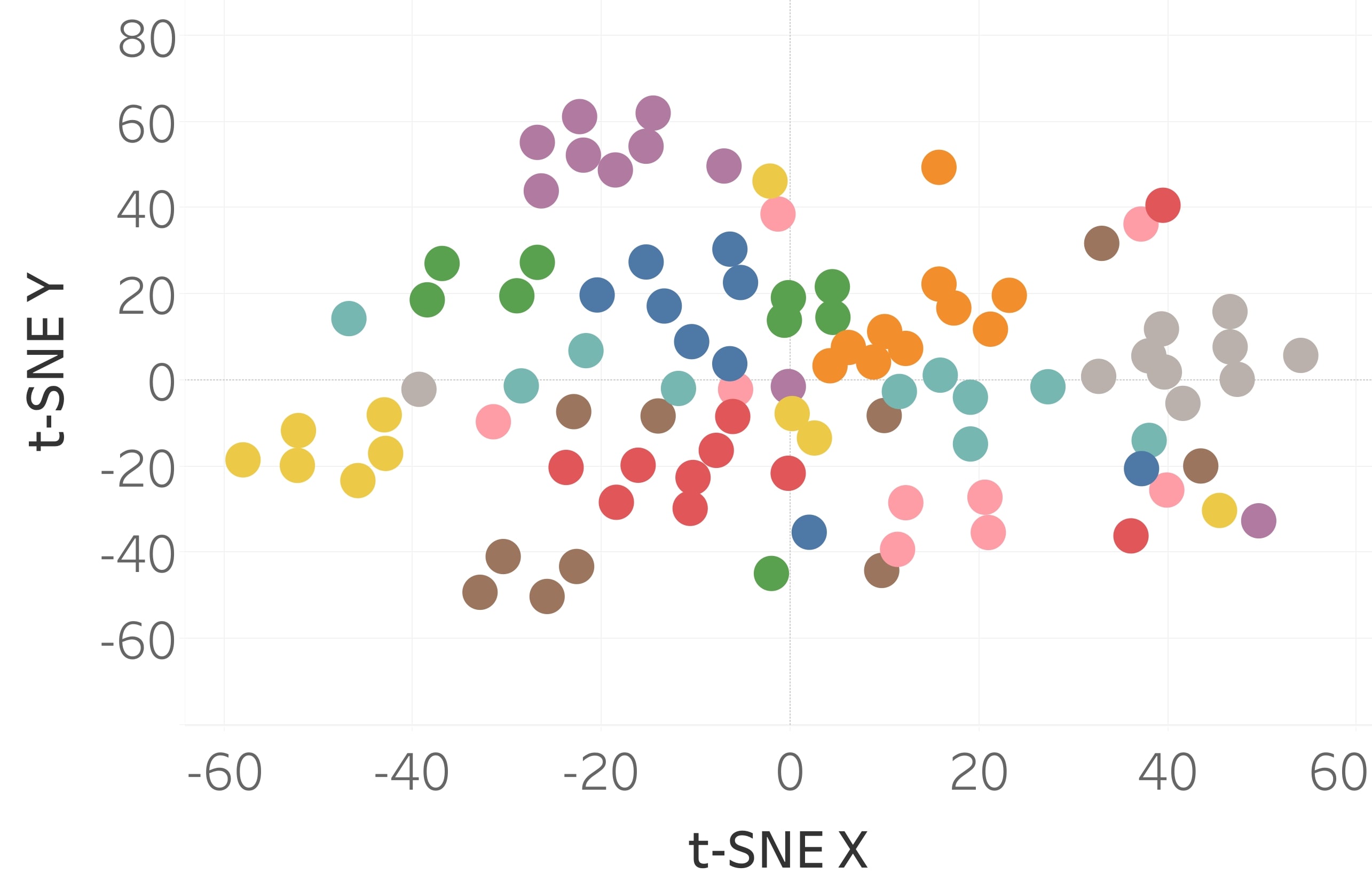}
        \caption{ResNet50, epoch 18, colored according to $c$.}
    \end{subfigure}%
    ~
    \begin{subfigure}[t]{0.19\textwidth}
        \centering
        \includegraphics[width=1\textwidth]{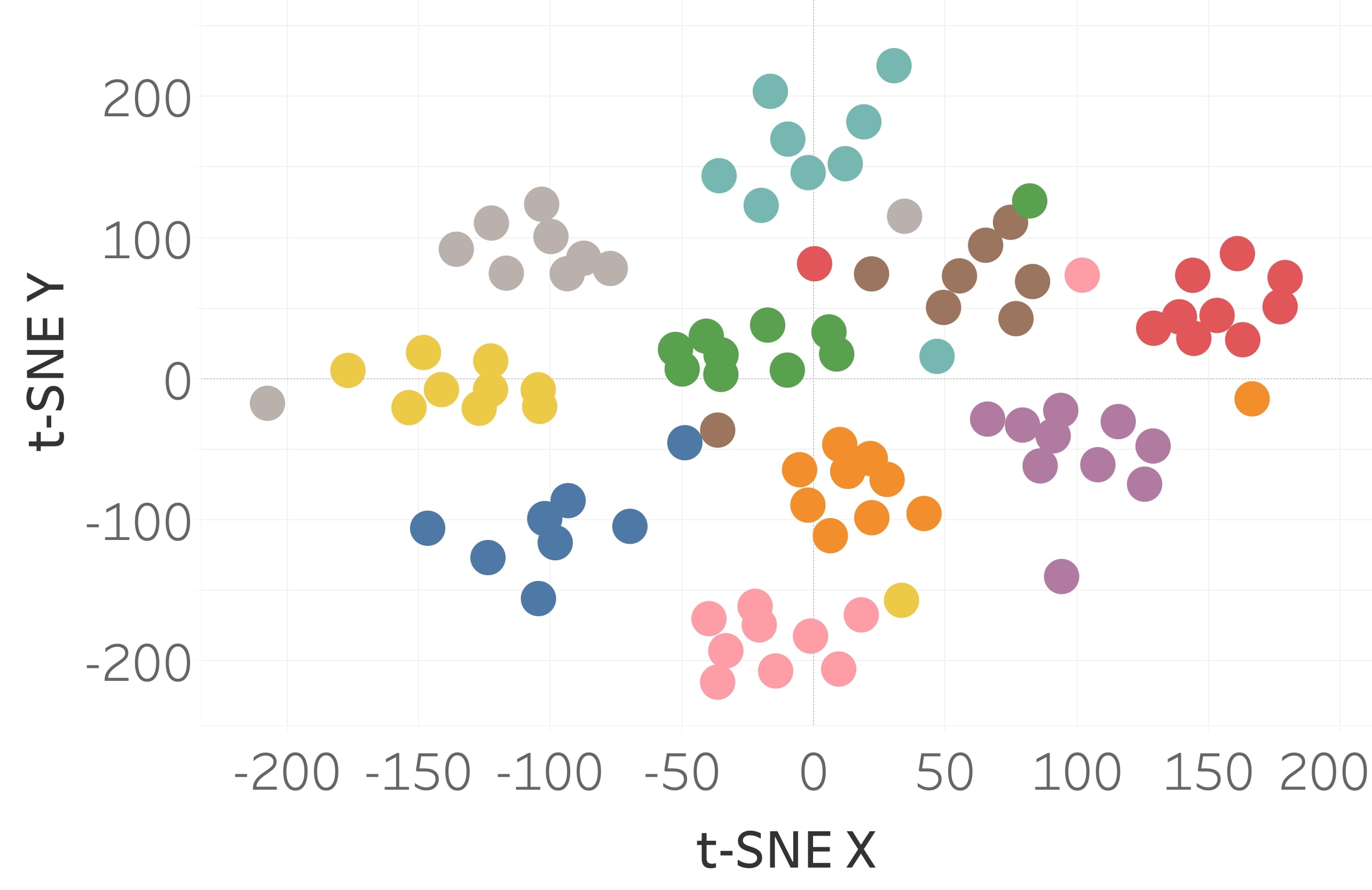}
        \caption{ResNet50, epoch 33, colored according to $c$.}
    \end{subfigure}%
    ~
    \begin{subfigure}[t]{0.19\textwidth}
        \centering
        \includegraphics[width=1\textwidth]{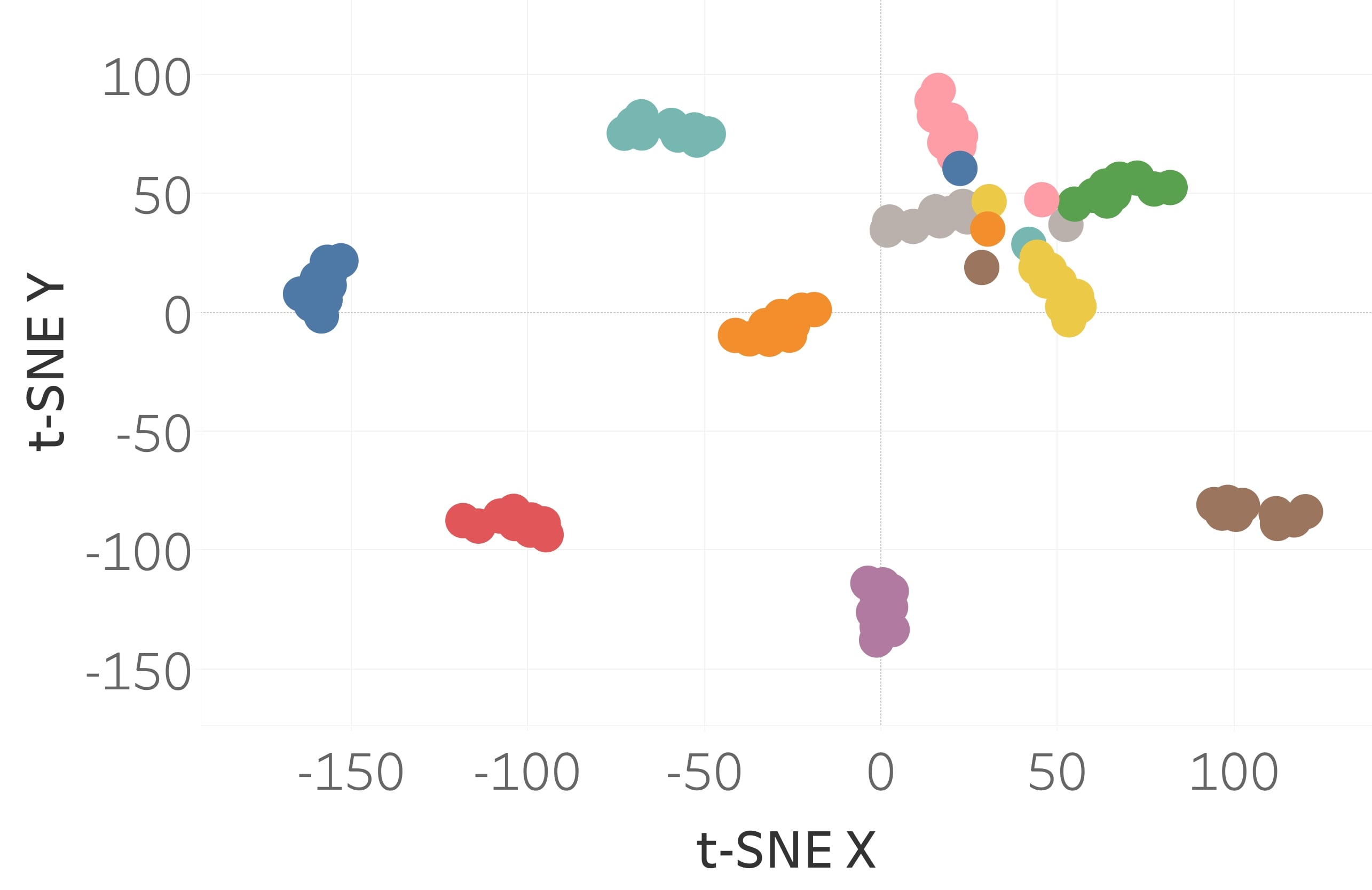}
        \caption{ResNet50, epoch 350, colored according to $c$.}
    \end{subfigure}
    
    \centering
    \begin{subfigure}[t]{0.19\textwidth}
        \centering
        \includegraphics[width=1\textwidth]{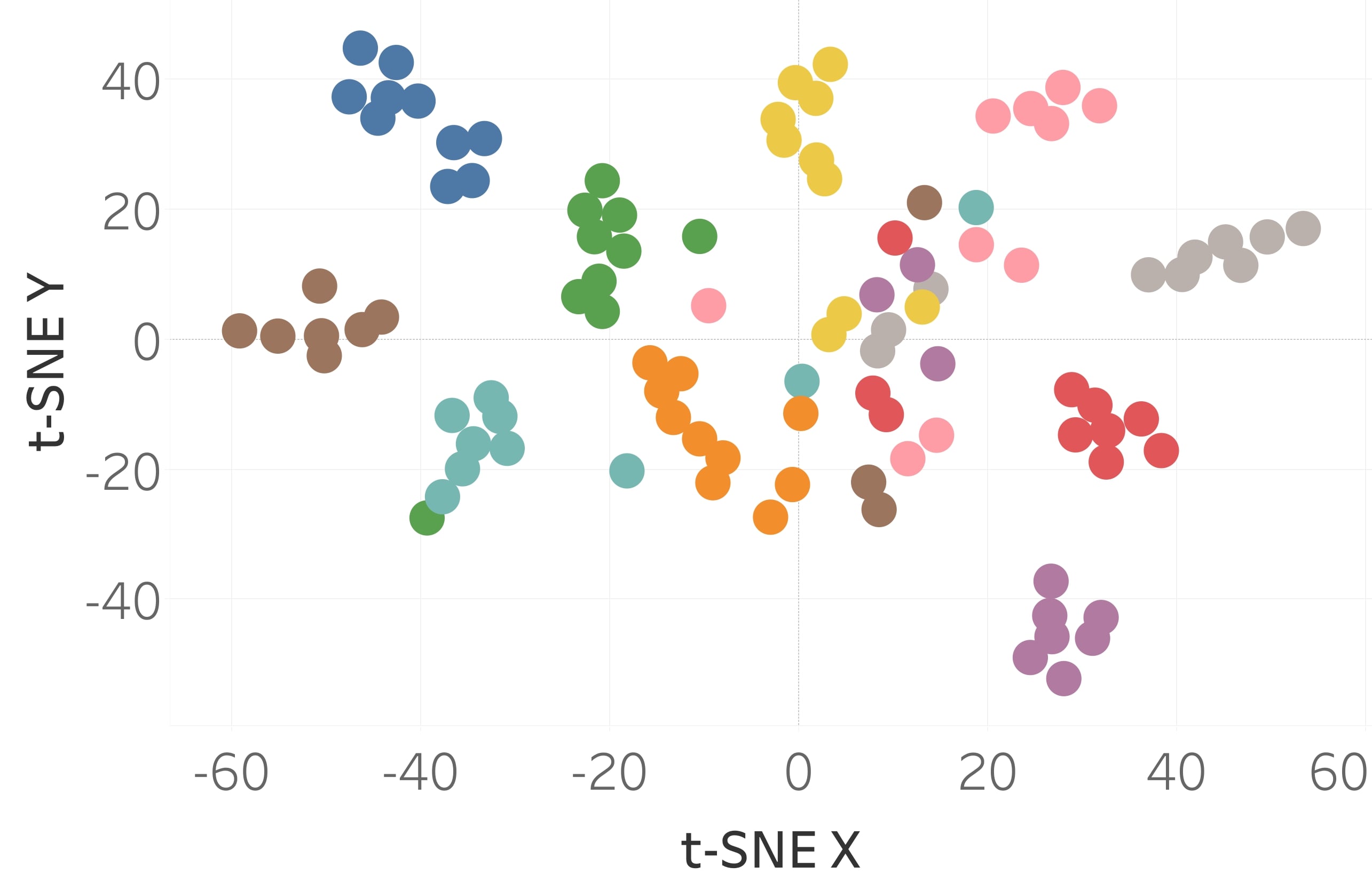}
        \caption{VGG16, epoch 1, colored according to $c'$.}
    \end{subfigure}%
    ~
    \begin{subfigure}[t]{0.19\textwidth}
        \centering
        \includegraphics[width=1\textwidth]{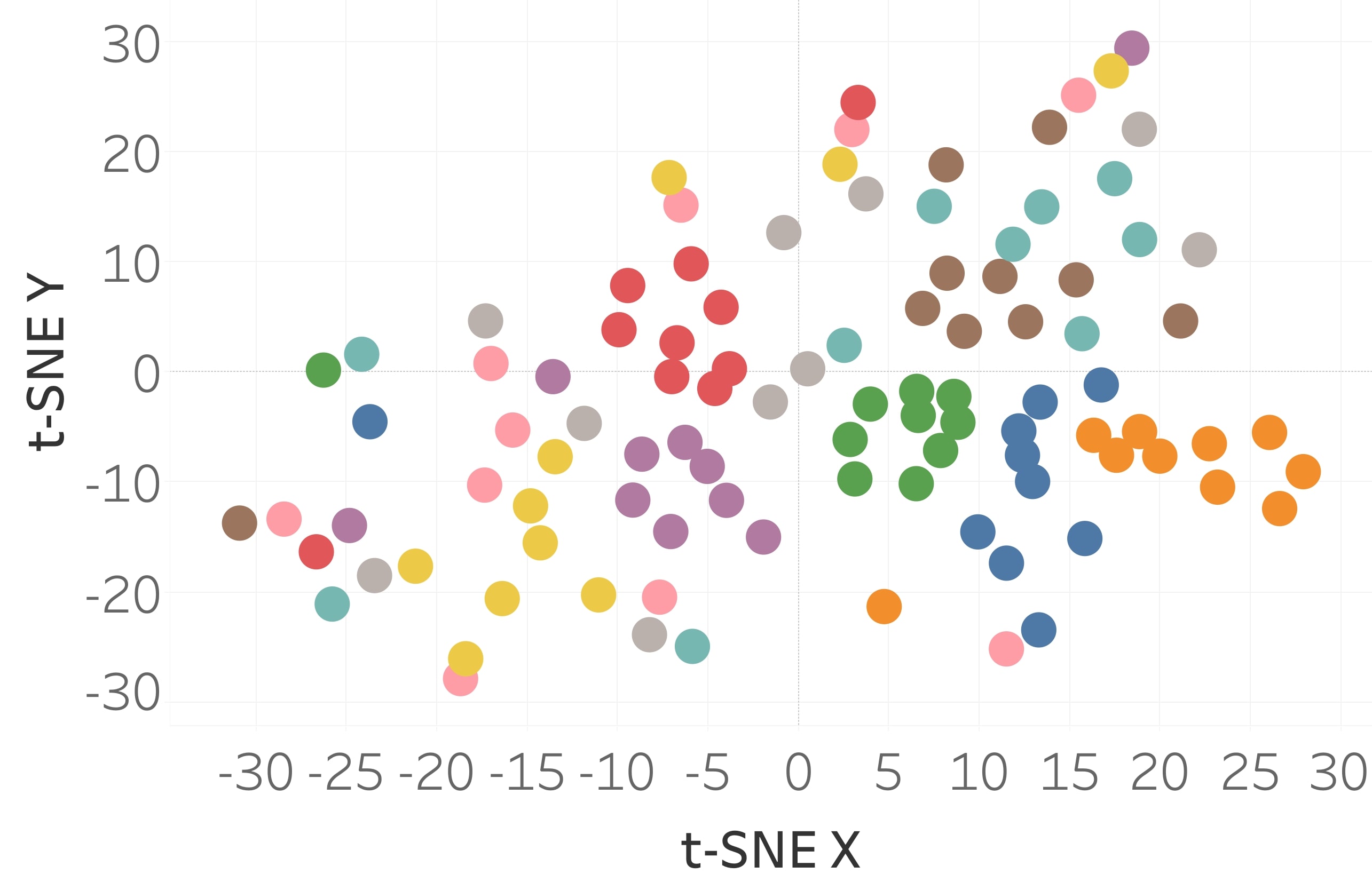}
        \caption{VGG16, epoch 10, colored according to $c$.}
    \end{subfigure}%
    ~
    \begin{subfigure}[t]{0.19\textwidth}
        \centering
        \includegraphics[width=1\textwidth]{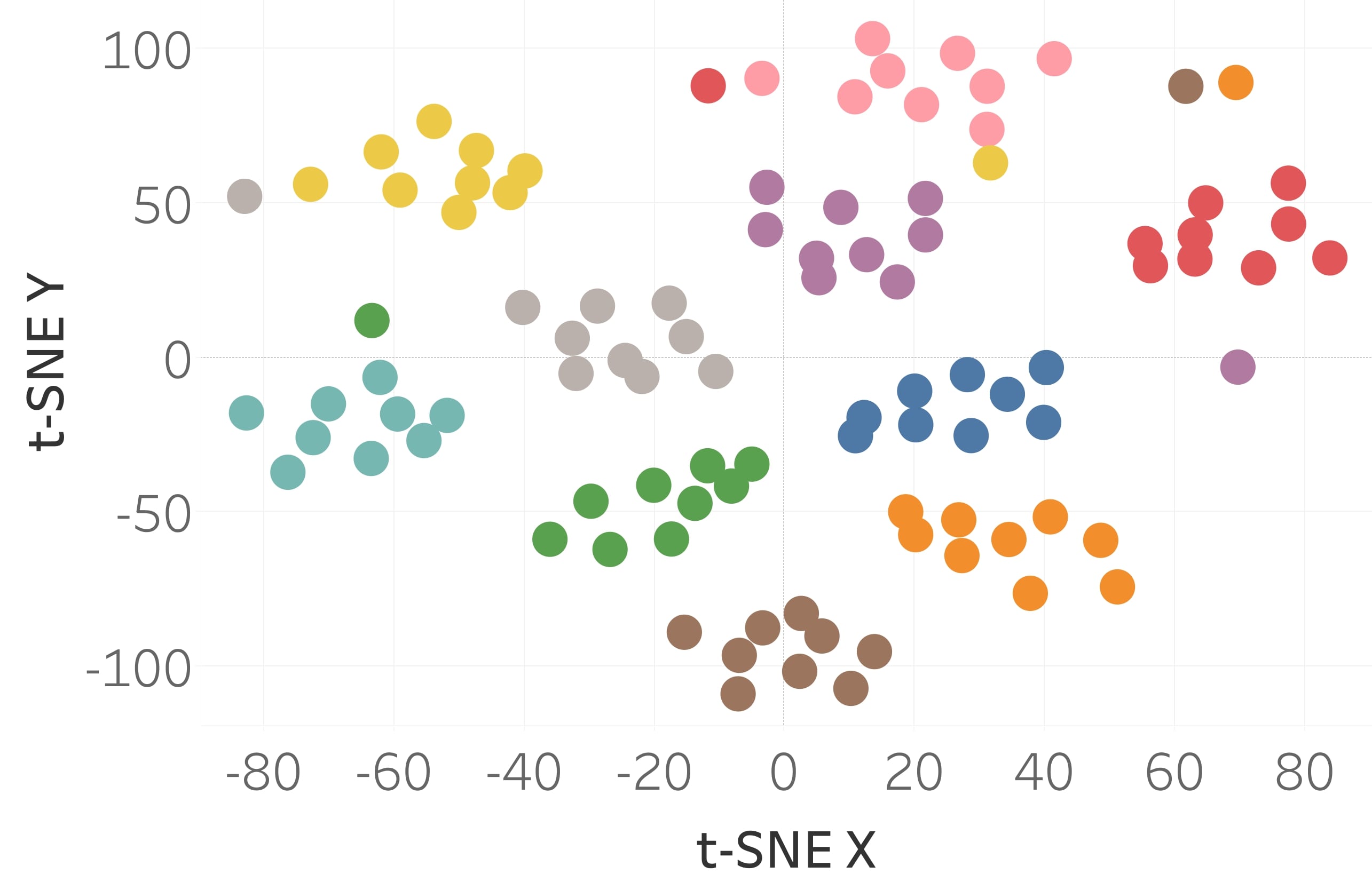}
        \caption{VGG16, epoch 18, colored according to $c$.}
    \end{subfigure}%
    ~
    \begin{subfigure}[t]{0.19\textwidth}
        \centering
        \includegraphics[width=1\textwidth]{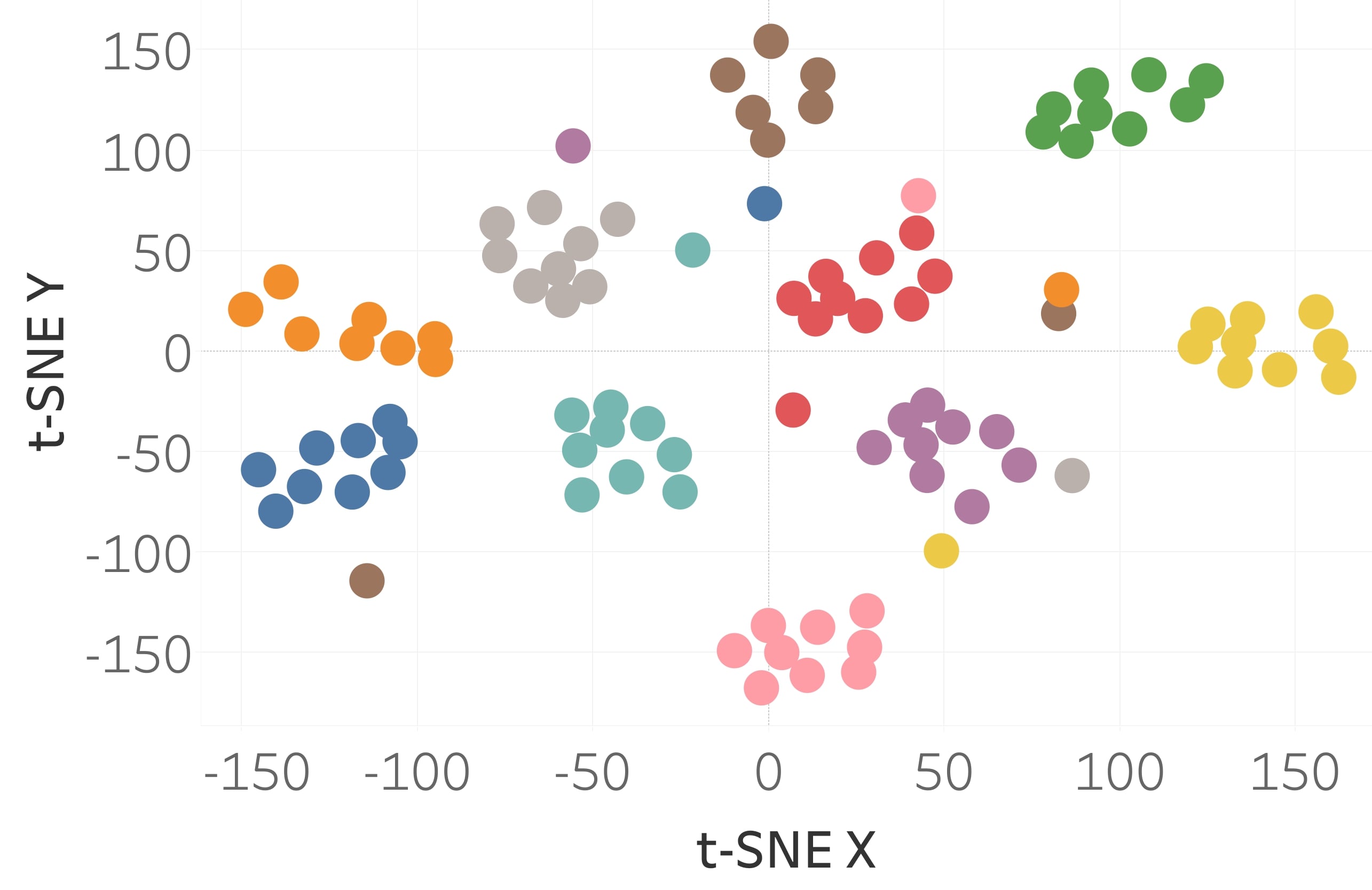}
        \caption{VGG16, epoch 33, colored according to $c$.}
    \end{subfigure}%
    ~
    \begin{subfigure}[t]{0.19\textwidth}
        \centering
        \includegraphics[width=1\textwidth]{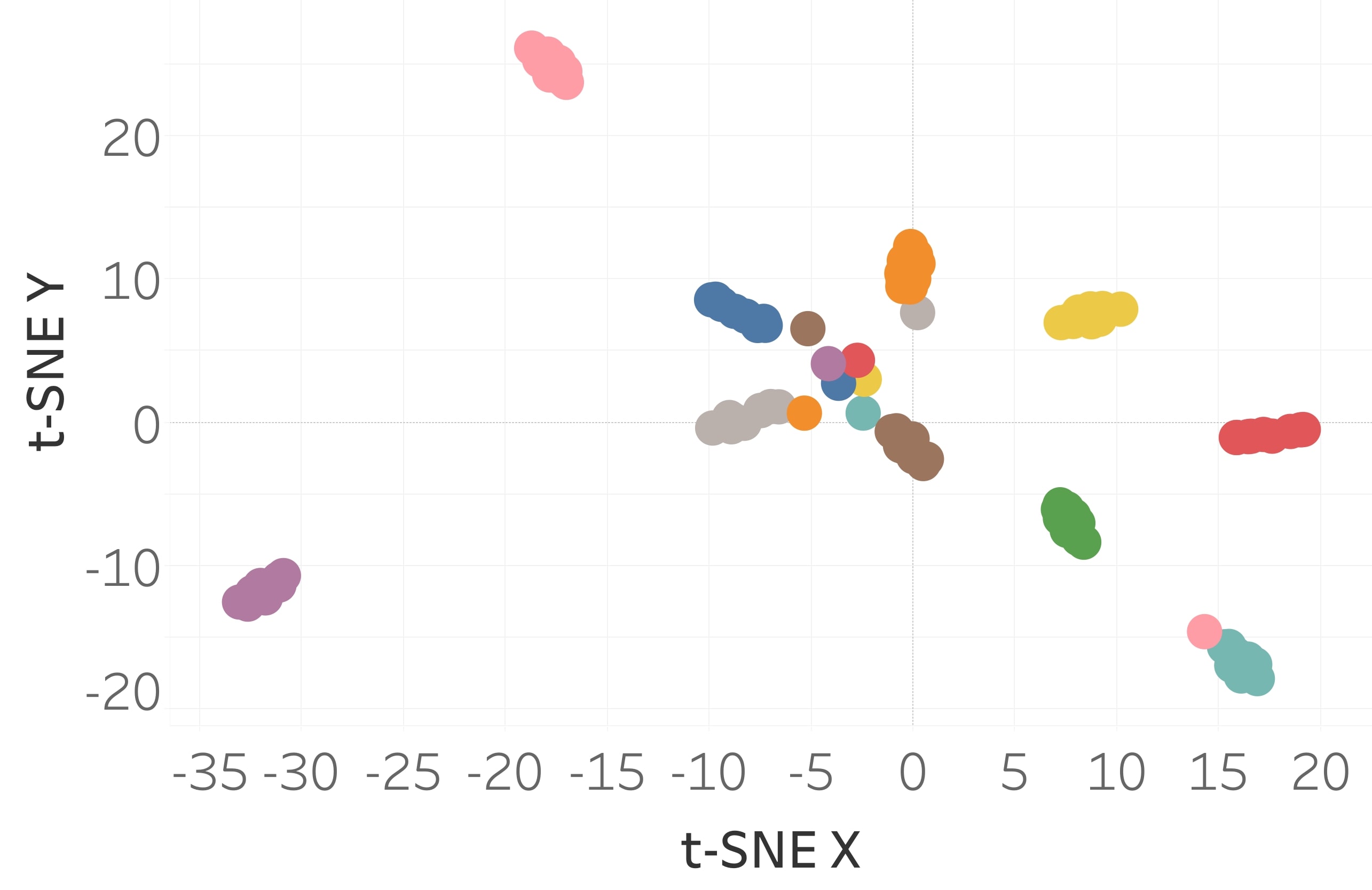}
        \caption{VGG16, epoch 350, colored according to $c$.}
    \end{subfigure}%
    \caption{\textit{t-SNE visualization of the hierarchical structure in ImageNet as a function of epoch.} Each row of panels corresponds to a different architecture, and each column to a different epoch. Each panel depicts the two-dimensional t-SNE embedding of the cluster members $\{ \delta_{c,c'} \}_{c,c'}$. ResNet50 was trained on $600$ examples per class, while VGG16 on the full dataset. For computational reasons, we subset randomly ten classes. The smaller epochs are colored according to the logit coordinate $c'$, while the larger according to the true class $c$.  A transition occurs around epoch $18$ for ResNet50 and epoch $10$ for VGG16 -- from cluster members clustering according to the logit coordinate $c'$, to cluster members clustering according to the true class $c$.
    } \label{fig:imagenet_epochs}
\end{figure*}

We summarize below our main deliverables:
\begin{enumerate}
    \item We show that the outliers in the spectrum of the Hessian, previously attributed to $G$, are due to $G$ being a second moment matrix and not a Covariance.
    \item We show the columns of $\Delta$, the matrix of logit derivatives that goes to form $G = \frac{1}{n} \Delta \Delta^T$, can be grouped into $C^2$ groups, which can then be grouped into $C$ clusters.
    \item We show how to approximate the top-$C$ outliers in the spectrum of $G$ from the Gram of cluster centers.
    \item We show that, empirically, the variation within each cluster is small compared to the variations between the clusters. That quantitative observation is responsible for the fact that the outliers in $G$ are attributable to the Gram of the cluster centers.
    \item We investigate the hierarchical structure throughout the epochs of SGD, showing that initially the group means are clustered according to the logit coordinate $c'$ and only afterwards according to the true class $c$.
    \item We verify empirically our claims across various datasets, networks and sample sizes.
    \item We observe a deviation between the $C$ outliers and our approximations and draw connections to RMT.
\end{enumerate}

\section{G is a second moment matrix}
We begin by observing the dimensions of the components constituting $G$,
\begin{equation}
    \Ave_{i,c} \Bigg\{
    \underbrace{\pdv{f(x_{i,c};\theta)}{\theta}^T}_{p \times C} \underbrace{\pdv[2]{\ell ( z, y_c)}{z} \Bigg|_{z = f(x_{i,c}; \theta)}}_{C \times C} \underbrace{\pdv{f(x_{i,c};\theta)}{\theta}}_{C \times p}
    \Bigg\}.
\end{equation}
In the following steps, we will decompose the $C \times C$ matrix $\pdv[2]{\ell ( z, y_c)}{z}$ into an outer product of length $C$ vectors. Notice that $\pdv[2]{\ell ( z, y_c)}{z}$ is the Hessian of multinomial logistic regression. In \cite{bohning1992multinomial} it was shown to be equal to
\begin{equation}
    \pdv[2]{\ell ( z, y_c)}{z}
    = \diag(p(x_{i,c}; \theta)) - p(x_{i,c}; \theta) p(x_{i,c}; \theta)^T,
\end{equation}
where $p(x_{i,c}; \theta)$ are the probabilities obtained from applying softmax to the logits of $x_{i,c}$.
\begin{lemma}
The Hessian of multinomial logistic regression can be equivalently written as follows:
\begin{align}
    \pdv[2]{\ell ( z, y_c)}{z} = & (I - \one p(x_{i,c}; \theta)^T)^T \diag(p(x_{i,c}; \theta)) \nonumber \\
    \times & (I - \one p(x_{i,c}; \theta)^T).
\end{align}
\end{lemma}

\begin{proof}
    Denote $p_i = p(x_{i,c};\theta)$ and $\sqrt{p_i}$ an element-wise square root of $p_i$. Then,
    \begin{align}
        & (I - \one p_i^T)^T \diag(p_i) (I - \one p_i^T) \\
        = & (I - \one p_i^T)^T \sqrt{\diag(p_i)} \sqrt{\diag(p_i)} (I - \one p_i^T) \\
        = & (\sqrt{\diag(p_i)} - p_i \sqrt{p_i}^T) (\sqrt{\diag(p_i)} - \sqrt{p_i} p_i^T) \\
        = & \diag(p_i)
        - p_i p_i^T
        - p_i p_i^T
        + p_i \sqrt{p_i}^T \sqrt{p_i} p_i^T \\
        = & \diag(p_i)
        - p_i p_i^T
        = \pdv[2]{\ell ( z, y_c)}{z},
    \end{align}
    proving our desired claim.
\end{proof}

Plugging the above into the $G$ term, we obtain:
\begin{align}
    G = & \Ave_{i,c} \left\{ \pdv{f(x_{i,c};\theta)}{\theta}^T (I - p(x_{i,c};\theta) \one^T) \right. \\
    \times & \left. \diag(p(x_{i,c};\theta)) (I - \one p(x_{i,c};\theta)^T) \pdv{f(x_{i,c};\theta)}{\theta} \right\}. \nonumber
\end{align}
Let $\Delta_{i,c}$ denote the matrix associated with fixed $i$ and $c$ and varying $c'$,
\begin{equation} \label{eq:centering_scaling}
    \Delta_{i,c}^T = \diag \left( \sqrt{p(x_{i,c};\theta)} \right) (I - \one p(x_{i,c};\theta)^T) \pdv{f(x_{i,c};\theta)}{\theta}.
\end{equation}
The above is a product of three matrices. The first is a matrix of logit derivatives $\pdv{f(x_{i,c};\theta)}{\theta}$, which contains in its $c'$-th row the $c'$-th logit derivative. The second is a centering matrix; the term $p(x_{i,c};\theta)^T \pdv{f(x_{i,c};\theta)}{\theta}$ is a weighted average of the $C$ logit derivatives and the vector $\one$ duplicates this mean $C$ times. The third is a diagonal matrix with the square root of the probabilities. The whole expression can be interpreted as centering the logit derivatives by subtracting their mean and then weighting the result by $\sqrt{p(x_{i,c};\theta)}$.

Returning to the derivation, using the definition of $\Delta_{i,c}$,
\begin{equation}
    G = \Ave_{i,c} \left\{ \Delta_{i,c} \Delta_{i,c}^T \right\}.
\end{equation}
Concatenating the matrices $\Delta_{i,c} \in \R^{p \times C}$ into a single matrix $\Delta \in \R^{p \times C n}$, we get
\begin{equation}
    G = \frac{1}{n} \Delta \Delta^T.
\end{equation}
Note that no mean term has been subtracted; hence, $G$ is a second moment matrix of logit derivatives, defined in Equation \ref{eq:centering_scaling}.

\subsection{$\bm{G}$ is a second moment of logit derivatives, indexed by three integers}
Note that $\Delta_{i,c} \in \R^{p \times C}$ and denote by $\delta_{i,c,c'}$ its $c'$-th column. Using this definition, we can decompose $\Delta_{i,c} \Delta_{i,c}^T$ into a summation over $C$ elements, obtaining
\begin{equation} \label{eq:G_delta_iccp}
    G = \Ave_{i,c} \left\{ \sum_{c'=1}^C \delta_{i,c,c'} \delta_{i,c,c'}^T \right\}.
\end{equation}
Hence, $G$ is a second moment matrix of logit derivatives, which can be indexed by three integers, $(i,c,c')$.

\subsection{Relation to the gradients of the loss.}
We now consider the relation between $G$ (equivalently $\Delta_{i,c}$) and the gradients of the loss. Recalling Equation \eqref{eq:loss}, the gradient of the $i$-th example can be written as follows,
\begin{equation}
    \pdv{\ell ( f(x_{i,c}; \theta), y_c)}{\theta}^T = \pdv{\ell(z, y_c)}{z}^T \Bigg|_{z=f(x_{i,c}; \theta)} \pdv{f(x_{i,c}; \theta)}{\theta}.
\end{equation}
In \cite{bohning1992multinomial} it was shown that the gradient of multinomial logistic regression is given by,
\begin{equation}
    \pdv{\ell(z, y_c)}{z} = y_c - p(x_{i,c}; \theta).
\end{equation}
Hence,
\begin{equation} \label{eq:grad}
    \pdv{\ell ( f(x_{i,c}; \theta), y_c)}{\theta}^T = (y_c - p(x_{i,c}; \theta))^T \pdv{f(x_{i,c}; \theta)}{\theta}.
\end{equation}
Using Equation \eqref{eq:centering_scaling} and the definition of $\delta_{i,c,c'}$, we get
\begin{equation} \label{eq:delta_i,c}
    \delta_{i,c,c}^T = \sqrt{p_c(x_{i,c};\theta)} (y_c - p(x_{i,c};\theta))^T \pdv{f(x_{i,c};\theta)}{\theta},
\end{equation}  
where $p_c(x_{i,c};\theta)$ is the $c$-th element of $p(x_{i,c};\theta)$. Comparing Equations \eqref{eq:grad} and \eqref{eq:delta_i,c}, we observe that $\delta_{i,c,c}$ and the gradient of the loss are equal up to a scalar. Note, however, that $G= \frac{1}{n} \Delta \Delta^T$ contains in addition to the outer products of $\delta_{i,c,c}$, outer products of $\delta_{i,c,c'}$ for $c' \neq c$. Hence, $G$ is not a second moment of the \textit{gradients of the loss}; instead it is a second moment of the \textit{logit derivatives}.

\section{Decomposing $\bm{G}$ into $\bm{C^2}$ populations}
Having established that $G$ is a second moment matrix, our goal in this section is to decompose it into two components: one associated with its mean and the other with its variance. Denoting
\begin{align}
    \delta_{c,c'} = & \Ave_i \left\{ \delta_{i,c,c'} \right\} \\
    \Sigma_{c,c'} = & \Ave_i \left\{ (\delta_{i,c,c'} - \delta_{c,c'}) (\delta_{i,c,c'} - \delta_{c,c'})^T \right\},
\end{align}
we can decompose $G$ in Equation \eqref{eq:G_delta_iccp} as follows\footnote{We assume here the classes are balanced. Otherwise, $G_{1+2}$ would be a weighted sum, with weights proportional to the number of examples in each class.}:
\begin{equation} \label{eq:L1+2}
    G = \underbrace{\frac{1}{C} \sum_{c,c'} \delta_{c,c'} \delta_{c,c'}^T}_{G_{1+2}}
    + \underbrace{\frac{1}{C} \sum_{c,c'} \Sigma_{c,c'}}_{G_3}.
\end{equation}
In the context of Figure \ref{fig:definitions}, note that $G_{1+2}$ corresponds to the aggregate of both the red circle and the blue one, while $G_3$ corresponds to the green circle.

Our original motivation for decomposing $G$ into its mean and variance terms was to isolate the component that was creating the outliers in the spectrum. Previous observations \cite{sagun2016eigenvalues,sagun2017empirical,papyan2018full} suggest the existence of $C$ dominant outliers in the spectrum of $G$. On the other hand, the first summation in the above expression, being the outer product of $C^2$ elements, could be of rank $C^2$. This, in turn, would lead to $C^2$ outliers in the spectrum. We explain this purported contradiction by noting that while there exist $C^2$ outliers, $C$ of them are significantly more dominant than the others. In the next section, we show how to extract the $C$ dominant outliers.

\section{The means themselves have structure}
In this section we focus on further decomposing the $G_{1+2}$ term. For reasons that will become clear later, we separate the elements that correspond to $c=c'$ from the rest,
\begin{equation}
    G_{1+2} = \Ave_c \{ \delta_{c,c} \delta_{c,c}^T \} + \frac{1}{C} \sum_c \sum_{c' \neq c} \delta_{c,c'} \delta_{c,c'}^T.
\end{equation}
Denoting
\begin{align} \label{eq:delta_c}
    \delta_c = & \Ave_{c' \neq c} \{ \delta_{c,c'} \} \\
    \Sigma_c = & \Ave_{c' \neq c} \{ (\delta_{c,c'} - \delta_c) (\delta_{c,c'} - \delta_c)^T \},
\end{align}
we can further decompose $G_{1+2}$ into:
\begin{align}
    G_{1+2}
    = \Ave_c \{ \delta_{c,c} \delta_{c,c}^T \}
    + & (C-1) \Ave_c \left\{ \delta_c \delta_c^T \right\} \nonumber \\
    + & (C-1) \Ave_c \left\{ \Sigma_c \right\}.
\end{align}
Plugging the above expression into Equation \ref{eq:L1+2}, we obtain
\begin{align}
    G = & \underbrace{\Ave_c \{ \delta_{c,c} \delta_{c,c}^T \}}_{G_0}
    + \underbrace{(C-1) \Ave_c \left\{ \delta_c \delta_c^T \right\}}_{G_1} \nonumber \\
    + & \underbrace{(C-1) \Ave_c \left\{ \Sigma_c \right\}}_{G_2}
    + \underbrace{\frac{1}{C} \sum_{c,c'} \Sigma_{c,c'}}_{G_3}. \label{eq:G}
\end{align}
Compare Equation \eqref{eq:G} with Figure \ref{fig:definitions}. The expressions given here implement the structure depicted in Figure \ref{fig:definitions}.

\section{Experiments}
We train VGG11 \cite{simonyan2014very}, ResNet18 \cite{he2016deep} and DenseNet40 \cite{huang2017densely} on the MNIST \cite{lecun2010mnist}, Fashion MNIST \cite{xiao2017fashion} and CIFAR10 \cite{krizhevsky2009learning} datasets. We use stochastic gradient descent with $0.9$ momentum, $5{\times}10^{-4}$ weight decay and $128$ batch size. The initial learning rate is annealed by a factor of $10$ at $1/3$ and $2/3$ of the number of epochs. We train for $200$ epochs on MNIST and Fashion MNIST and $350$ for CIFAR10. For each dataset and network, we sweep over $100$ logarithmically spaced initial learning rates in the range $[0.25,0.0001]$ and pick the one that results in the best test error in the last epoch. For each dataset and network, we repeat the previous experiments on $20$ training sample sizes logarithmically spaced in the range $[10, 5000]$. The total number of experiments ran:
\begin{align}
    & \text{3 datasets} \times \text{3 networks} \times \text{20 sample sizes} \nonumber \\
    \times & \text{100 learning rates} = \text{18,000 experiments}.
\end{align}

We also train VGG16 and ResNet50 on ImageNet \cite{deng2009imagenet}, using the same parameters described above, except for the following differences. We use a batch size of $512$, with an initial learning rate of $0.01$ and $350$ epochs. We train ResNet50 on $600$ examples per class and VGG16 on the full dataset.

We compute the eigenvalues of $G_1$, $G_2$ and $G_{1+2}$ using the \textsc{eig} function available in modern standard libraries such as SciPy. However, instead of computing the eigenavlues of $G_1 = (C-1) \sum_c \delta_c \delta_c^T$, for example, we compute the eigenvalues of the corresponding $C \times C$ Gram matrix.

We summarize our results in Figures \ref{fig:tSNE}, \ref{fig:imagenet}, \ref{fig:imagenet_epochs}, \ref{fig:DOS} and \ref{fig:SI_SB_SW_ST}, and discuss their implications in the captions. We plan to publish our code with the publication of this paper.

\begin{figure*}
    \centering
    \begin{subfigure}[t]{0.33\textwidth}
        \centering
        \includegraphics[width=1\textwidth]{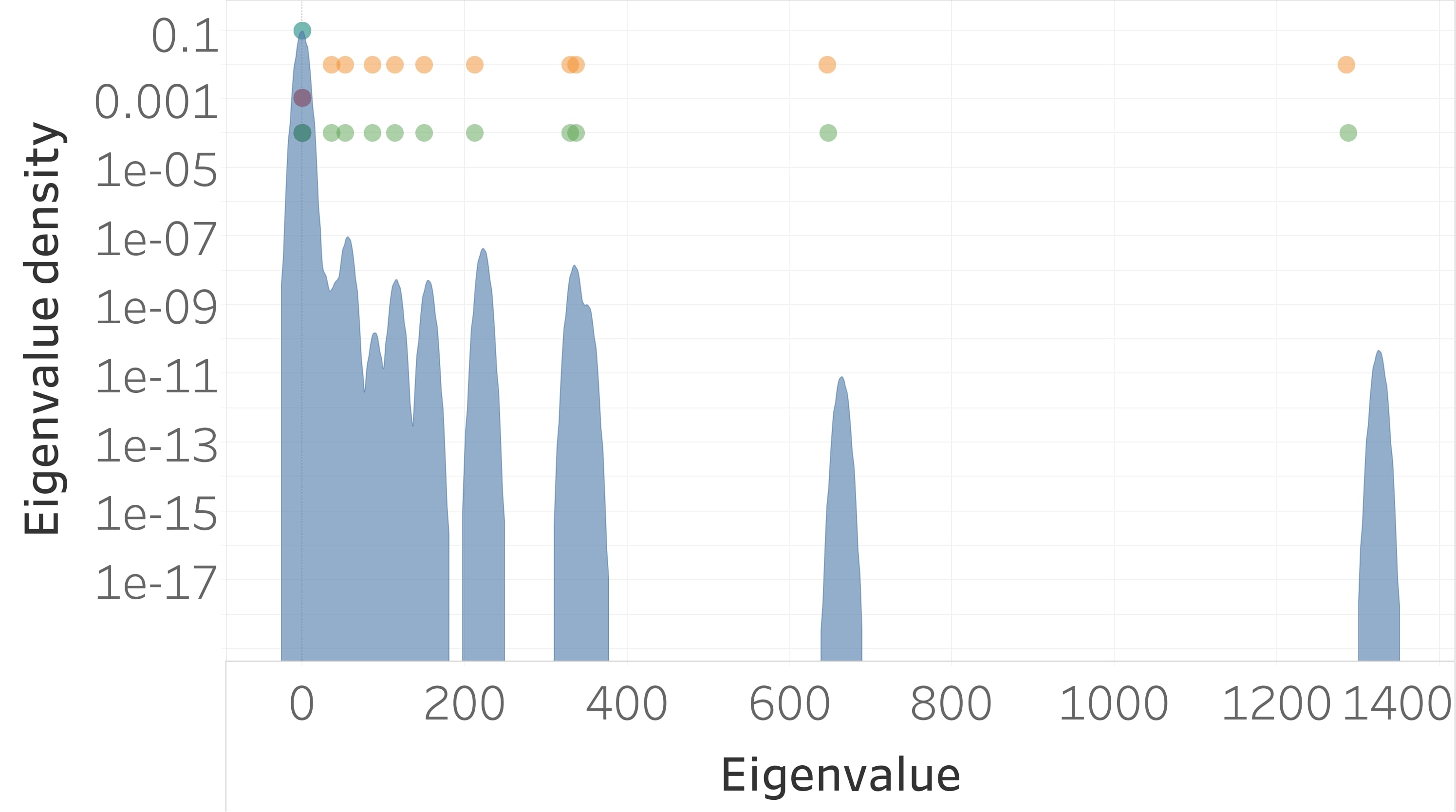}
        \caption{MNIST, 10 examples per class.}
    \end{subfigure}%
    ~
    \begin{subfigure}[t]{0.33\textwidth}
        \centering
        \includegraphics[width=1\textwidth]{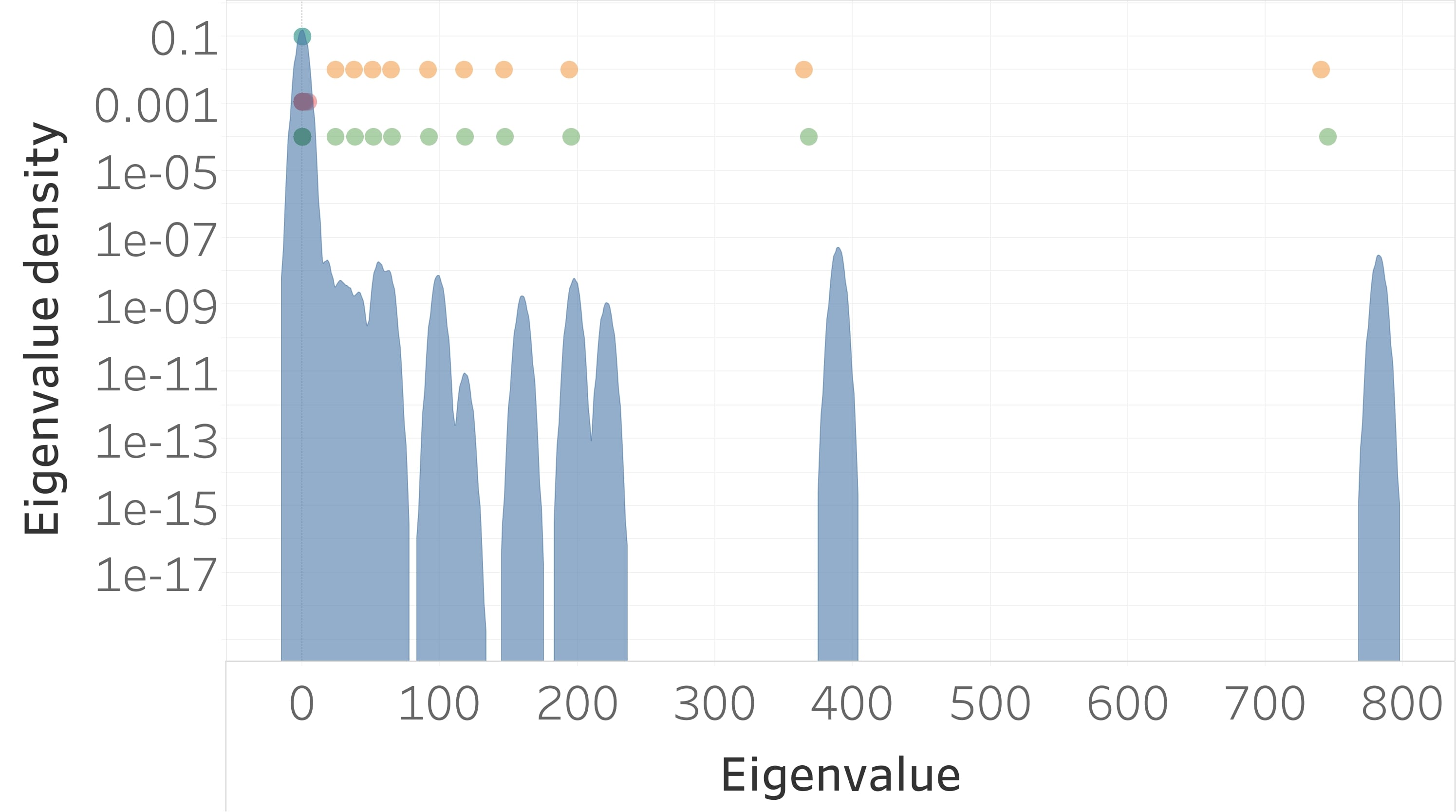}
        \caption{Fashion MNIST, 10 examples per class.}
    \end{subfigure}%
    ~
    \begin{subfigure}[t]{0.33\textwidth}
        \centering
        \includegraphics[width=1\textwidth]{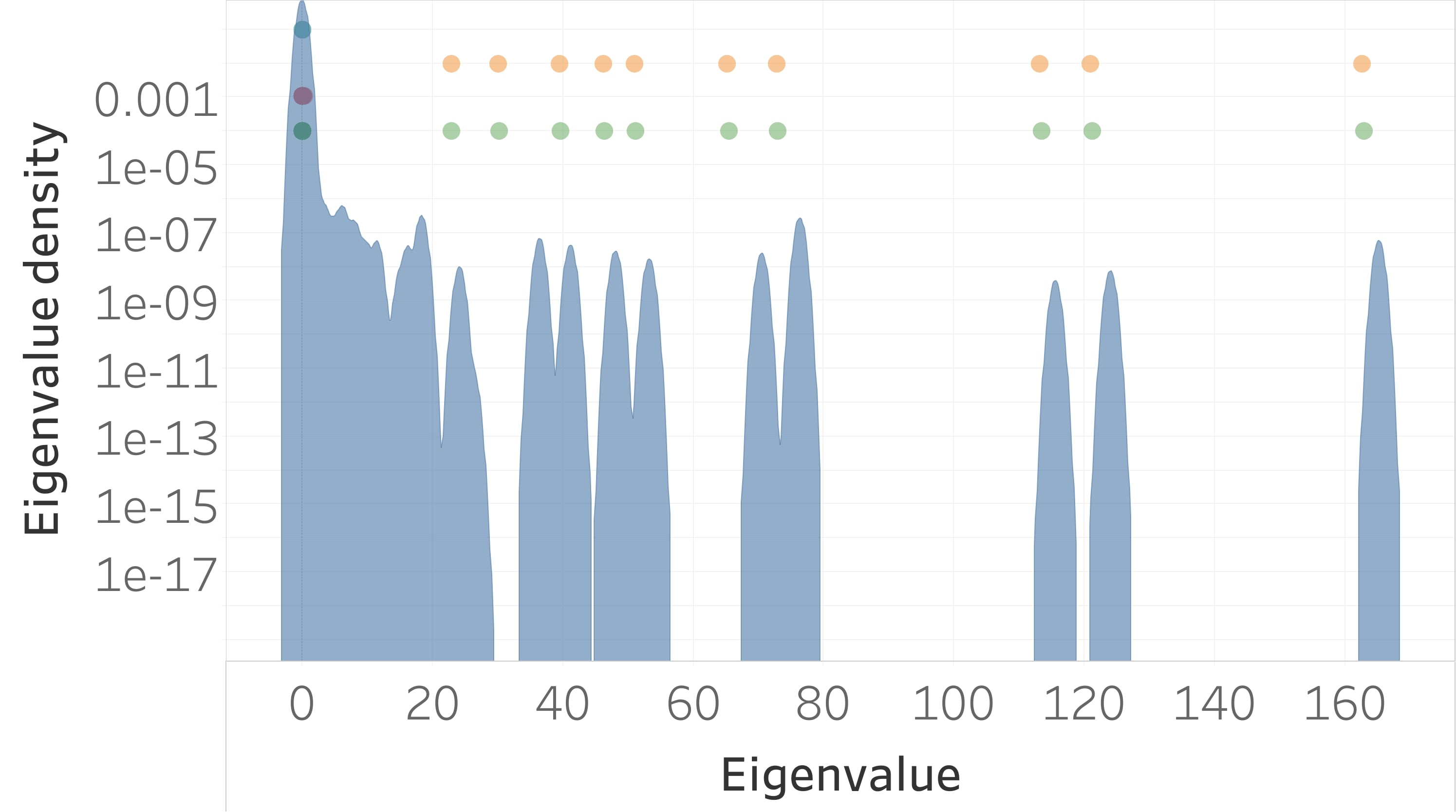}
        \caption{CIFAR10, 10 examples per class.}
    \end{subfigure}
    
    \vspace{0.25cm}
    \centering
    \begin{subfigure}[t]{0.33\textwidth}
        \centering
        \includegraphics[width=1\textwidth]{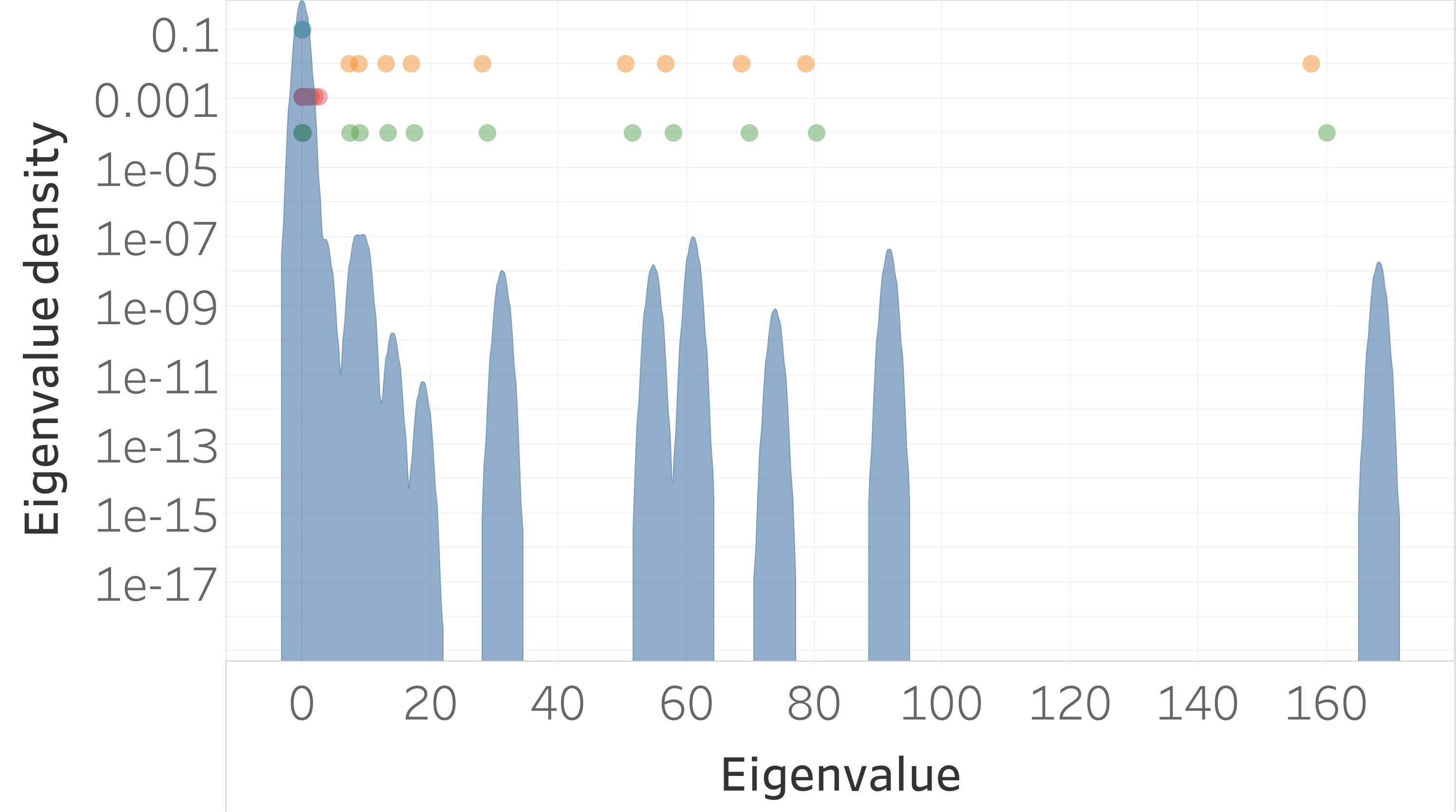}
        \caption{MNIST, 136 examples per class.}
    \end{subfigure}%
    ~
    \begin{subfigure}[t]{0.33\textwidth}
        \centering
        \includegraphics[width=1\textwidth]{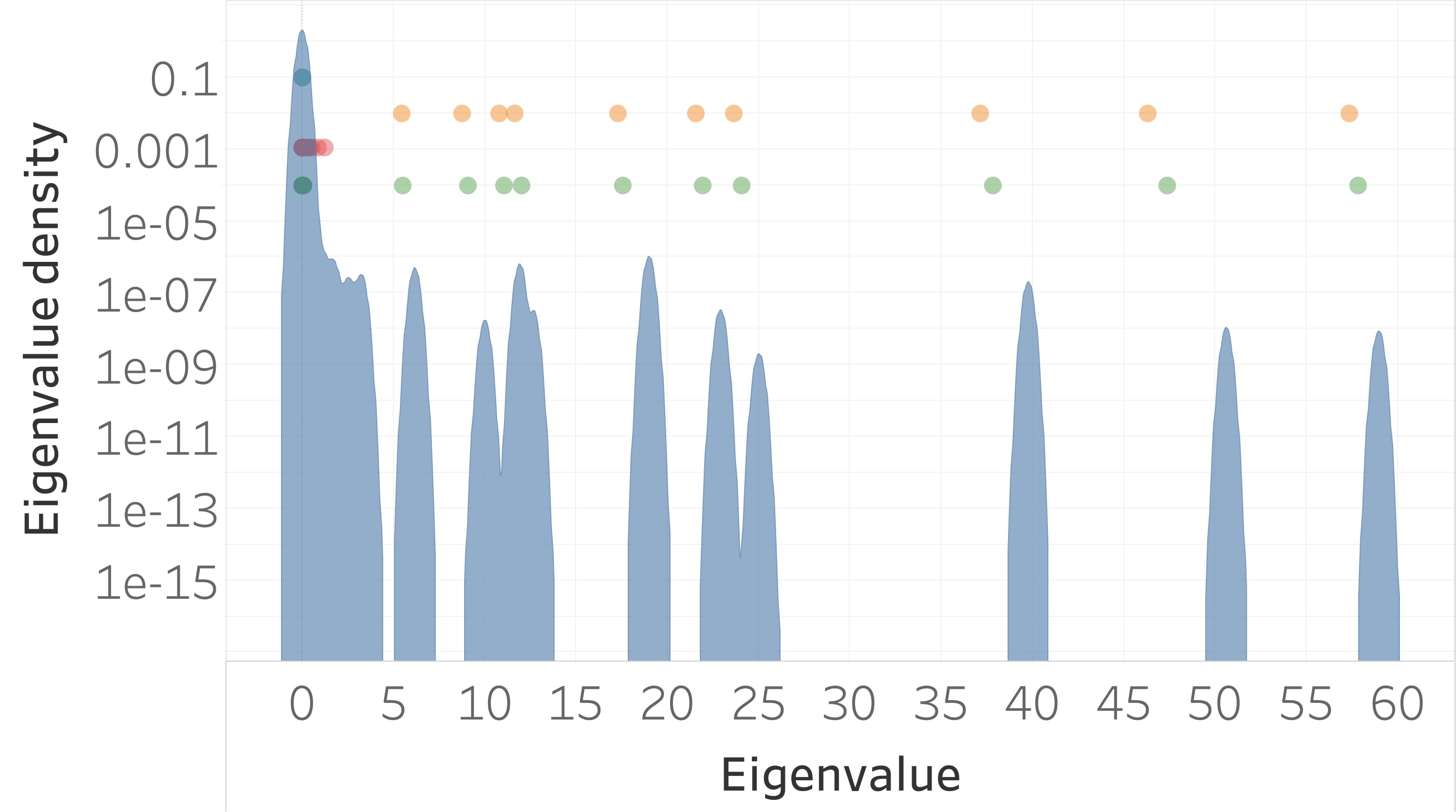}
        \caption{Fashion MNIST, 136 examples per class.}
    \end{subfigure}%
    ~
    \begin{subfigure}[t]{0.33\textwidth}
        \centering
        \includegraphics[width=1\textwidth]{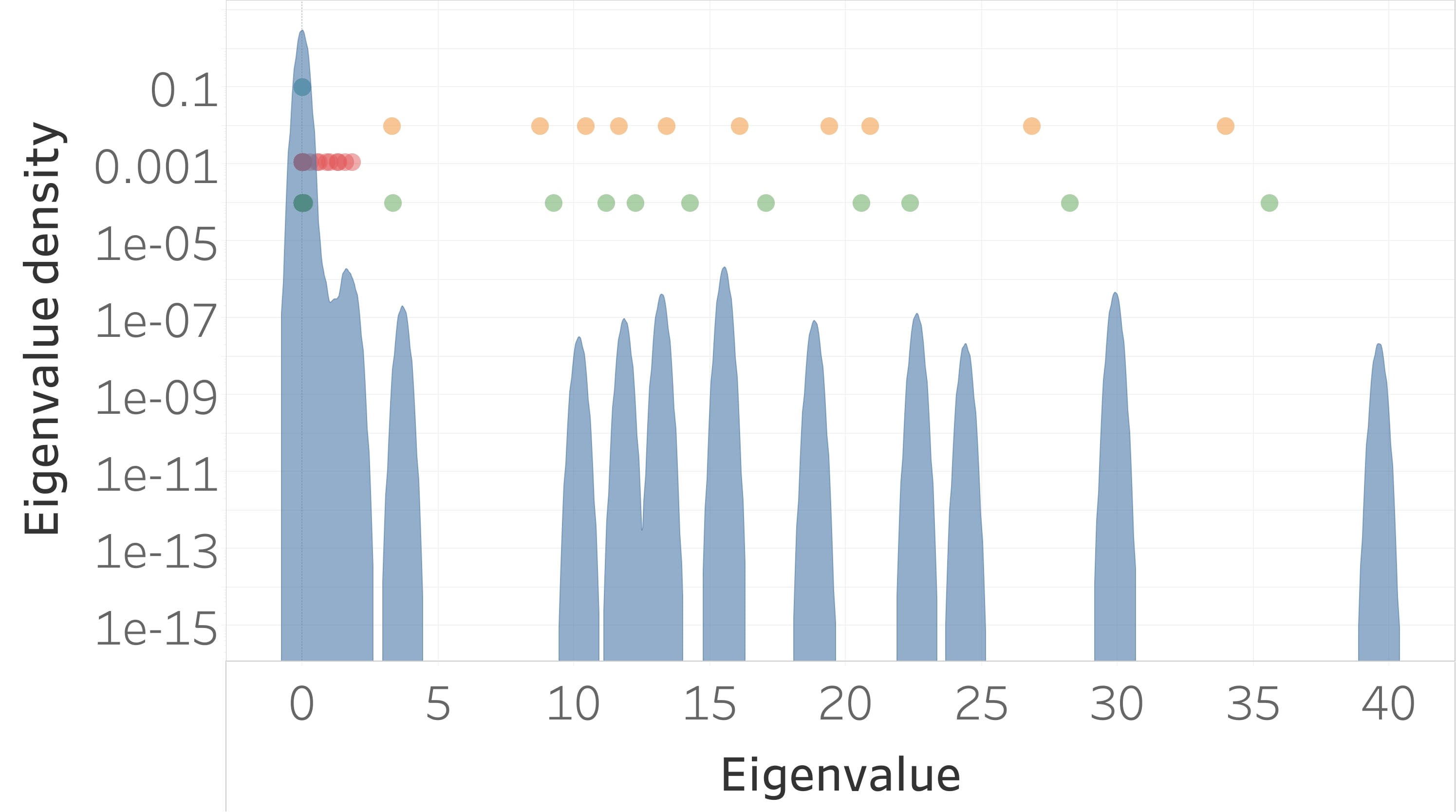}
        \caption{CIFAR10, 136 examples per class.}
    \end{subfigure}
    
    \vspace{0.25cm}
    
    \centering
    \begin{subfigure}[t]{0.33\textwidth}
        \centering
        \includegraphics[width=1\textwidth]{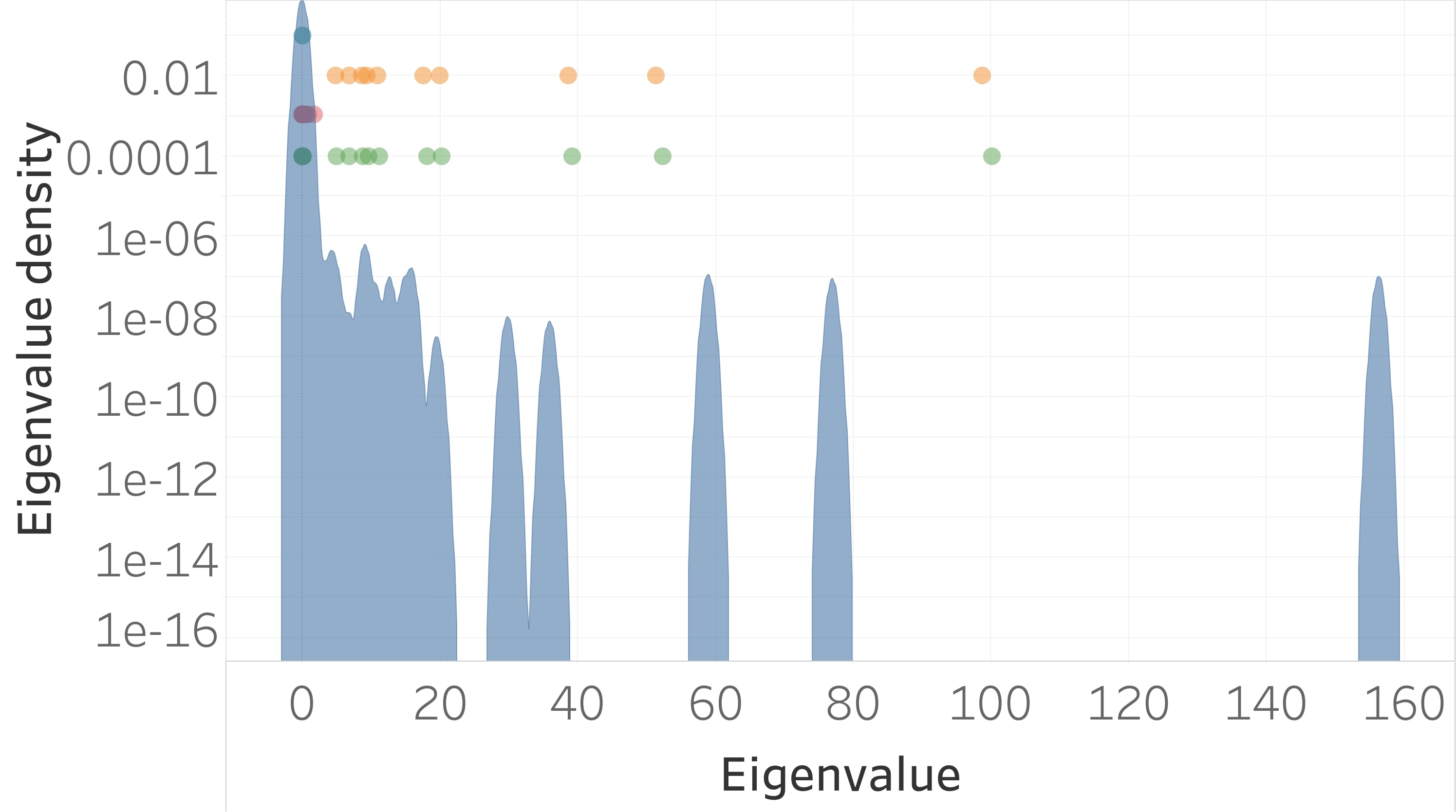}
        \caption{MNIST, 365 examples per class.}
    \end{subfigure}%
    ~
    \begin{subfigure}[t]{0.33\textwidth}
        \centering
        \includegraphics[width=1\textwidth]{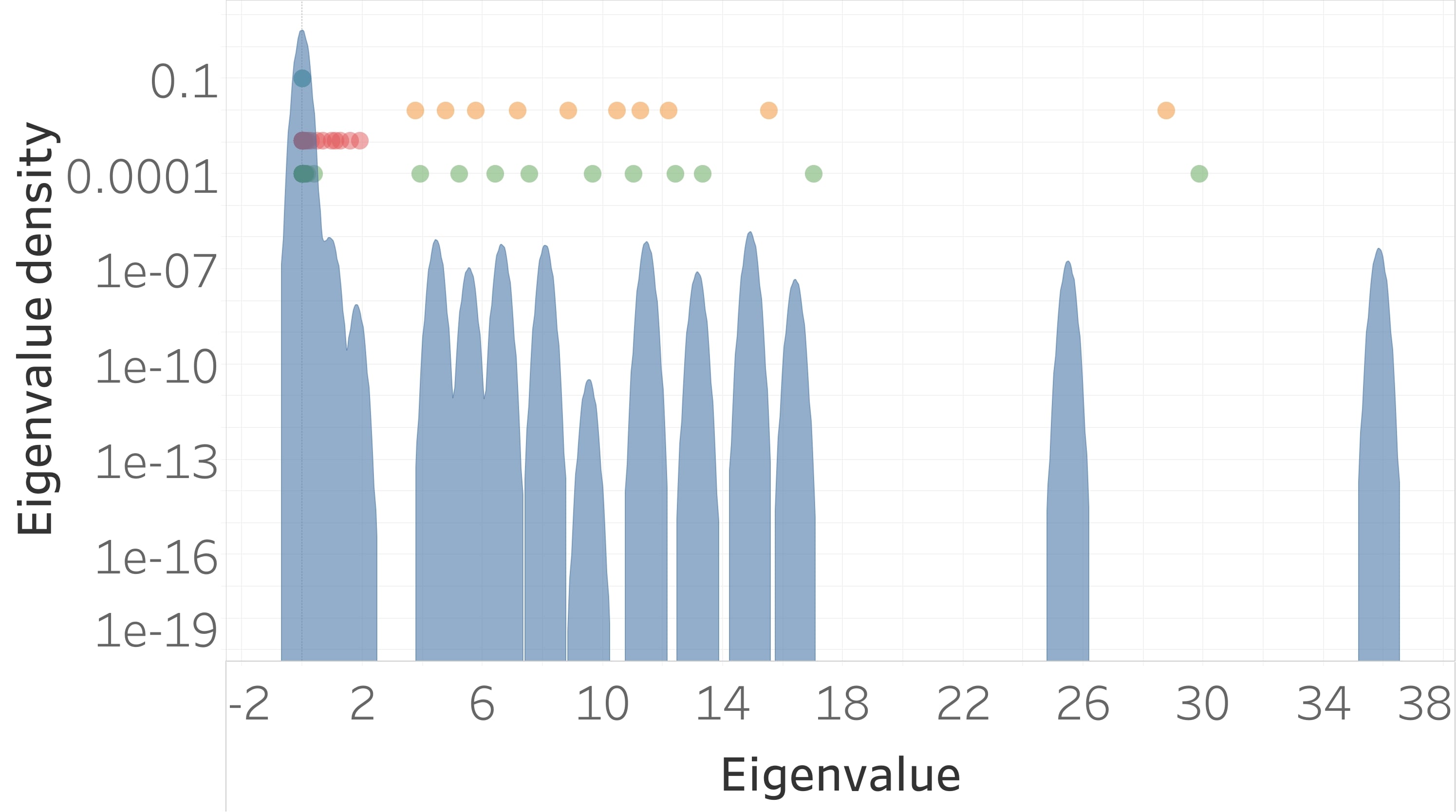}
        \caption{Fashion MNIST, 365 examples per class.}
    \end{subfigure}%
    ~
    \begin{subfigure}[t]{0.33\textwidth}
        \centering
        \includegraphics[width=1\textwidth]{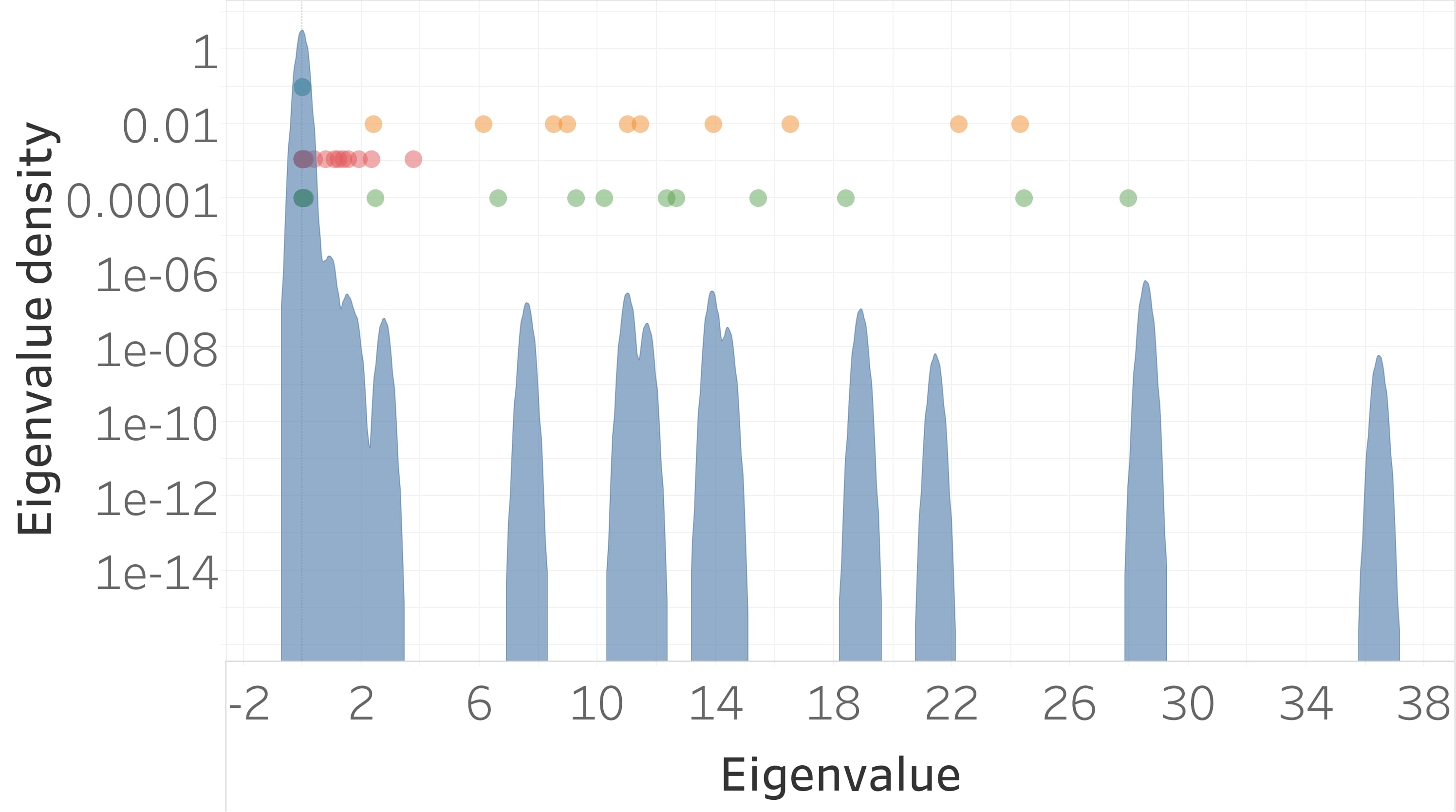}
        \caption{CIFAR10, 365 examples per class.}
    \end{subfigure}
    
    \vspace{0.25cm}
    
    \centering
    \begin{subfigure}[t]{0.33\textwidth}
        \centering
        \includegraphics[width=1\textwidth]{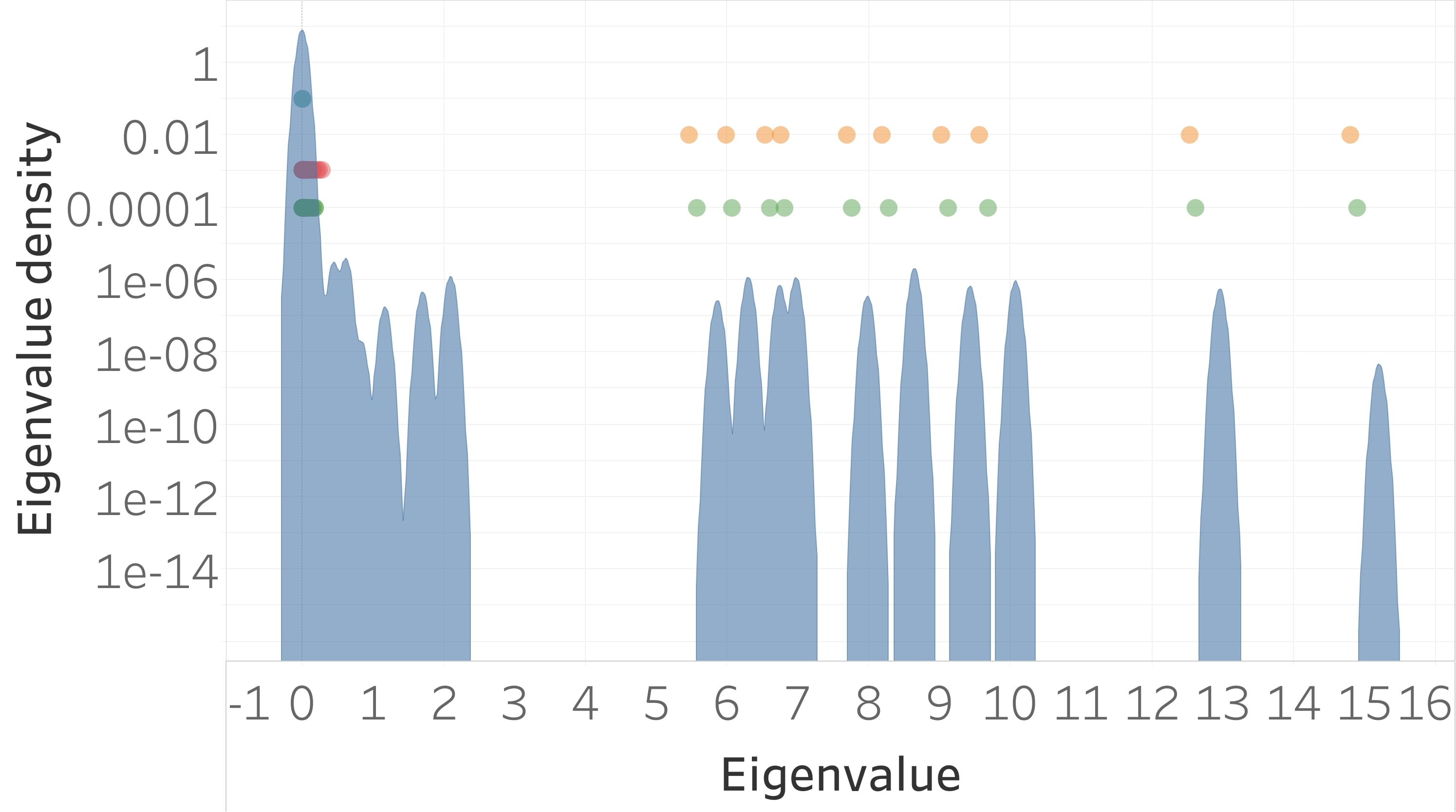}
        \caption{MNIST, 702 examples per class.}
    \end{subfigure}%
    ~
    \begin{subfigure}[t]{0.33\textwidth}
        \centering
        \includegraphics[width=1\textwidth]{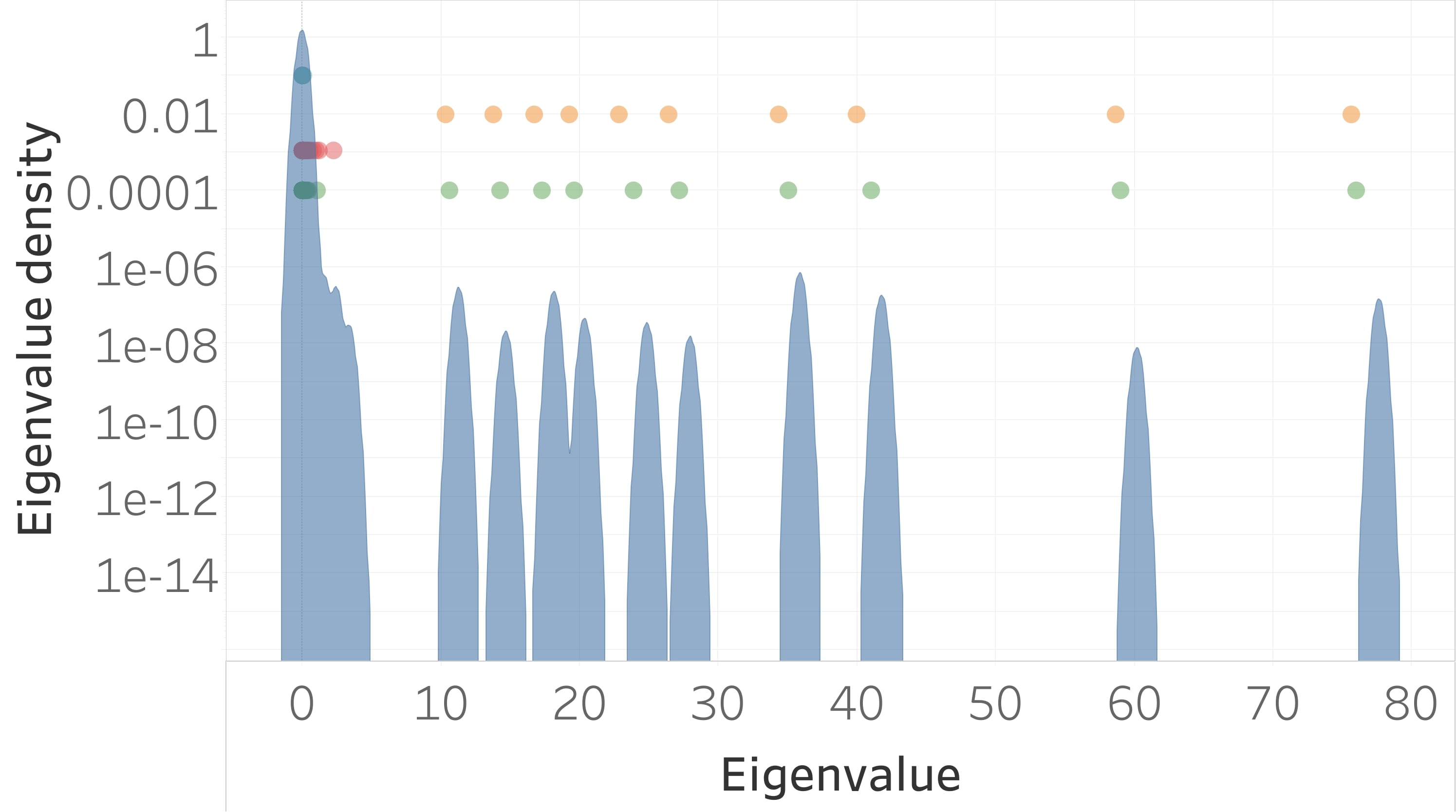}
        \caption{Fashion MNIST, 702 examples per class.}
    \end{subfigure}%
    ~
    \begin{subfigure}[t]{0.33\textwidth}
        \centering
        \includegraphics[width=1\textwidth]{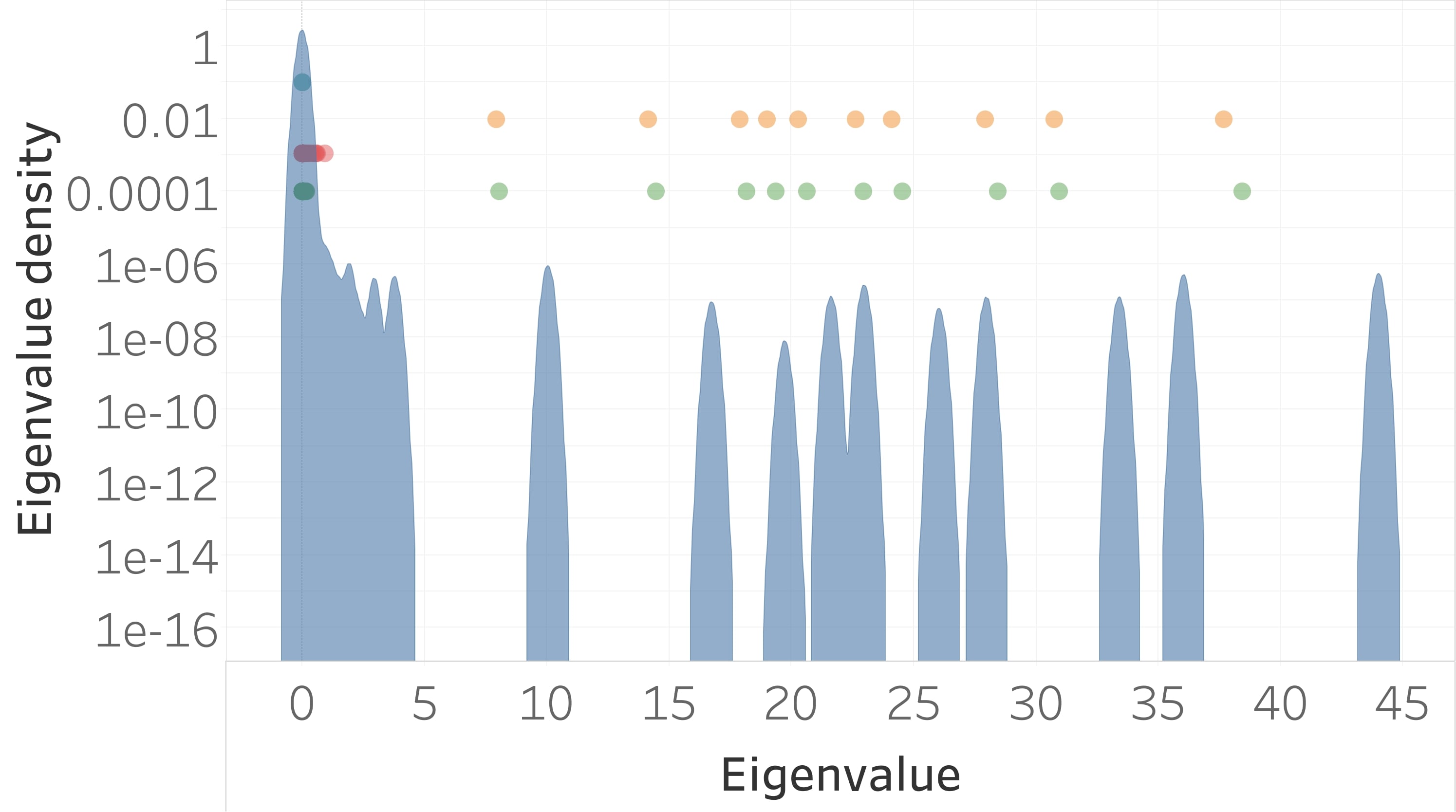}
        \caption{CIFAR10, 702 examples per class.}
    \end{subfigure}
    
    \vspace{0.25cm}
    
    \centering
    \begin{subfigure}[t]{0.33\textwidth}
        \centering
        \includegraphics[width=1\textwidth]{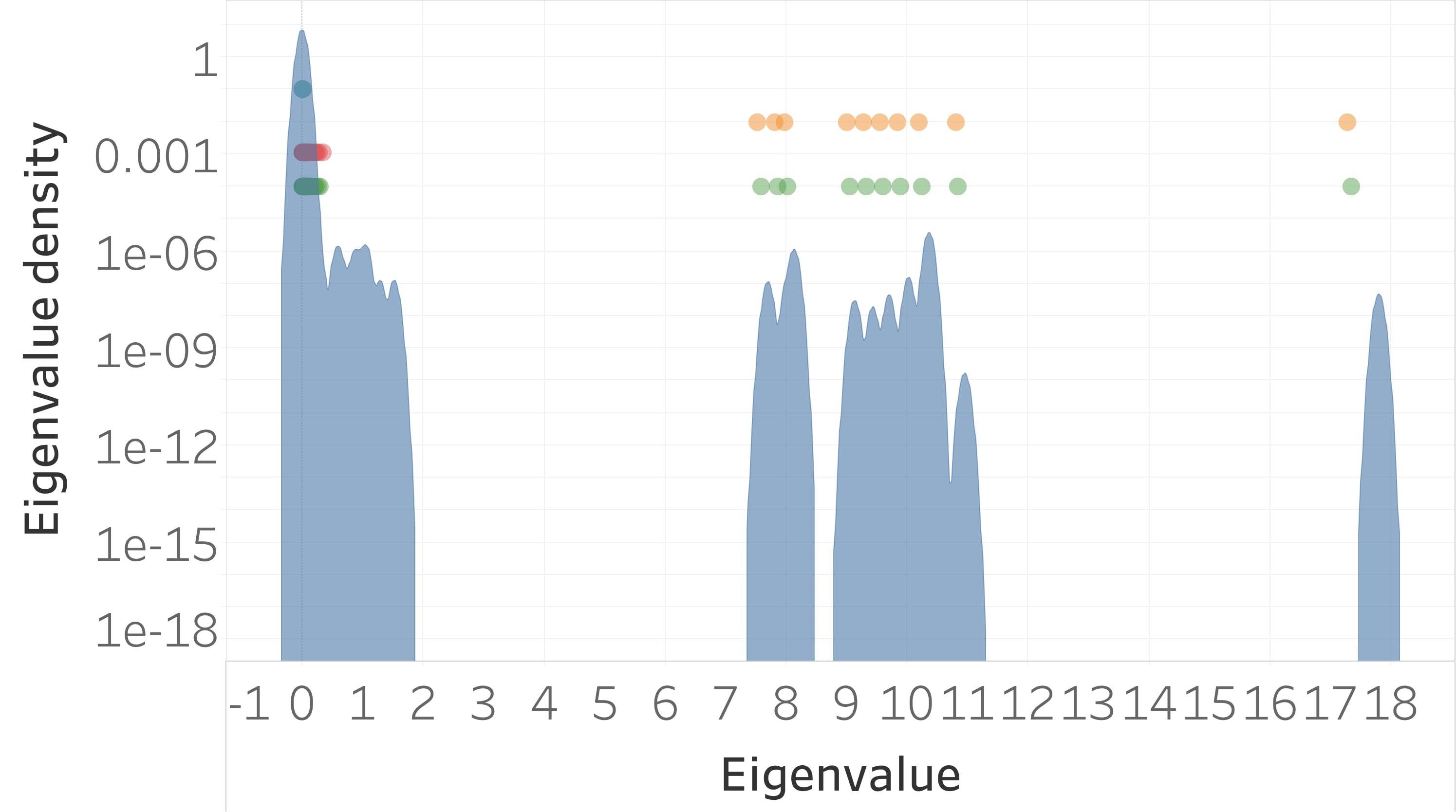}
        \caption{MNIST, 2599 examples per class.}
    \end{subfigure}%
    ~
    \begin{subfigure}[t]{0.33\textwidth}
        \centering
        \includegraphics[width=1\textwidth]{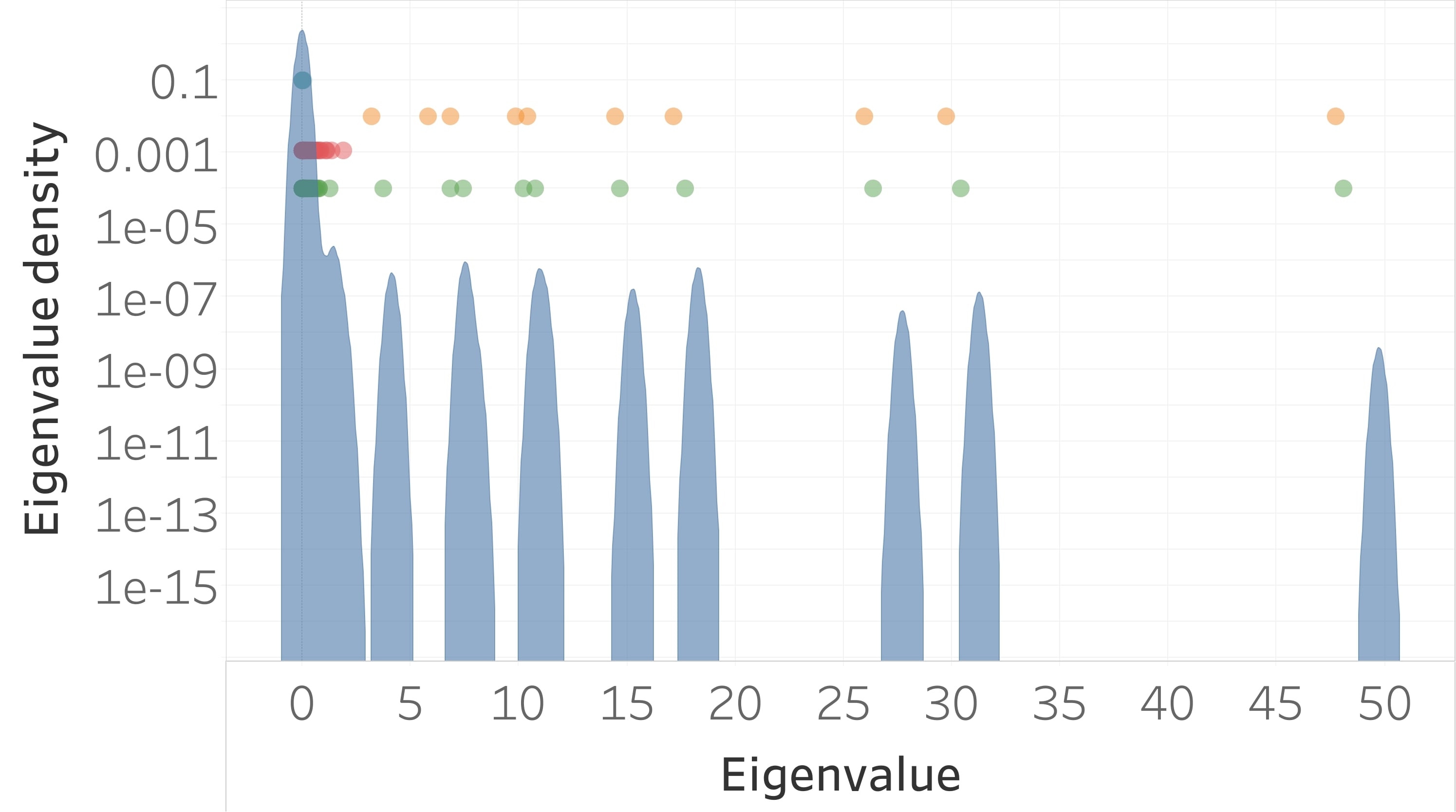}
        \caption{Fashion MNIST, 2599 examples per class.}
    \end{subfigure}%
    ~
    \begin{subfigure}[t]{0.33\textwidth}
        \centering
        \includegraphics[width=1\textwidth]{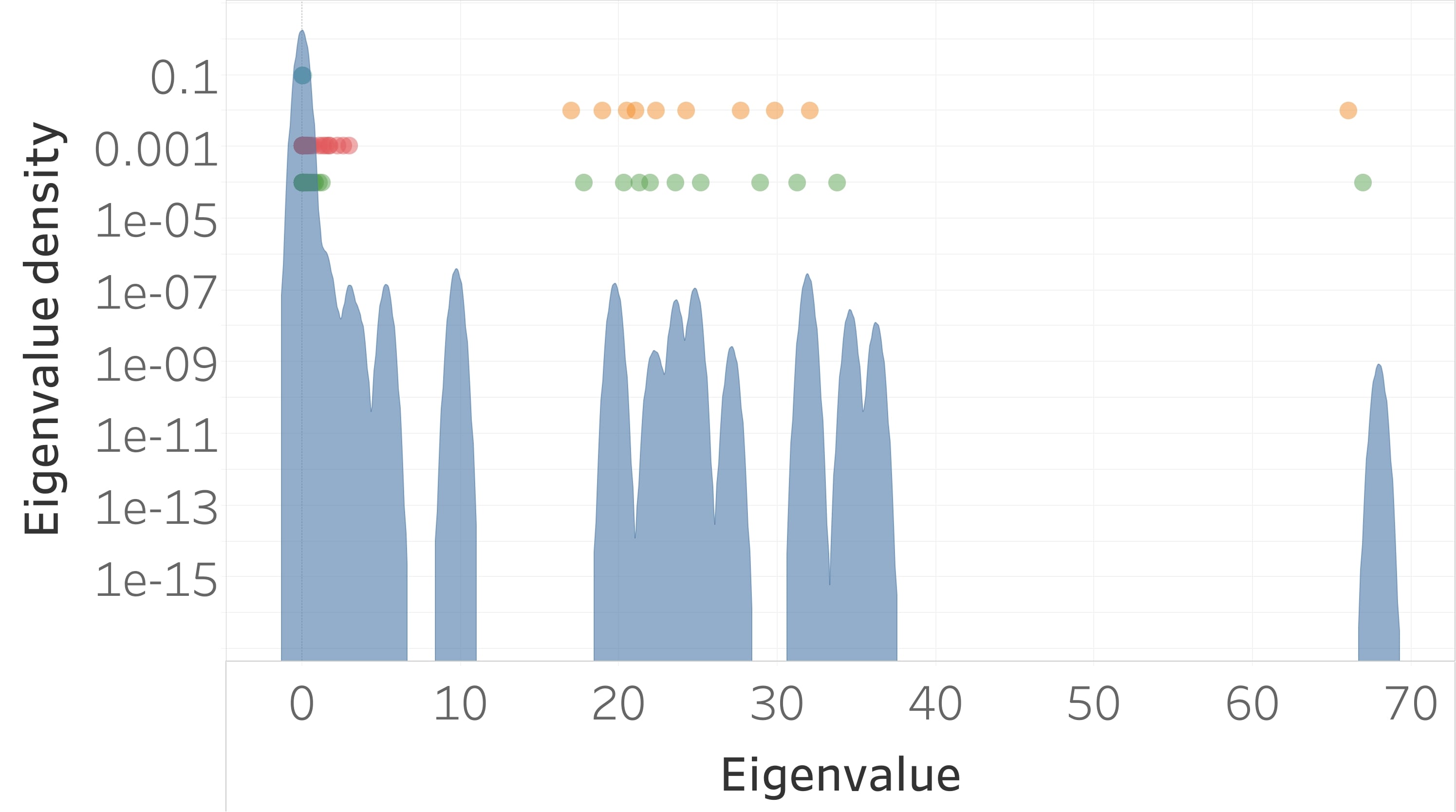}
        \caption{CIFAR10, 1351 examples per class.}
    \end{subfigure}
    \caption{\textit{Decomposing $G$ into its constituent components for the ResNet18 architecture.} Each column of panels corresponds to a different dataset, and each row to a different sample size. Each panel depicts the density of the spectrum of $G$ in \textbf{{\color{MidnightBlue}blue}}, where the y-axis is on a logarithmic scale. The density was approximated using the \textsc{FastLanczos} method presented in \cite{papyan2018full}. Each panel also plots the eigenvalues of $G_0$ in \textbf{{\color{Cyan}cyan}}, $G_1$ in \textbf{{\color{Orange}orange}}, $G_2$ in \textbf{{\color{Red}red}} and $G_{1+2}$ in \textbf{{\color{ForestGreen}green}}. The obtained results corroborate certain predictions made throughout our analysis. Specifically, the outliers in the second moment matrix $G$ can be attributed to the eigenvalues of $G_{1+2}$. The top-$C$ eigenvalues of $G_{1+2}$ are dominant compared to the others and they match those of $G_1$. The eigenvalues of $G_0$ are negligible compared to the spread of the others. They correspond to a single cyan point in the main lobe of each plot.
    } \label{fig:DOS}
\end{figure*}

\begin{figure*}[t!]
    \centering
    \begin{subfigure}[t]{0.33\textwidth}
        \centering
        \includegraphics[width=1\textwidth]{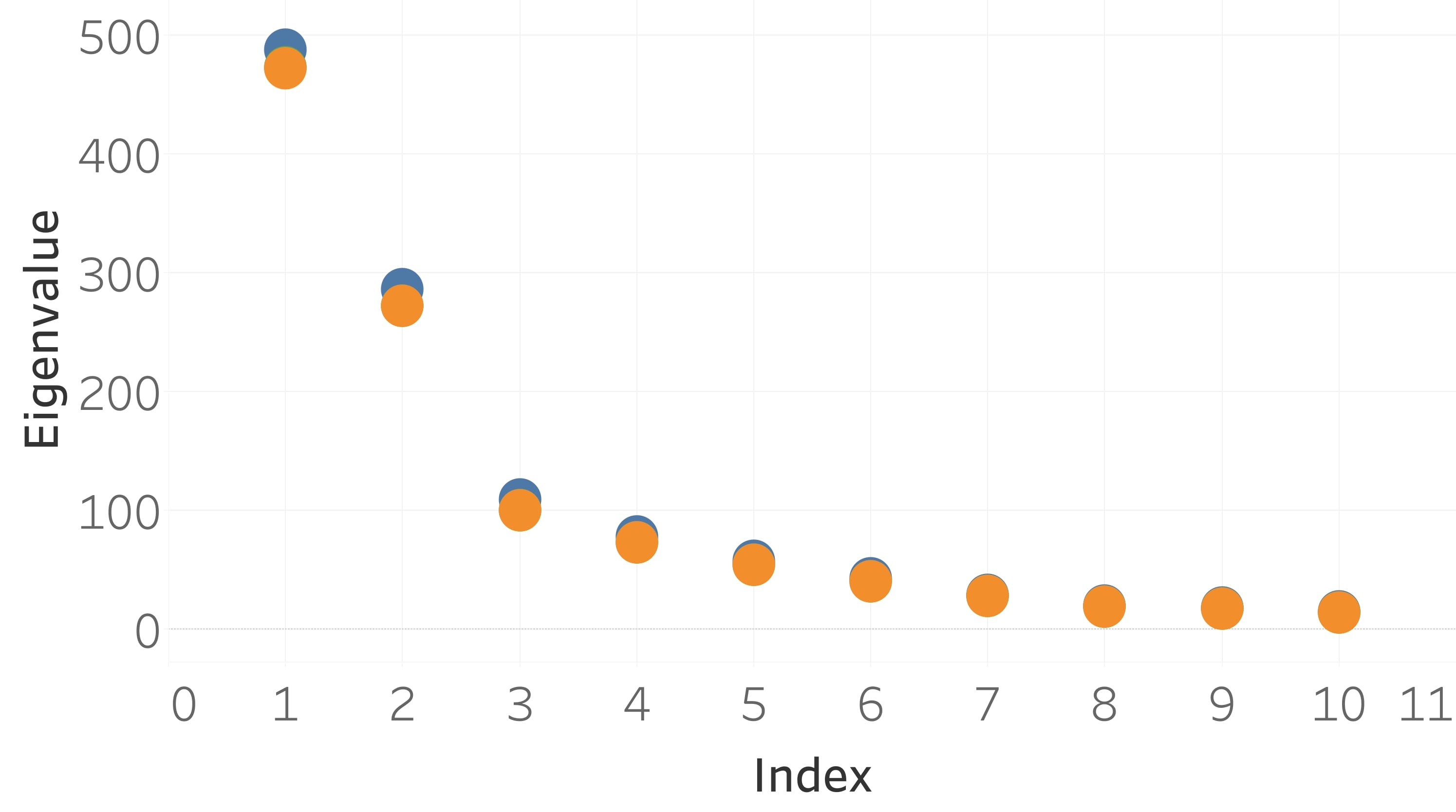}
        \caption{MNIST, 10 examples per class.}
    \end{subfigure}%
    ~
    \begin{subfigure}[t]{0.33\textwidth}
        \centering
        \includegraphics[width=1\textwidth]{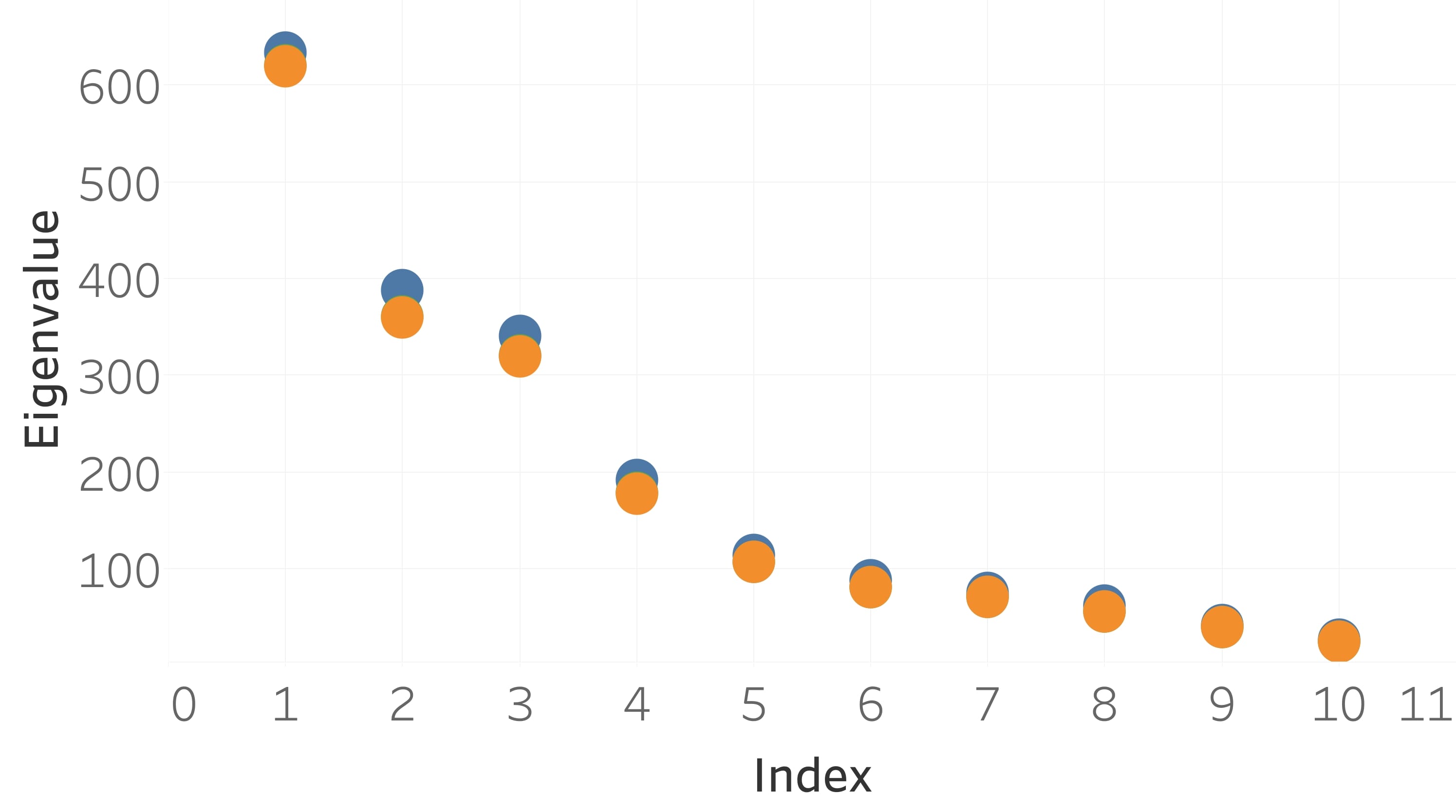}
        \caption{Fashion MNIST, 10 examples per class.}
    \end{subfigure}%
    ~
    \begin{subfigure}[t]{0.33\textwidth}
        \centering
        \includegraphics[width=1\textwidth]{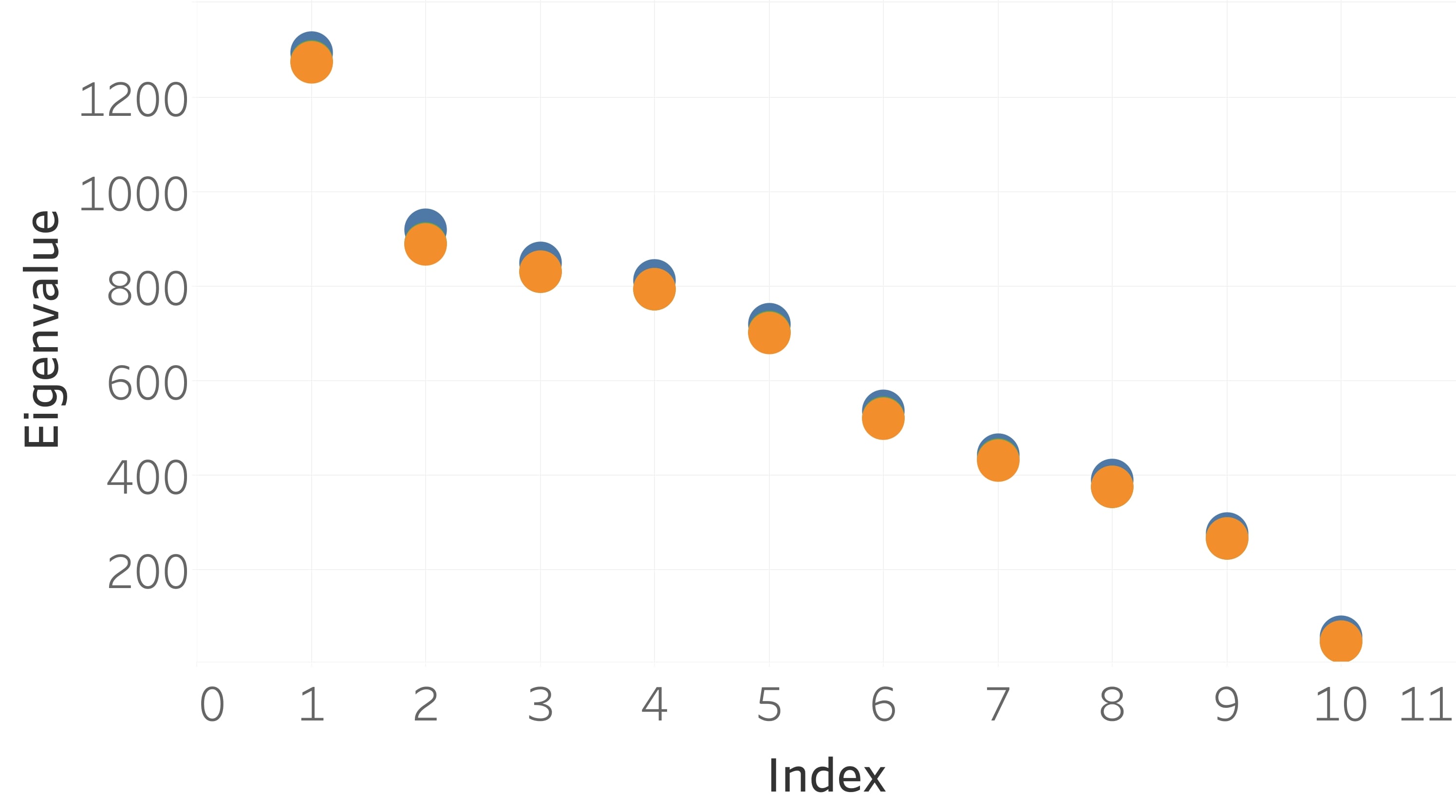}
        \caption{CIFAR10, 10 examples per class.}
    \end{subfigure}
    
    \vspace{0.1cm}
    
    \centering
    \begin{subfigure}[t]{0.33\textwidth}
        \centering
        \includegraphics[width=1\textwidth]{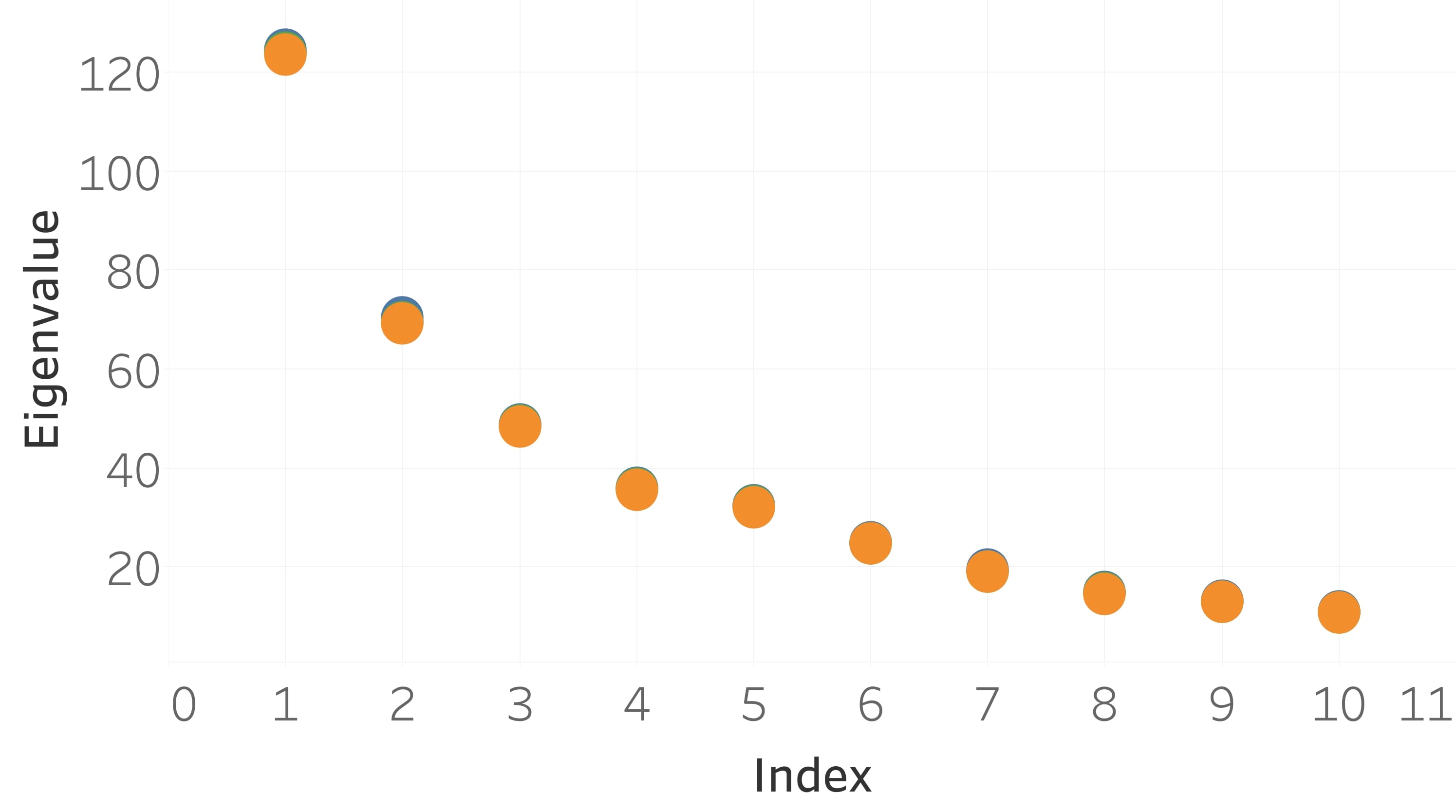}
        \caption{MNIST, 136 examples per class.}
    \end{subfigure}%
    ~
    \begin{subfigure}[t]{0.33\textwidth}
        \centering
        \includegraphics[width=1\textwidth]{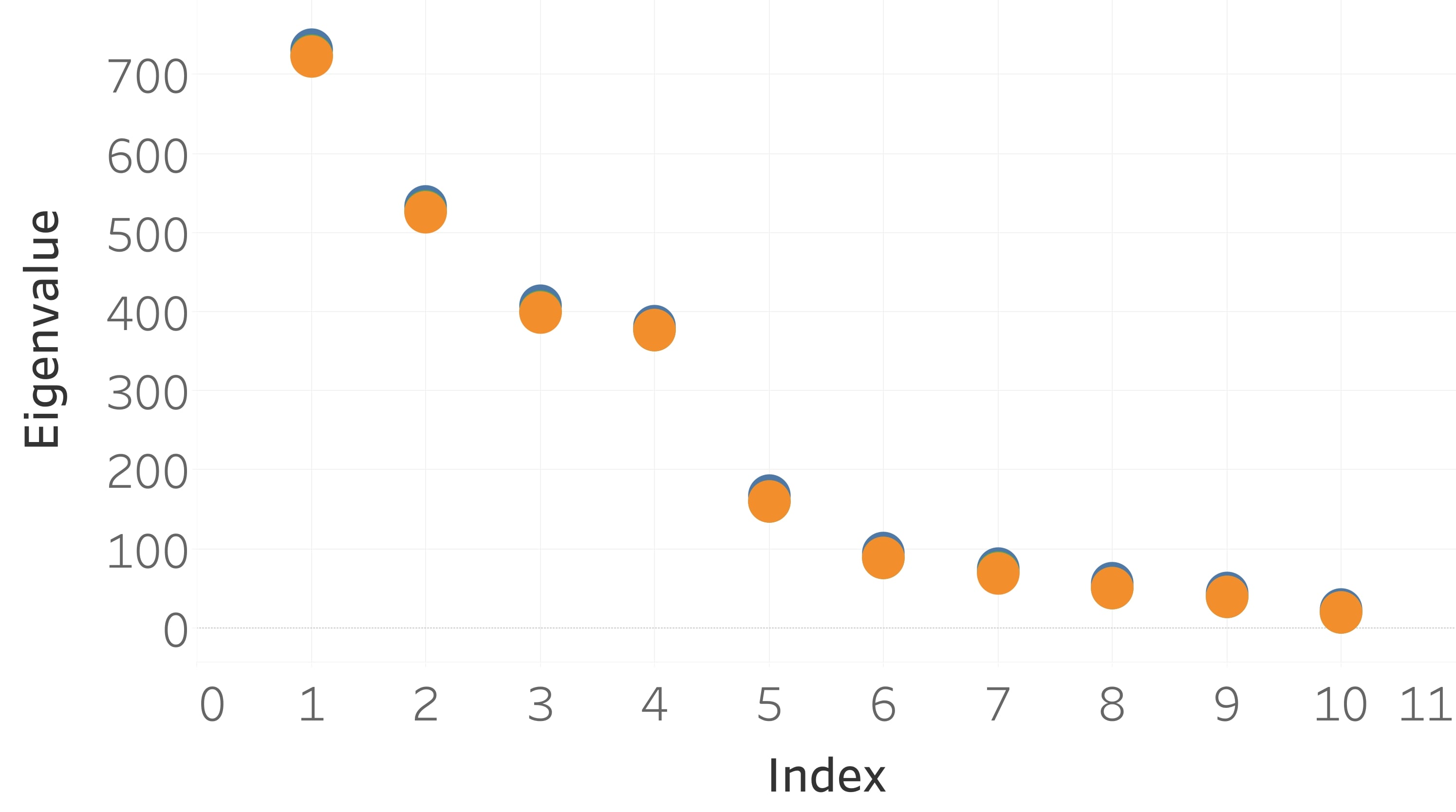}
        \caption{Fashion MNIST, 136 examples per class.}
    \end{subfigure}%
    ~
    \begin{subfigure}[t]{0.33\textwidth}
        \centering
        \includegraphics[width=1\textwidth]{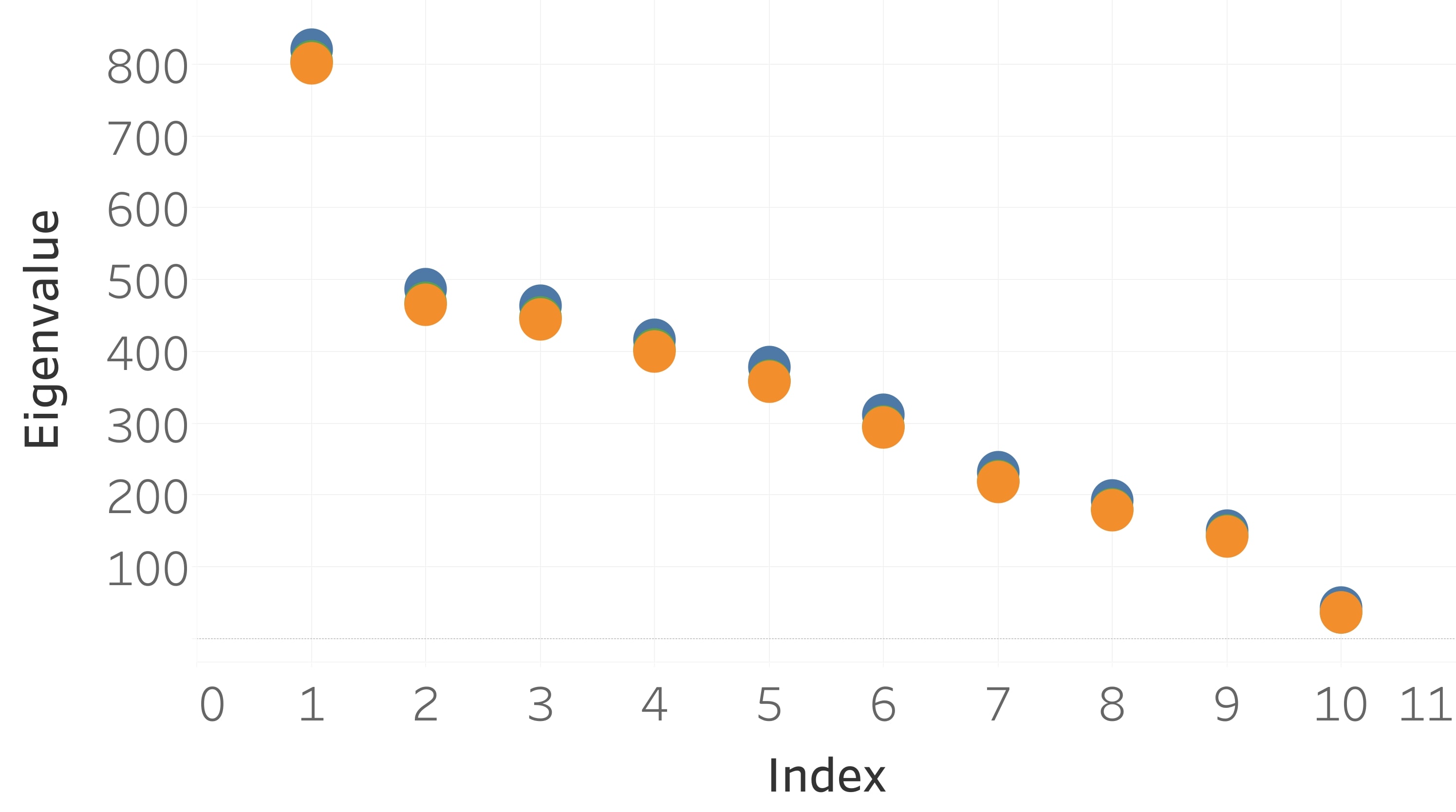}
        \caption{CIFAR10, 136 examples per class.}
    \end{subfigure}
    
    
    
    \vspace{0.1cm}
    
    \centering
    \begin{subfigure}[t]{0.33\textwidth}
        \centering
        \includegraphics[width=1\textwidth]{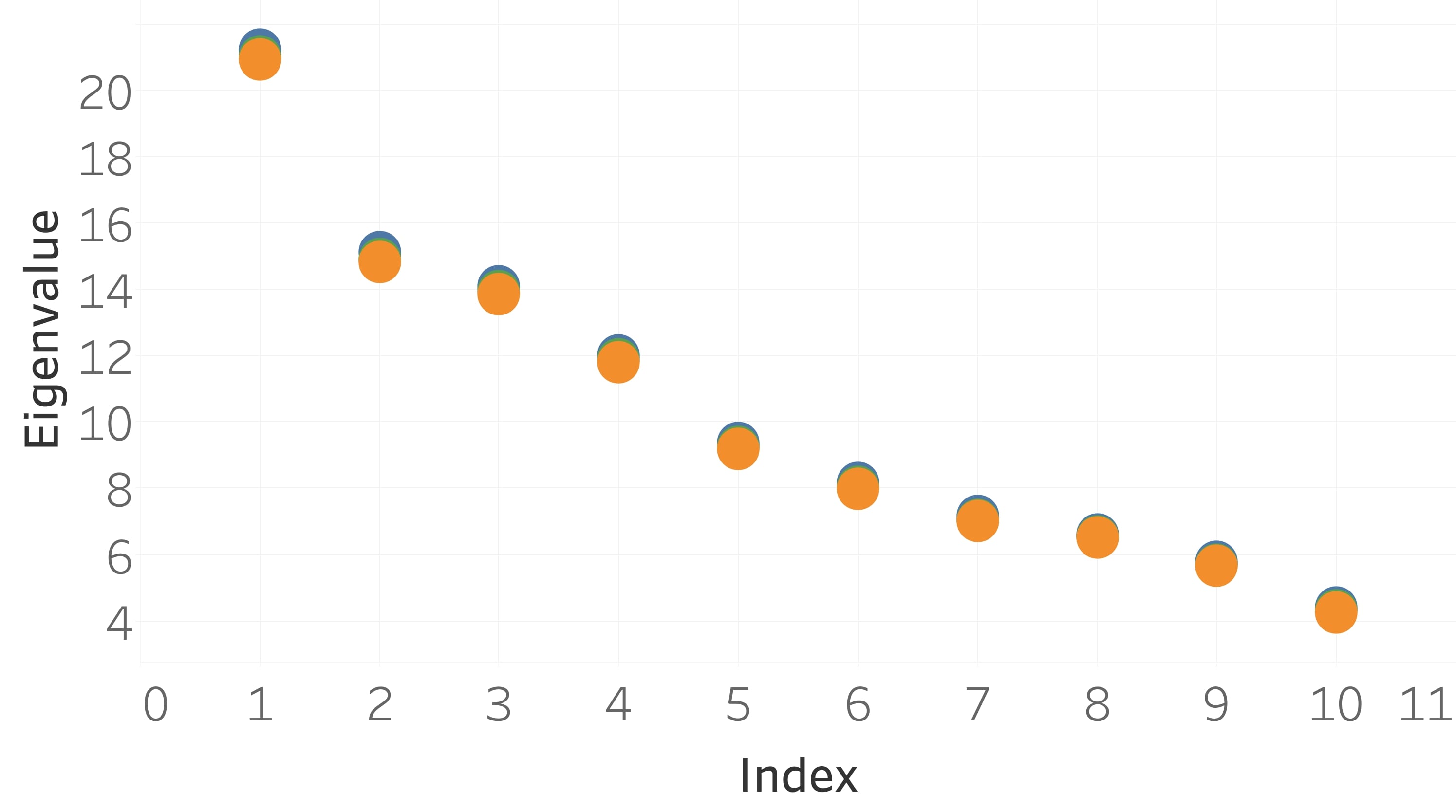}
        \caption{MNIST, 1351 examples per class.}
    \end{subfigure}%
    ~
    \begin{subfigure}[t]{0.33\textwidth}
        \centering
        \includegraphics[width=1\textwidth]{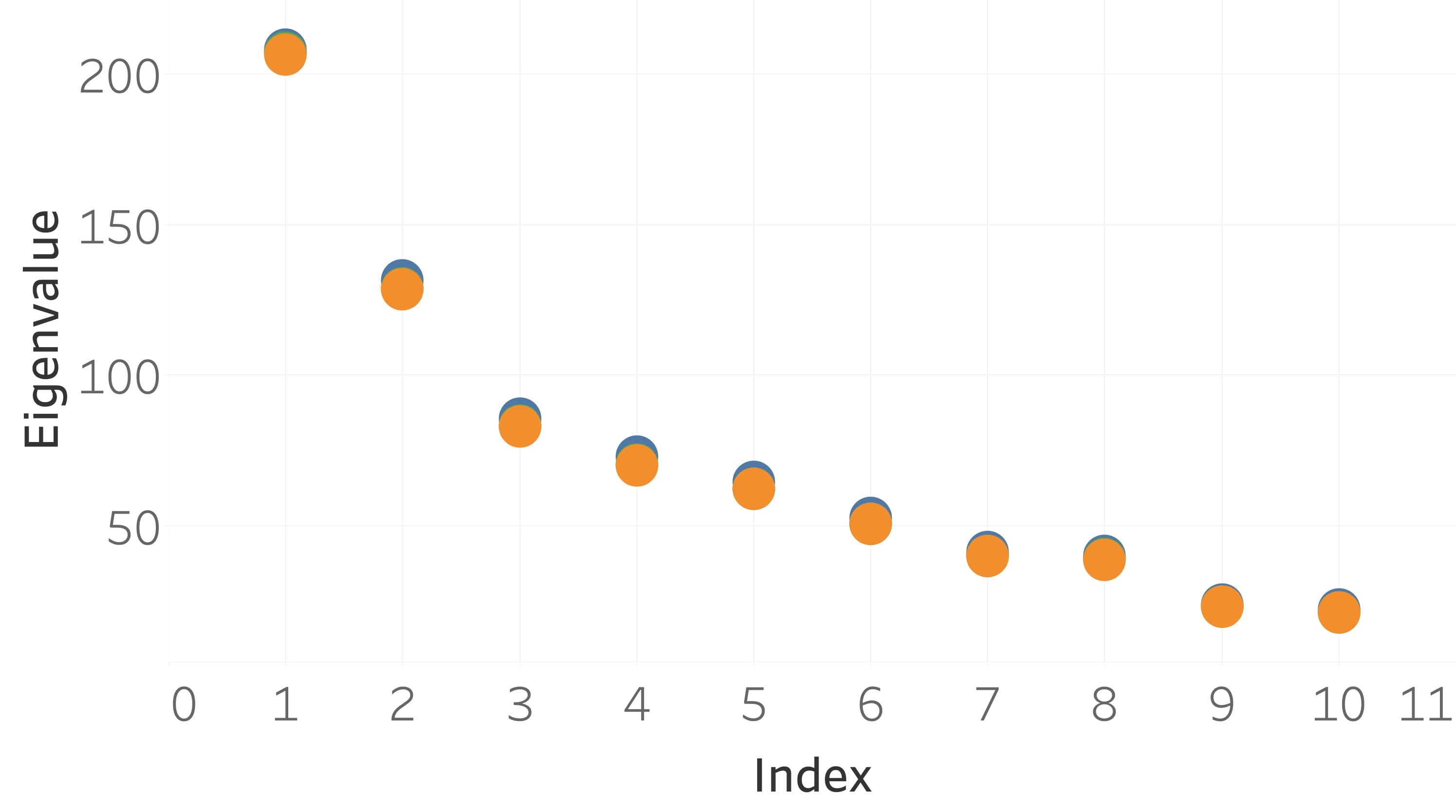}
        \caption{Fashion MNIST, 1351 examples per class.}
    \end{subfigure}%
    ~
    \begin{subfigure}[t]{0.33\textwidth}
        \centering
        \includegraphics[width=1\textwidth]{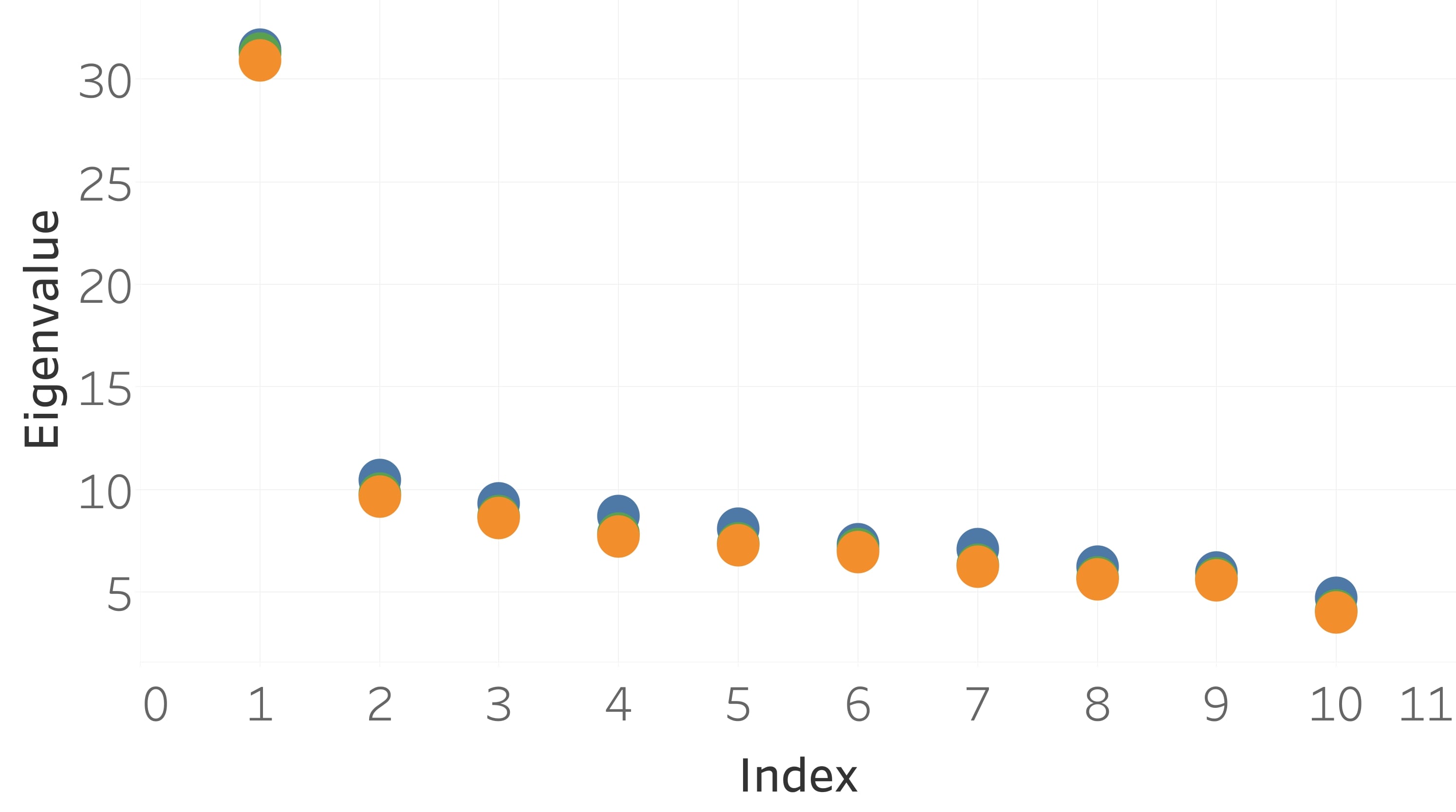}
        \caption{CIFAR10, 1351 examples per class.}
    \end{subfigure}
    
    \vspace{0.1cm}
    
    \centering
    \begin{subfigure}[t]{0.33\textwidth}
        \centering
        \includegraphics[width=1\textwidth]{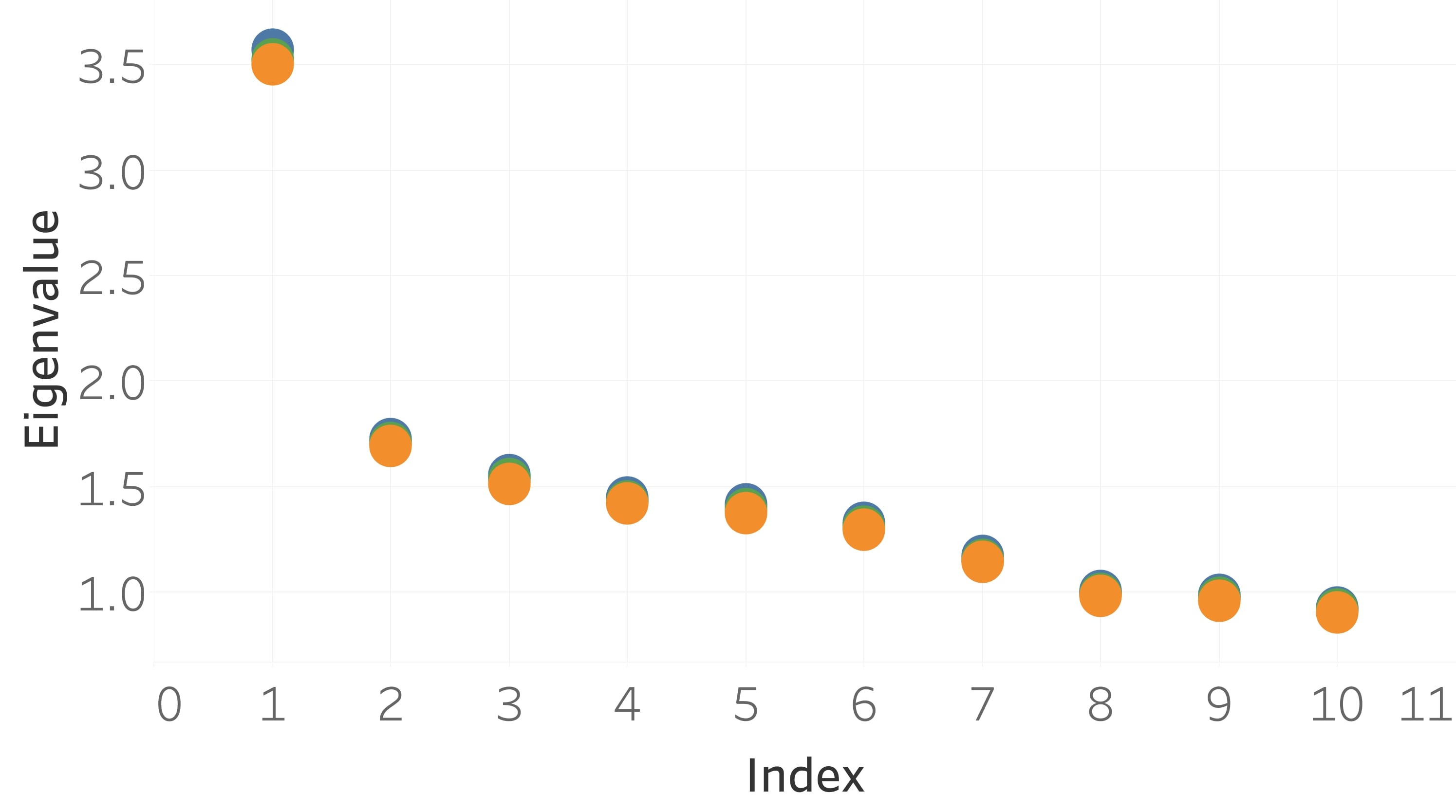}
        \caption{MNIST, 3605 examples per class.}
    \end{subfigure}%
    ~
    \begin{subfigure}[t]{0.33\textwidth}
        \centering
        \includegraphics[width=1\textwidth]{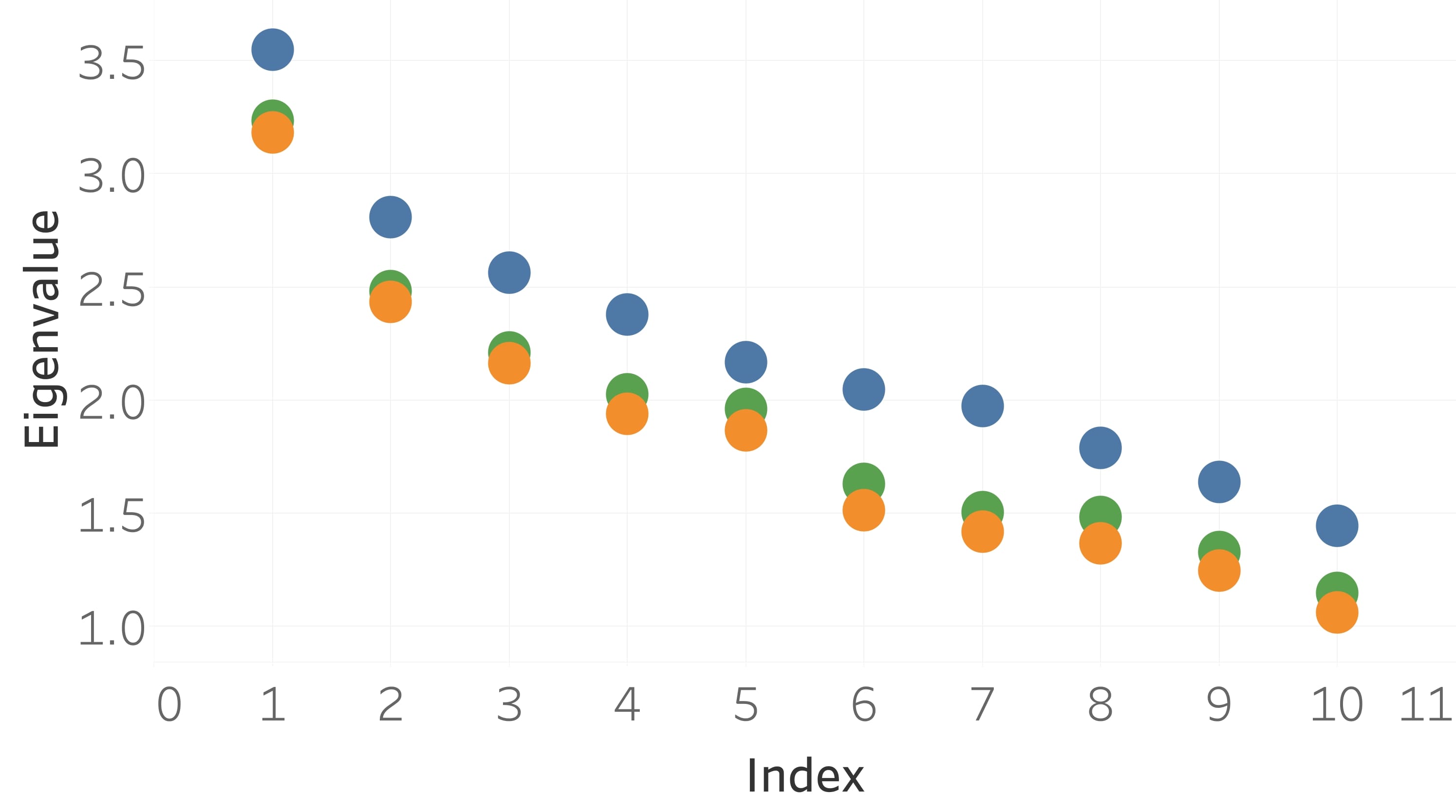}
        \caption{Fashion MNIST, 3605 examples per class.}
    \end{subfigure}%
    ~
    \begin{subfigure}[t]{0.33\textwidth}
        \centering
        \includegraphics[width=1\textwidth]{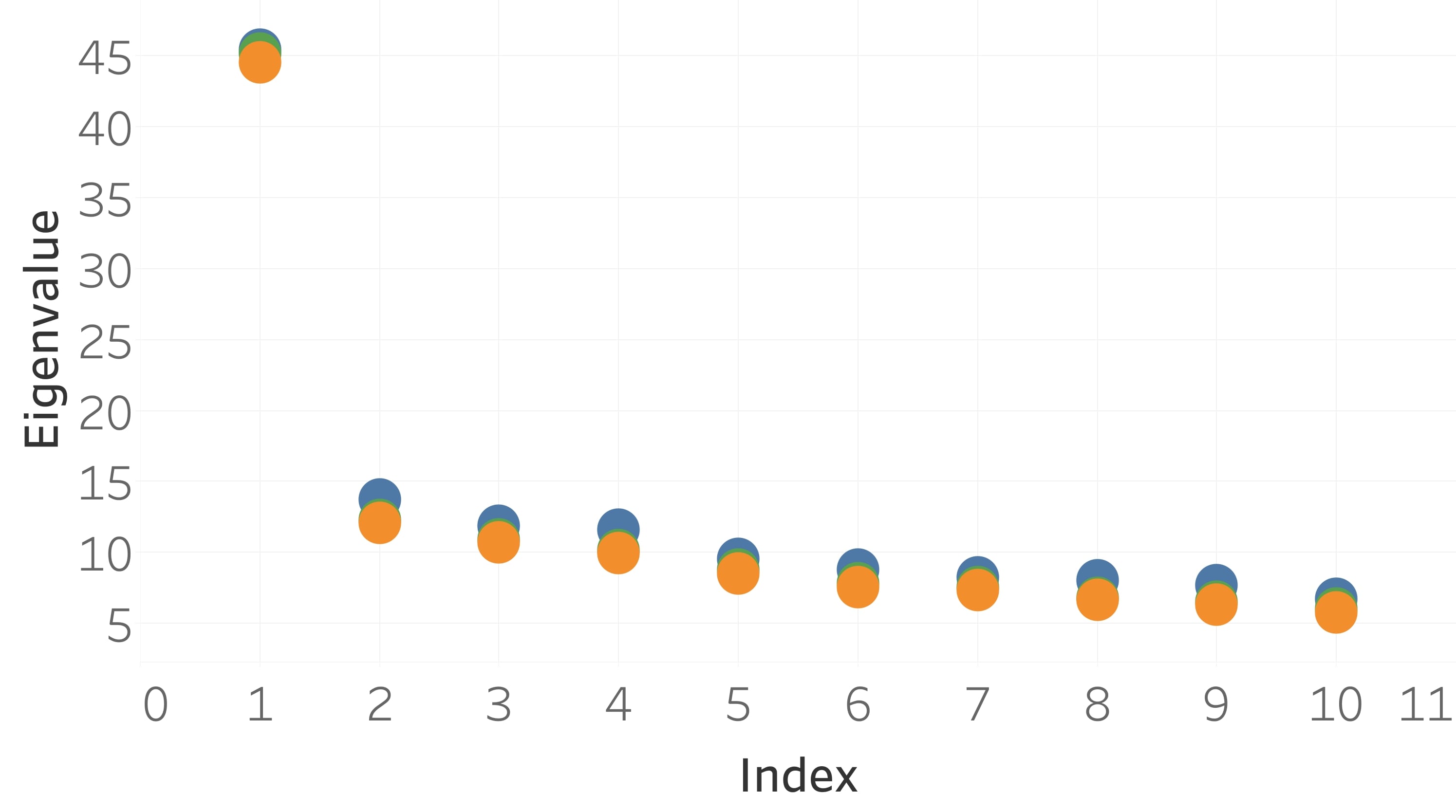}
        \caption{CIFAR10, 3605 examples per class.}
    \end{subfigure}
    \caption{\textit{Scree plots of $G_1$, $G_{1+2}$ and $G$ for the VGG11 architecture.} Each column of panels corresponds to a different dataset, and each row to a different sample size. Each panel plots the top-$C$ eigenvalues of $G_1$ in \textbf{{\color{Orange}orange}}, $G_{1+2}$ in \textbf{{\color{ForestGreen}green}} and $G$ in \textbf{{\color{MidnightBlue}blue}} (following the same color code as in Figure \ref{fig:DOS}). The top eigenvalues in $G$ -- which correspond to the outliers in the approximated spectrum of $G$ in Figure \ref{fig:DOS} -- were computed using the \textsc{LowRankDeflation} procedure in \cite{papyan2018full}. For every $1\leq c \leq C$, we have $\lambda_c(G) \geq \lambda_c(G_{1+2}) \geq \lambda_c (G_1)$. Moreover, $\lambda_c(G_{1+2})$ and $\lambda_c (G_1)$ are usually very close.
    } \label{fig:SI_SB_SW_ST}
\end{figure*}

\subsection{A note on stochasticity}
Deep learning practitioners often insert randomness into their architectures. The most common examples are preprocessing the input data, for example using random flips and crops, or using dropout \cite{srivastava2014dropout} layers. These sources of randomness complicate the analysis of the Hessian and its components in that they turn them into random variables. This, in turn, complicates the usage of the methods we employ in this paper -- such as Lanczos, subspace iteration and SVD -- all of which assume deterministic linear operators. To circumvent these nuisances, we do not employ any preprocessing on the input data and we replace the dropout layers in the VGG architecture with batch normalization layers.

\section{Conclusion}
Outliers sticking beyond the bulk edge were previously observed in the spectrum of the Hessian of deep networks. This paper described an organization of the ingredients of the Hessian which explains the outliers. The structuring we introduce here offers a novel three-level hierarchical decomposition. This provides an approximation for the outliers, which was proven empirically across many scenarios. Moreover, deviations between the two were found to exist, as might have been predicted by RMT.

\bibliography{example_paper}
\bibliographystyle{icml2019}

\end{document}